%% file: main.tex
\documentclass[12pt]{scrartcl}
\usepackage{amssymb}
\usepackage{amsmath}
\usepackage{amsthm}
\usepackage{cleveref}
\usepackage{cite}
\usepackage[normalem]{ulem}
\usepackage{subcaption} 
\usepackage{float}
\newtheorem{theorem}{Theorem}[section]

\newtheorem{lemma}[theorem]{Lemma}
\newtheorem{corollary}[theorem]{Corollary}

\newtheorem{assumption}[theorem]{Assumption}
\newtheorem{remark}[theorem]{Remark}
\input{macros.tex}

\title{Max-affine regression via first-order methods}

\author{Seonho Kim and Kiryung Lee}

\begin{document}

\maketitle

\begin{abstract}
\input{abstract}
\end{abstract}


\input{main_body}
\bibliographystyle{abbrv}
\bibliography{ref}
\input{appendix}
\end{document}

%% file: abstract.tex
We consider regression of a max-affine model that produces a piecewise linear model by combining affine models via the max function. 
The max-affine model ubiquitously arises in applications in signal processing and statistics including multiclass classification, auction problems, and convex regression. 
It also generalizes phase retrieval and learning rectifier linear unit activation functions. 
We present a non-asymptotic convergence analysis of gradient descent (GD) and mini-batch stochastic gradient descent (SGD) for max-affine regression when the model is observed at random locations following the sub-Gaussianity and an anti-concentration with additive sub-Gaussian noise. 
Under these assumptions, a suitably initialized GD and SGD converge linearly to a neighborhood of the ground truth specified by the corresponding error bound. 
We provide numerical results that corroborate the theoretical finding. 
{Importantly, SGD not only converges faster in run time with fewer observations than alternating minimization and GD in the noiseless scenario but also outperforms them in low-sample scenarios with noise.}



%% file: main_body.tex
\section{Introduction}
\label{sec:intro}

The \emph{max-affine} model combines $k$ affine models in the form of 
\begin{equation}
\label{eq:def_max_affine_model}
y = \max_{j \in [k]} \left(\langle \mb x, \mb \theta_j^\star \rangle+b_j^\star\right)
\end{equation} 
to produce a piecewise-linear mutivariate functions, where $\mb x$ and $y$ respectively denote the covariate and the response, and $[k]$ denotes the set $\{1,\dots,k\}$.
The max-affine model frequently arises in applications of statistics, machine learning, economics, and signal processing. 
For example, the max-affine model has been adopted for multiclass classification problems \cite{crammer2001algorithmic,daniely2012multiclass} and simple auction problems \cite{morgenstern2016learning,rubinstein2018simple}. 

We consider a regression of the max-affine model in \eqref{eq:def_max_affine_model} via least squares 
\begin{equation}
\label{eq:loss_def_max0}
\min_{\{\mb \theta_j,b_j\}_{j=1}^k}\frac{1}{2n}\sum_{i=1}^{n}\left(y_i- \max_{j \in [k]} (\langle\mb x_i,\mb \theta_j\rangle+b_j)\right)^2
\end{equation}
from statistical observations $\{(\mb x_i, y_i)\}_{i=1}^{n}$ potentially corrupted with noise. 
A suite of numerical methods has been proposed to solve the nonconvex optimization in  \eqref{eq:loss_def_max0} (e.g., \cite{magnani2009convex,toriello2012fitting,hannah2013multivariate,balazs2016convex}). 
The \emph{least-squares partition algorithm} \cite{magnani2009convex} iteratively refines the parameter estimate by alternating between the partition and the least-squares steps when the number of affine models $k$ is known a priori. 
The partitioning step classifies the inputs $\mb x_1, \dots, \mb x_n$ with respect to the maximizing affine models given estimated model parameters. 
The least-squares step updates the parameters for each affine model by using the corresponding observations. 
Later variations of the alternating minimization algorithm used an adaptive search for unknown $k$ \cite{hannah2013multivariate,balazs2016convex}. 
The consistency of these estimators has been derived. 
In more recent papers, Ghosh et al. \cite{ghosh2019max,ghosh2020max,ghosh2021max} established finite-sample analysis of the \emph{alternating minimization} (AM) estimator \cite{magnani2009convex} for the special case when the observations are generated from a ground-truth model. 
One can interpret their analysis through the lens of the popular \emph{teacher-student framework} \cite{mace1998statistical}. 
This framework has been widely adopted in statistical mechanics \cite{mace1998statistical,engel2001statistical} and machine learning \cite{zhong2017recovery,goldt2019dynamics,zhang2019learning,hu2020sharp}. 
It provides a theoretical understanding of how a specific model is trained and generalized through a ground-truth generative model \cite{hu2020sharp}. 
In this framework, a max-affine model (student) is trained by data generated from a ground-truth max-affine model (teacher) from $k$ fixed affine models. 
By using the provided data, the student model recovers parameters that produce the ground-truth model via AM. 
Since the max affine model is invariant under the permutation of the component affine models, the minimizer to \eqref{eq:loss_def_max0} is determined only up to the corresponding equivalence class. 
Ghosh et al. \cite{ghosh2021max} established a finite-sample analysis of AM under the standard Gaussian covariate assumption with independent stochastic noise. 
They showed that a suitably initialized alternating minimization converges linearly to a consistent estimate of the ground-truth parameters along with a non-asymptotic error bound. 
Moreover, they proposed and analyzed a spectral method that provides the desired initialization. 
They also further extended the theory to a generalized scenario with relaxed assumptions on the covariate model \cite{ghosh2019max,ghosh2020max}. 

In this paper, we present analogous theoretical and numerical results on max-affine regression by first-order methods including \emph{gradient descent} (GD) and \emph{stochastic gradient descent} (SGD).
The first-order methods have been widely used to solve various nonlinear least squares problems in machine learning \cite{goodfellow2016deep,finn2017model,sutton2018reinforcement,kairouz2021advances}. 
We observe that GD and SGD also perform competitively on max-affine regression compared to AM. 
In particular, SGD converges significantly faster (in run time) than AM in a noise-free scenario. 
\Cref{fig:linearconvergence} compares AM, GD, and a mini-batch SGD on random $50$ trials of max-affine regression where the ground-truth parameter vectors ${\{\mb \beta_j^{\star}\}}_{j=1}^k$ are selected randomly from the unit sphere. 
We plot the median of relative errors versus the average run time where the relative error is calculated as
\begin{equation*}
\label{eq:def_relative_error}
\min_{\pi\in\mathrm{Perm}([k])}\log_{10}\left({\sum_{j=1}^{k}\|\hat{\mb \beta}_{\pi(j)} - \mb \beta_j^\star\|_2^2/\sum_{j=1}^{k}\|\mb \beta_{j}^\star\|_2^2}\right)
\end{equation*} 
with $\mathrm{Perm}([k])$ and $\{\hat{\mb \beta}_j\}_{j=1}^k$ denoting the set of all possible permutations over $[k]$ and the 
estimated parameters, respectively.
Our main result provides a theoretical analysis of SGD that explains this empirical observation. 

\begin{figure}[th]
\centering
\includegraphics[scale=0.35]{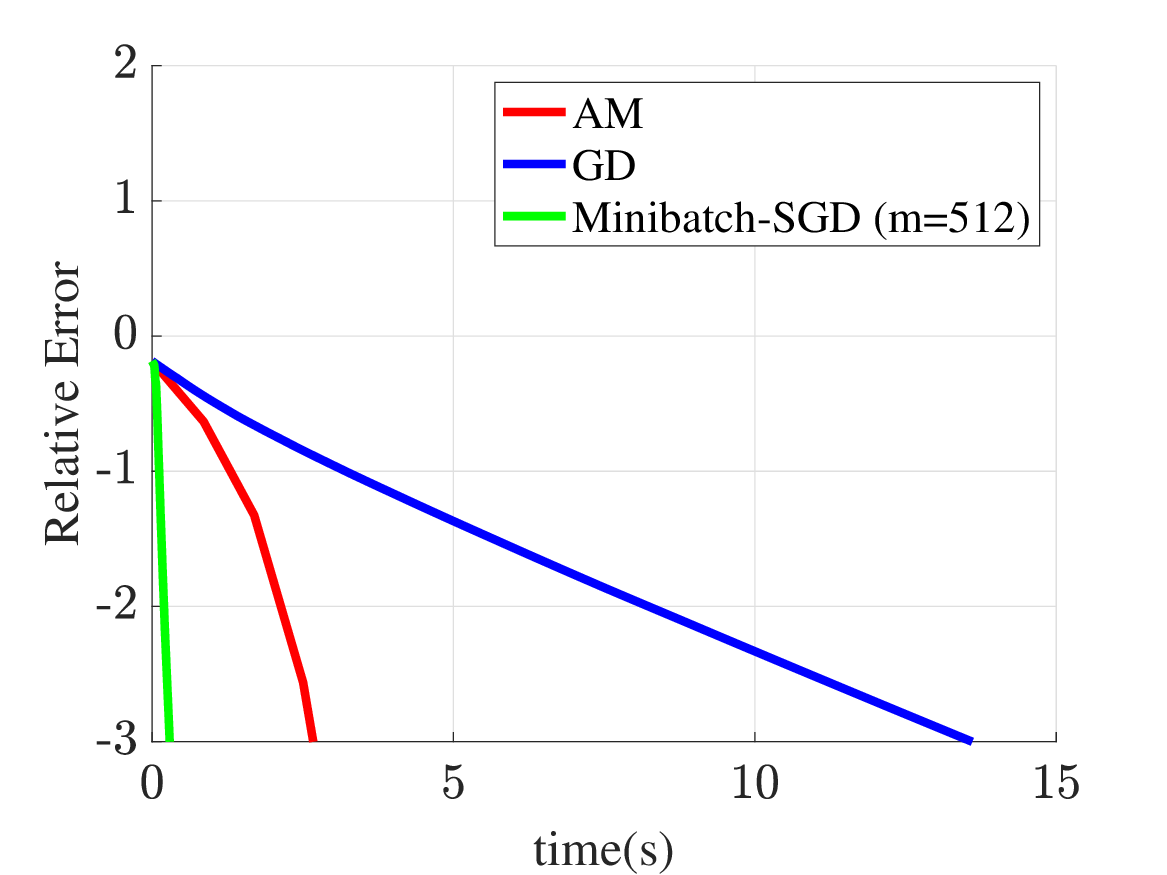}
\caption{Convergence of estimators for noise-free max-affine regression ($k=3$, $d=500$, and $n=8,000$).
}
\label{fig:linearconvergence}
\end{figure}

\subsection{Main results}

We derive convergence analyses of GD and mini-batch SGD under the same covariate and noise assumptions in the previous work on AM by Ghosh et al. \cite{ghosh2019max}. 
They assumed that covariates $\mb x_1, \dots, \mb x_n$ are independent copies of a random vector $\mb x$ that satisfies the sub-Gaussianity and anti-concentration defined below. 

\begin{assumption}[Sub-Gaussianity]
\label{main_assumption}
The covariate distribution satisfies \[
\norm{\langle \mb v, \mb x \rangle}_{\psi_2} \leq \eta, 
\quad \forall \mb v \in \mathbb{S}^{d-1},
\]
where $\|\cdot\|_{\psi_2}$ and $\mathbb{S}^{d-1}$ denote   the sub-Gaussian norm (i.e., see \cite[Equation~2.13]{vershynin2018high}) and the unit sphere in $\ell_2^d$, respectively. 
\end{assumption}

\begin{assumption}[Anti-concentration]
\label{main_assumption2}
The covariate distribution satisfies 
\[
\sup_{w \in \mathbb{R}, \mb v \in \mathbb{S}^{d-1}} \P( (\langle \mb v, \mb x \rangle + w)^2 \leq \epsilon) \leq (\gamma \epsilon)^\zeta, \quad \forall \epsilon>0. 
\]
\end{assumption}

\noindent The class of covariate distributions by Assumptions~\ref{main_assumption} and \ref{main_assumption2} generalizes the standard independent and identically distributed Gaussian distribution. 
For example, the uniform and beta distributions satisfy Assumptions~\ref{main_assumption} and \ref{main_assumption2}. 
Therefore, the theoretical result under this relaxed covariate model will apply to a wider range of applications. 
They also assumed that observations are corrupted with independent additive $\sigma$-sub-Gaussian noise. 

This paper establishes the first theoretical analysis of 
GD and mini-batch SGD 
for max-affine regression. 
The following pseudo-theorem demonstrates that GD shows a local linear convergence under the above assumptions. 

\begin{theorem}[Informal]
Let $\mb \beta^\star\in\mathbb{R}^{k(d+1)}$ denote the column vector that collects all ground-truth parameters $(\mb \theta^\star_j, b^\star_j)_{j\in [k]}$. 
Given $\tilde{O}(C_{\mb \beta^\star} kd(k^3\vee \sigma^2))$ observations, a suitably initialized GD for max-affine regression converges linearly to an estimate of $\bm{\beta}^\star$ with $\ell_2$-error scaling as $\tilde{O}(\sigma k^2\sqrt{d/n})$, where $C_{\mb \beta^\star}$ is a constant that implicitly depends on $k$ through $\mb \beta^\star$ but is independent of $d$.
\end{theorem}

The error bound by this theorem improves upon the best-known result on max-affine regression achieved by AM \cite[Theorem~2]{ghosh2019max}. 
The error bound for AM is larger by a factor that grows at least as $k^{-1+2\zeta^{-1}}$. 
We also present an analogous analysis for SGD. 
A specification for the noise-free observation scenario is stated as follows. 

\begin{theorem}[Informal]
A suitably initialized mini-batch SGD for max-affine regression with $\tilde{O}(C_{\bm{\beta}^\star} k^9 d)$ noise-free observations converges linearly to the ground truth $\bm{\beta}^\star$ for any batch size. 
\end{theorem}

The per-iteration cost of a mini-batch SGD with batch size $m$ is $O(kmd)$, which is significantly lower than those for GD $O(knd)$ and of AM $O(knd^2)$.
This implies the faster convergence of SGD in run time shown in \Cref{fig:linearconvergence}.
We also observe that SGD empirically recovers the ground-truth parameters from fewer observations (see \Cref{fig:k3,fig:p50}).

\subsection{Related Work}


\noindent\textbf{Relation to phase retrieval and ReLU regression: } 
The max-affine model includes well-known models in signal processing and machine learning as special cases. 
The instance of \eqref{eq:def_max_affine_model} for $k=2$ with $b_1^\star=b_2^\star=0$ and $\mb \theta_1^\star=-\mb \theta_2^\star=\mb \theta^\star$ reduces to $y=|\langle\mb x,\mb \theta^\star\rangle|$, which corresponds to a measurement model in phase retrieval. 
Similarly, the rectified linear unit (ReLU) $y=\max(\langle\mb x,\mb \theta^\star\rangle,0)$ is written in the form of \eqref{eq:def_max_affine_model} for $k=2$ with $\mb \theta_1^\star=\mb 0$ and $\mb \theta_2^\star=\mb \theta^\star$. 
A series of studies in \cite{zhang2017nonconvex,soltanolkotabi2017learning,tan2019phase,tan2019online,wang2018phase,kalan2019fitting,yehudai2020learning,vardi2021learning} has developed a statistical analysis of GD and SGD for phase retrieval and ReLU regression. It has been shown that for the noiseless case, GD and SGD converge linearly to a near-optimal estimate of the ground-truth parameters when the number of observations grows linearly with the ambient dimension $d$. In the context of bounded noise, GD converges to the ground truth within a radius determined by the noise level \cite{zhang2017nonconvex,wang2018phase}. However, it remained an open question whether GD is consistent under stochastic noise assumptions. Additionally, SGD in the presence of noise has not been thoroughly investigated yet. 
The main results of this paper address these questions on phase retrieval as a special case of max-affine regression.\\

\noindent\textbf{Relation to convex regression: } 
The max-affine model has also been adopted in parametric approaches to convex regression \cite{magnani2009convex,hannah2013multivariate,han2016multivariate,balazs2015near,balazs2016convex,balazs2022adaptively,siahkamari2019learning,siahkamari2020learning,siahkamari2022faster}. 
Let $f_\star:\mathbb{R}^d\rightarrow{\mathbb{R}}$ be an arbitrary convex function. The observations are given by  $\{(\mb x_i, y_i)\}_{i=1}^{n}$ where $y_i=f_\star(\mb x_i)$ for all $i$ in $[n]$. The nonparametric convex regression problem aims to estimate $f_\star$ by solving
\begin{equation}
\label{eq:convex_regression}
\min_{f\in \mathcal{F}_{\mathrm{cvx}}}\sum_{i=1}^{n}(y_i-f(\mb x_i))^2,
\end{equation}
where $\mathcal{F}_{\mathrm{cvx}}$ denotes the set of convex functions. Since $f$ exists in the space of continuous real-valued functions on $\mathbb{R}^p$, the optimization problem in \eqref{eq:convex_regression} is infinite-dimensional. 
A line of research \cite{boyd2004convex,balazs2015near,siahkamari2020learning} investigated the interpolation approach with a max-affine model in the form of 
\begin{equation}
\label{eq:max_affine_convexreg}
\hat{f}(\mb x)=\max_{i\in[n]}\left(y_i+\mb g_i^\T(\mb x-\mb x_i)\right). 
\end{equation} 
It provides a perfect interpolation of data $\{(\mb x_i,y_i)\}_{i=1}^{n}$ with zero training error. 
For example, the interpolation is achieved by choosing $\mb g_i\in \partial f_\star(\mb x_i)$ for all $i \in [n]$. 
It has been show that the least squares estimator provides near-optimal generalization bounds relative to a matching minimax bound
\cite{lim2012consistency,guntuboyina2012covering,balazs2016convex,han2016multivariate,kur2019optimality}. 
However, the minimax bound for the parametric model in \eqref{eq:max_affine_convexreg} decays slowly due to the curse of dimensionality for a set of max affine with $n$ segments.
The least squares for the model in \eqref{eq:max_affine_convexreg} is formulated as a quadratic program (QP) \cite[Section~6.5.5]{boyd2004convex}. 
However, off-the-shelf interior-point methods do not scale to large instances of this QP due to the high computational cost $O(d^4n^5)$ \cite{magnani2009convex,hannah2013multivariate}. 

The $k$-max-affine model in \eqref{eq:def_max_affine_model} is considered as an alternative compact parametrization to approximate convex regression. 
The worst-case error in approximating $d$-variate Lipschtiz convex functions on a bounded domain by a $k$-max-affine model decays as $O(k^{-2/d})$ \cite[Lemma~5.2]{balazs2016convex}. 
However, data in practical applications such as aircraft wing design, wage prediction, and pricing stock options are often well approximated by the $k$-max-affine model with small $k$ (e.g., \cite[Section~6]{hannah2013multivariate}, \cite[Section~7]{balazs2016convex}). 
Unlike the interpolation approach to convex regression, if the compact model fits data in applications, the estimation error decays much faster in $n$. \\ 

\noindent\textbf{Max-linear regression in the presence of deterministic noise: } 
A special instance of \eqref{eq:def_max_affine_model} with $b_j^\star=0$ for $j\in[k]$ is called the max-linear model. 
A convex optimization method to max-linear regression obtained with an initial estimate has been studied under the standard Gaussian covariate assumption and deterministic noise \cite{kim2021max}.  
They empirically showed that the convex estimator outperforms the existing methods in the presence of outliers.

\subsection{Organizations and Notations}

The rest of the paper is organized as follows: 
\Cref{sec:performance} formulates the least squares estimator for max-affine regression, describes the GD algorithm and presents the convergence analysis of GD. 
\Cref{sec:stochastic} describes a mini-batch SGD for max-affine regression and provides its convergence analysis. 
\Cref{sec:numerical} presents numerical results to compare the empirical performance of GD, SGD, and AM for max-affine regression. 
Finally, \Cref{sec:discussion} summarizes the contributions and discusses future directions.

Boldface lowercase letters denote column vectors, and boldface capital letters denote matrices. 
The concatenation of two column vectors $\mb a$ and $\mb b$ is denoted by $[\mb a;\ \mb b]$. 
The subvector of $\mb a \in \mathbb{R}^{d+1}$ with the first $d$ entries will be denoted by $(\mb a)_{1:d}$. 
Various norms are used throughout the paper. 
We use $\|\cdot\|$, $\|\cdot\|_{\mathrm{F}}$, $\|\cdot\|_{2}$, and $\|\cdot\|_{\psi_2}$ to denote the spectral norm, Frobenius norm, Euclidean norm, and sub-Gaussian norm respectively. 
Moreover, ${B_2^d}$ and $\mathbb{S}^{d-1}$ will denote the $d$-dimensional unit ball and unit sphere with respect to the Euclidean norm. 
For two scalars $q$ and $d$, we write $q\lesssim p$ if there exists an absolute constant $C>0$ such that $q\leq Cp$. 
We use $C,C_1,C_2,\ldots$ and $c,c_1,c_2,\dots$ to denote absolute constants that may vary from line to line. 
We adopt the big-$O$ notation so that $q \lesssim p$ is alternatively written as $q = O(p)$. 
With a tilde on top of $O$, we ignore logarithmic factors. 
For brevity, the shorthand notation $[n]$ denotes the set $\{1,\ldots,n\}$ for $n\in\mathbb{N}$. 
Moreover, $a \vee b$ and $a \wedge b$ will denote $\max(a,b)$ and $\min(a,b)$ for $a,b \in \mathbb{R}$.

\section{Convergence analysis of gradient descent}
\label{sec:performance}

We first formulate the least squares estimator for max-affine regression and derive the gradient descent algorithm. 
For brevity, let $\mb \xi:=[\mb x ;\ 1] \in \mathbb{R}^{d+1}$ and $\mb \beta_j:=[\mb \theta_j ;\ b_j]\in\mathbb{R}^{d+1}$. 
Then the model in \eqref{eq:def_max_affine_model} is rewritten as
\[
y = \max_{j \in [k]} \langle \mb \xi, \mb \beta_j^\star \rangle + \mathrm{noise}.
\]
The least squares estimator minimizes the quadratic loss function given by
\begin{equation}
\label{eq:loss_def_max}
\ell(\mb \beta):=\frac{1}{2n}\sum_{i=1}^{n}\left(y_i-\max_{j\in[k]}\langle\mb \xi_i,\mb \beta_j\rangle\right)^2,
\end{equation}
where $\mb \beta=[\mb \beta_1 ;\ \dots ;\ \mb \beta_k] \in\mathbb{R}^{k(d+1)}$. 

The gradient descent algorithm iteratively updates the estimate by
\[
\mb \beta^{t+1}=\mb \beta^t-\mu \nabla_{\mb \beta}\ell(\mb \beta^t),
\] 
where $\mu>0$ denotes a step size. 
A sub-gradient of the cost function in \eqref{eq:loss_def_max} with respect to the $j$th block $\mb \beta_j$ is written as
\begin{equation}
\label{eq:gradient}
\begin{aligned}
\nabla_{\mb \beta_j} \ell(\mb \beta) 
= \frac{1}{n}\sum_{i=1}^{n} \bbone_{\{\mb x_i\in \mathcal{C}_j\}} \left(\max_{j\in[k]}\langle\mb \xi_i,\mb \beta_j\rangle-y_i\right) \mb \xi_i,
\end{aligned}
\end{equation} 
where $\mathcal{C}_1, \dots, \mathcal{C}_k$ are defined by $\mb \beta$ as 
\begin{equation}
\label{eq:def_calCj}
\begin{aligned}
\mathcal{C}_j := \{ \mb w \in \mbb{R}^{d} ~:~ \langle [\mb w ;\ 1], \mb \beta_{j} - \mb \beta_{l} \rangle > 0, ~ \forall l < j, ~ \langle [\mb w ;\ 1], \mb \beta_{j} - \mb \beta_{l} \rangle \geq 0, ~ \forall l > j \}.
\end{aligned}
\end{equation}
The set $\mathcal{C}_j$ contains all inputs maximizing the $j$th linear model.\footnote{In case of a tie when multiple linear models attain the maximum for a given sample, we assign the sample to the smallest maximizing index. Since the event of duplicate maximizing indices will happen with probability $0$ for any absolutely continuous probability measure on $\mb{x}_i$s, the choice of a tie-break rule does not affect the analysis.\label{footnote1}}
Note that each $\mathcal{C}_j$ is determined by $k-1$ half spaces given by the pairwise difference of the $j$th linear model and the others. 
Then the gradient $\nabla_{\mb \beta}\ell(\mb \beta)$ is obtained by concatenating $\{\nabla_{\mb \beta_j}\ell(\mb \beta)\}_{j=1}^k$ by
\[
\nabla_{\mb \beta}\ell(\mb \beta)=\sum_{j=1}^{k}\mb e_j\otimes\nabla_{\mb \beta_j}\ell(\mb \beta),
\] 
where $\mb e_j \in \mathbb{R}^k$ denotes the $j$th column of the $k$-by-$k$ identity matrix $\mb I_k$ for $j \in [k]$.
Moreover, $\ell(\mb \beta)$ is differentiable except on a set of measure zero, with a slight abuse of terminology, $\nabla_{\mb \beta}\ell(\mb \beta)$ is referred to as the ``gradient''.

Next, we present a convergence analysis of the gradient descent estimator. 
The analysis depends on a set of geometric parameters of the ground-truth model. 
The first parameter $\pi_{\min}$ describes the minimum portion of observations corresponding to the linear model which achieved the maximum least frequently. 
It is formally defined as a lower bound on the probability measure on the smallest partition set, i.e. 
\begin{equation}
\label{eq:def_pimin_pimax}
\min_{j\in[k]}\P(\mb x\in \mathcal{C}_j^\star) \geq \pi_{\min}, 
\end{equation}
where $\mathcal{C}_1^\star, \dots, \mathcal{C}_k^\star$ are polytopes determined by 
\begin{equation}
\label{eq:def_calCjstar}
\begin{aligned}
\mathcal{C}_j^\star := \{ \mb w \in \mbb{R}^{d} ~:~ \langle [\mb w ;\ 1], \mb \beta_{j}^\star - \mb \beta_{l}^\star \rangle > 0, ~ \forall l < j, ~ \langle [\mb w ;\ 1], \mb \beta_{j}^\star - \mb \beta_{l}^\star \rangle \geq 0, ~ \forall l > j \}.
\end{aligned}
\end{equation}
The next parameter $\kappa$ quantifies the separation between all pairs of distinct linear models in \eqref{eq:def_max_affine_model} so that the pairwise distance on two distinct linear models satisfy
\label{ass:kappa}
\begin{equation}
\label{eq:defkappa}
\min_{j'\neq j}\|(\mb \beta_j^\star)_{1:d} - (\mb \beta_{j'}^\star)_{1:d} \|_2 
\geq \kappa.
\end{equation}

Our main result in the following theorem presents a local linear convergence of the gradient descent estimator uniformly over all $\mb \beta^\star$ satisfying \eqref{eq:def_calCjstar} and \eqref{eq:defkappa}. 

\begin{theorem}
\label{thm:main_noise}
Let $\delta\in(0,1/e),$ $y_i=\max_{j\in[k]}\langle\mb \xi_i,\mb 
\beta_j^\star\rangle+z_i$ for $i \in [n]$ with $\mb \xi_i=[\mb x_i ;\ 1]$, and $\{z_i\}_{i=1}^n$ being additive $\sigma$-sub-Gaussian noise independent from everything else. 
Suppose that Assumptions~\ref{main_assumption} and \ref{main_assumption2} hold.\footnote{To simplify the presentation, we assume that the parameters $\eta$, $\zeta$, $\gamma$ in Assumptions~\ref{main_assumption} and \ref{main_assumption2} are fixed numerical constants in the statement and proof of Theorem~\ref{thm:main_noise}. 
Therefore, any constant determined only by $\eta$, $\zeta$, $\gamma$ will be treated as a numerical constant.}
Then there exist absolute constants $C,C',R>0,$ and $\nu \in(0,1)$, for which the following statement holds with probability at least $1-\delta$: 
If the initial estimate $\mb \beta^0$ belongs to a neighborhood of $\mb \beta^\star$ given by 
\begin{equation}
\label{eq:defnbr}
\begin{aligned}
\mathcal{N}(\mb \beta^\star):= \Bigg\{ \mb \beta \in \mbb{R}^{k(d+1)} \,:\, \max_{j \in [k]} \|\mb \beta_j - \mb \beta_j^\star\|_2 \leq \kappa\rho \Bigg\}
\end{aligned}
\end{equation}
with 
\begin{equation}
\label{eq:choice_rho1}
\rho:= \frac{R\pi_{\min}^{\zeta^{-1}(1+\zeta^{-1})}}{4k^{\zeta^{-1}}} \cdot \log^{-1/2}\left(\frac{k^{\zeta^{-1}}}{R\pi_{\min}^{\zeta^{-1}(1+\zeta^{-1})}}\right) \wedge \frac{1}{4},
\end{equation} then for all $\mb \beta^\star$ satisfying \eqref{eq:def_pimin_pimax} and \eqref{eq:defkappa}, the sequence $\left(\mb \beta^t\right)_{t\in\mathbb{N}}$ by the gradient descent method with a constant step size satisfies 
\begin{equation}
\label{eq:errorbound_noise}
\begin{aligned}
\left\|\mb \beta^t-\mb \beta^\star\right\|_2&\leq \nu^t\left\|\mb \beta^0-\mb \beta^\star\right\|_2
+ C'\sigma k\frac{ \sqrt{k\left(kd\log(n/d)+\log(k/\delta)\right)}}{\sqrt{n}}, \quad\forall t\in\mathbb{N},
\end{aligned} 
\end{equation} 
provided that 
\begin{equation}
\label{eq:samplcomp_noise}
\begin{aligned}
n \geq C \pi_{\min}^{-2(1+\zeta^{-1})} \cdot \left( k^{1.5}\pi_{\min}^{-(1+\zeta^{-1})} \vee \frac{\sigma}{\kappa \rho} \right)^2 \cdot \left( kd\log(n/d)+\log(k/\delta) \right).
\end{aligned}
\end{equation}
\end{theorem}
\begin{proof}
See \Cref{sec:thmproof}.
\end{proof}

\Cref{thm:main_noise} demonstrates that the GD estimator with a constant step size converges linearly to a neighborhood of the ground-truth parameter of radius $\tilde{O}\left(\sigma^2k^4d/n\right)$.
The number of sufficient observations to invoke this convergence result scales linearly in $d$ and is proportional to a polynomial in $\pi_{\min}^{-1}$ and $k$. This result implies the consistency of the gradient descent estimator. To compare \Cref{thm:main_noise} to the analogous result for AM under the same covariate and noise models \cite[Theorem~1]{ghosh2020max}, we have the following remarks in order. 

\begin{itemize}
    \item First, the final estimation error by \eqref{eq:errorbound_noise} with $t \to \infty$ is smaller than that by \cite[Theorem~1]{ghosh2020max} by being independent of $\pi_{\min}^{-1}$, which grows at least proportional to $k$. A larger estimation error bound in their result is due to the analysis of the least squares update, wherein the smallest singular value of the design matrix of each linear model is utilized. These quantities do not appear in the analysis of the gradient descent update. 

    \item Second, the convergence parameter $\nu$ in \eqref{eq:errorbound_noise} is smaller than $3/4$ for AM\footnote{As shown in the proof in \Cref{sec:thmproof}, the parameter $\nu$ is given as $\nu=(1-\mu\lambda)$ by \eqref{eq:recursion_decrease}. The quantity $\mu\lambda$ is determined by \eqref{eq:muchoose} and \eqref{eq:lambdachoose} as a function of $\pi_{\min}$, $\pi_{\max}$, and $\zeta$ so that it decreases in $k$ and $\pi_{\min}^{-1}$.}, which might result in a slower convergence of GD in iteration count. The convergence speed issue becomes significant for large $k$ and $\pi_{\min}^{-1}$. For example, in the illustration by \Cref{fig:linearconvergence}, GD shows a slower convergence in run time despite the lower per-iteration cost $O(knd)$, which is lower than that of AM $O(knd^2)$ by a factor of $d$. However, as discussed in \Cref{sec:stochastic}, the slow convergence of GD can be improved by modifying the algorithm into a (mini-batch) SGD. 

    \item Third, the sample complexity results by \Cref{thm:main_noise} and \cite[Theorem~1]{ghosh2020max} are qualitatively comparable. There were mistakes in the proof of \cite[Theorem~1]{ghosh2020max}. We think that their result could be corrected with an increased order of dependence in their sample complexity on $k$ and $\pi_{\min}$ (see \Cref{sec:correctness} for a detailed discussion). 

    \item Lastly, regarding the proof technique, we adapt and improve the strategy by Ghosh et al. \cite{ghosh2019max,ghosh2020max}. Note that the subgradient of the loss function in \eqref{eq:gradient} involves clustering of covariates with respect to maximizing linear models such as \eqref{eq:def_calCj}, which also arises in alternating minimization. Due to this similarity, key quantities in the analysis have been estimated in \cite{ghosh2019max,ghosh2020max}. We provide sharpened estimates via different techniques. For example, \Cref{lem:lwb_trunc} provides a tighter bound than \cite[Lemma~7]{ghosh2019max} by a factor of $\alpha^{\zeta^{-1}}$ for a scalar $\alpha\in(0,1)$.

\end{itemize}

Theorem~\ref{thm:main_noise} also provides an auxiliary result.  
As a direct consequence of Theorem~\ref{thm:main_noise}, we obtain an upper bound on the prediction error, which is defined by
\[
\mathcal{E}(\hat{\mb \beta}) := 
\E\left(\max_{j\in[k]}\langle\mb \xi,\hat{\mb \beta}_j\rangle-\max_{j\in[k]}\langle\mb \xi,\mb \beta_j^\star\rangle\right)^2,
\]
where $\hat{\mb \beta} = [\hat{\mb \beta}_1 ;\ \dots ;\ \hat{\mb \beta}_k]$ denotes the estimated parameter vector by GD. 
Since the quadratic cost function in \eqref{eq:loss_def_max0} is $1$-Lipschitz with respect to the $\ell_{2}$ norm, it follows that the prediction error $\mathcal{E}(\hat{\mb \beta})$ is also bounded by $\tilde{O}(\sigma^2k^3d/n)$ as in \eqref{eq:errorbound_noise} with $t \to \infty$. 

A limitation of Theorem~\ref{thm:main_noise} is that its local convergence analysis requires an initialization within a specific neighborhood of the ground-truth parameter. 
To obtain the desired initial estimate, one may use spectral initialization by \cite[Algorithm~2, 3]{ghosh2021max}, which consists of dimensionality reduction followed by a grid search. 
They provided a performance guarantee of a spectral initialization scheme under the standard Gaussian covariate assumption \cite[Theorems~2 and 3]{ghosh2021max}. 
Therefore, the reduction of Theorem~\ref{thm:main_noise} to the Gaussian covariate case combined with \cite[Theorems~2 and 3]{ghosh2021max} provides a global convergence analysis of GD, which is comparable to that for alternating minimization \cite{ghosh2021max}. 
Even in this case, the number of sufficient samples for the success of spectral initialization overwhelms that for the subsequent gradient descent step. 
Since multiple steps of their analysis critically depend on the Gaussianity, it remains an open question whether the result on the spectral initialization generalizes to the setting by Assumptions~\ref{main_assumption} and \ref{main_assumption2}.

\section{Convergence analysis of mini-batch SGD}
\label{sec:stochastic}

SGD is an optimization method that updates parameters using a single or a small batch of randomly selected data point(s) instead of the entire dataset. 
SGD converges faster in run time than GD due to its significantly lower per-iteration cost. 
In particular, when applied to max-affine regression, SGD empirically outperforms GD and AM in both sample complexity and convergence speed (see \Cref{fig:linearconvergence,fig:k3,fig:p50}). 
In this section, we present an accompanying theoretical convergence analysis of mini-batch SGD for max-affine regression. 
The update rule of a mini-batch SGD with batch size $m$ for max-affine regression is described as follows. 
For each iteration index $t \in \mathbb{N}$, let $I_t$ be a multiset of $m$ randomly selected indices with replacement so that the entries of $I_t$ are independent copies of a uniform random variable in $[n]$. 
A mini-batch SGD iteratively updates the estimate by 
\[
\mb \beta^{t+1}=\mb \beta^t -\mu \frac{1}{m}\sum_{i\in I_t}\nabla_{\mb \beta} \ell_{i}(\mb \beta^t),
\] 
where 
\[
\ell_{i}(\mb \beta):=\frac{1}{2}\left(y_i-\max_{j\in[k]}\langle\mb \xi_{i},\mb \beta_j\rangle\right)^2, \quad i \in [n].
\] 
Then the following theorem presents a local linear convergence of SGD.

\begin{theorem}
\label{thm:main_mini_batch}
Under the hypothesis of \Cref{thm:main_noise}, there exist absolute constants $C,C'>0$ and $c,\nu\in(0,1)$, for which the following statement holds with probability at least $1-\delta$: 
For all $\mb \beta^\star$ satisfying \eqref{eq:def_calCjstar} and \eqref{eq:defkappa}, if the initial estimate $\mb \beta^0$ belongs to $\mathcal{N}(\mb \beta^\star)$ defined in \eqref{eq:defnbr}, $n$ satisfies \eqref{eq:samplcomp_noise}, and $m$ satisfies
\begin{equation}
\label{eq:minibatchsize}
\begin{aligned}
m & \geq C \cdot \left(\frac{\sigma}{\kappa\rho}\right)^2\cdot\left(d+\log(k/\delta)\right),
\end{aligned}
\end{equation}
then the sequence $\left(\mb \beta^t\right)_{t\in\mathbb{N}}$ by the mini-batch SGD with batch size $m$ and step size $\mu=c\left(1\vee{m}/{\left(d+\log(n/\delta)\right)}\right)$ satisfies 
\begin{equation}
\label{eq:errorbound_minibatch}
\begin{aligned}
\E_{I_t}\left\|\mb \beta^t-\mb \beta^\star\right\|_2&\leq \left(1-\left(1\wedge\frac{m}{d+\log(n/\delta)}\right)c\nu\right)^t\left\|\mb \beta^0-\mb \beta^\star\right\|_2\\
&+C'\sigma k\sqrt{\left(\frac{d+\log(n/\delta)}{m}\vee\frac{kd\log(n/d)+\log(1/\delta)}{n}\right)},\quad\forall t\in\mathbb{N}.
\end{aligned}
\end{equation}
\end{theorem}
\begin{proof}
See \Cref{sec:proof:minibatch}.
\end{proof}

\Cref{thm:main_mini_batch} establishes linear convergence of mini-batch SGD in expectation to the ground-truth parameters within error $\tilde{O}(\sigma^2k^2\left({d}/{m}\vee{kd}/{n}\right))$.
The local linear convergence applies uniformly over all $\mb \beta^\star$ satisfying \eqref{eq:def_calCjstar} and \eqref{eq:defkappa}. 
In general, the convergence rate of SGD is much slower even with strong convexity \cite{nemirovski2009robust,bottou2018optimization,harvey2019tight}. 
However, in a special case where the cost function is in the form of $\sum_{i=1}^{n}\ell_i(\mb \beta)$, smooth, and strongly convex, if $\bm{\beta}^\star$ is the minimizer of all summands $\{\ell_i(\mb \beta)\}_{i=1}^n$, then SGD converges linearly to $\bm{\beta}^\star$ \cite[Theorem~2.1]{needell2014stochastic}. 
The convergence analysis in \Cref{thm:main_mini_batch} can be considered along with this result. 
The cost function in \eqref{eq:loss_def_max} in the noiseless case satisfies the desired properties locally near the ground truth, whence establishes the local linear convergence of SGD.

\Cref{thm:main_mini_batch} also explains how the batch size $m$ affects the final estimation error by \eqref{eq:errorbound_minibatch} with $t\to \infty$. 
Let $n$ and $m$ satisfy \eqref{eq:samplcomp_noise} and \eqref{eq:minibatchsize} so that \Cref{thm:main_mini_batch} is invoked. 
Under this condition, one can still choose $m$ and $n$ so that $m\lesssim n/k$. 
Then the $\tilde{O}(\sigma^2k^{2}d/m)$ term determined by the batch size $m$ dominates the final estimation error. 
In this regime, the SGD estimator is not consistent since the estimation error $\tilde{O}(\sigma^2k^{2}d/m)$ does not vanish with increasing $n$. 
This result implies the trade-off between the convergence speed and the final estimation error determined by the batch size. 

Furthermore, since the condition on $m$ in \eqref{eq:minibatchsize} becomes trivial when $\sigma = 0$, we obtain a stronger result in the noiseless case given by the following corollary. 

\begin{corollary}
\label{cor:minibatch_highprob}
Let $\delta,\delta'\in(0,1),$ and $\epsilon>0$ fixed. 
Suppose that the hypothesis of \Cref{thm:main_mini_batch} holds. 
If $t\geq\left(\log(1/\epsilon)+\log(1/\delta)\right)\left(1\vee\frac{d+\log(n/\delta)}{m}\right)1/\nu,$  then
\[
\left\|\mb \beta^t-\mb \beta^\star\right\|_2\leq \epsilon\|\mb \beta^0-\mb \beta^\star\|_2
\] 
holds with probability at least $1-\delta-\delta'.$ 
\end{corollary}
\begin{proof} 
By \Cref{thm:main_mini_batch}, \eqref{eq:errorbound_minibatch} holds with probability at least $1-\delta$. 
By applying Markov's inequality, we have 
\[
\P\left(\left\|\mb \beta^t-\mb \beta^\star\right\|_2\geq\epsilon\|\mb \beta^0-\mb \beta^\star\|_2\right)\leq\frac{\E_{I_t}\|\mb \beta^t-\mb \beta^\star\|_2}{\epsilon\|\mb \beta^0-\mb \beta^\star\|_2}\leq\frac{\left(1-\left(1\wedge\frac{m}{d+\log(n/\delta)}\right)\nu\right)^t}{\epsilon}\leq\delta',
\] where the second and third inequalities hold by \eqref{eq:errorbound_minibatch} and assumption on $t$ respectively.
\end{proof}

\Cref{cor:minibatch_highprob} presents the convergence of SGD with high probability, which is stronger than the convergence in expectation. 
Furthermore, there is no requirement on the batch size in invoking \Cref{cor:minibatch_highprob}. 
This result is analogous to the recent theoretical analysis of phase retrieval by randomized Kaczmarz \cite{tan2019phase} and SGD \cite{tan2019online}.

\section{Numerical results}
\label{sec:numerical}
We study the empirical performance of GD and mini-batch SGD for max-affine regression. 
The performance of these first-order methods is compared to AM \cite{ghosh2021max}. 
All of these algorithms start from the spectral initialization by Ghosh et al. \cite{ghosh2021max}. 
We use a constant step size $0.5$ for GD. 
The step size for SGD is set to $\frac{1\wedge (m/d)}{2}$ adaptive to the batch size. 
Since the spectral initialization operates under the Gaussian covariate model, covariates $\mb x_1,\ldots,\mb x_n$ are generated as independent copies of a random vector following $\mathrm{Normal}(\mb 0,\mb I_d)$.

First, we observe the performance of the three estimators for the exact parameter recovery in the noiseless case. 
In this experiment, the ground-truth parameters $\mb \theta_1^\star, \dots, \mb \theta_k^\star$ are generated as $k$ random pairwise orthogonal vectors with $k < d$, and the offset terms are set to $0$, i.e., $b_j^\star=0$ for all $j \in [k]$. 
By the construction, the probability assigned to the maximizer set of each linear model will be approximately $\frac{1}{k}$. 
In other words, the parameters $\pi_{\max}$ and $\pi_{\min}$ of the ground truth concentrate around $\frac{1}{k}$ where $\pi_{\min}$ is defined in \eqref{eq:def_pimin_pimax} and $\pi_{\max}:=\max_{j\in[k]}\P(\mb x\in\mathcal{C}_j^\star)$. 
Furthermore, due to the orthogonality, the pairwise distance satisfies $\|\mb \theta_j^\star - \mb \theta_{j'}^\star\|_2 = \sqrt{2}$ for all $j \neq j' \in [k]$. 
Consequently, the sample complexity results for GD and SGD by \Cref{thm:main_noise} and \Cref{thm:main_mini_batch} simplify to an easy-to-interpret expression $\tilde{O}(k^{16}d)$ that involves only $k$ and $d$. 
The sample complexity result on AM \cite{ghosh2019max} simplifies similarly.

\begin{figure}[H]
    \centering
    \hfill    
    \begin{minipage}{0.32\textwidth}
        \centering
        \includegraphics[scale=0.2]{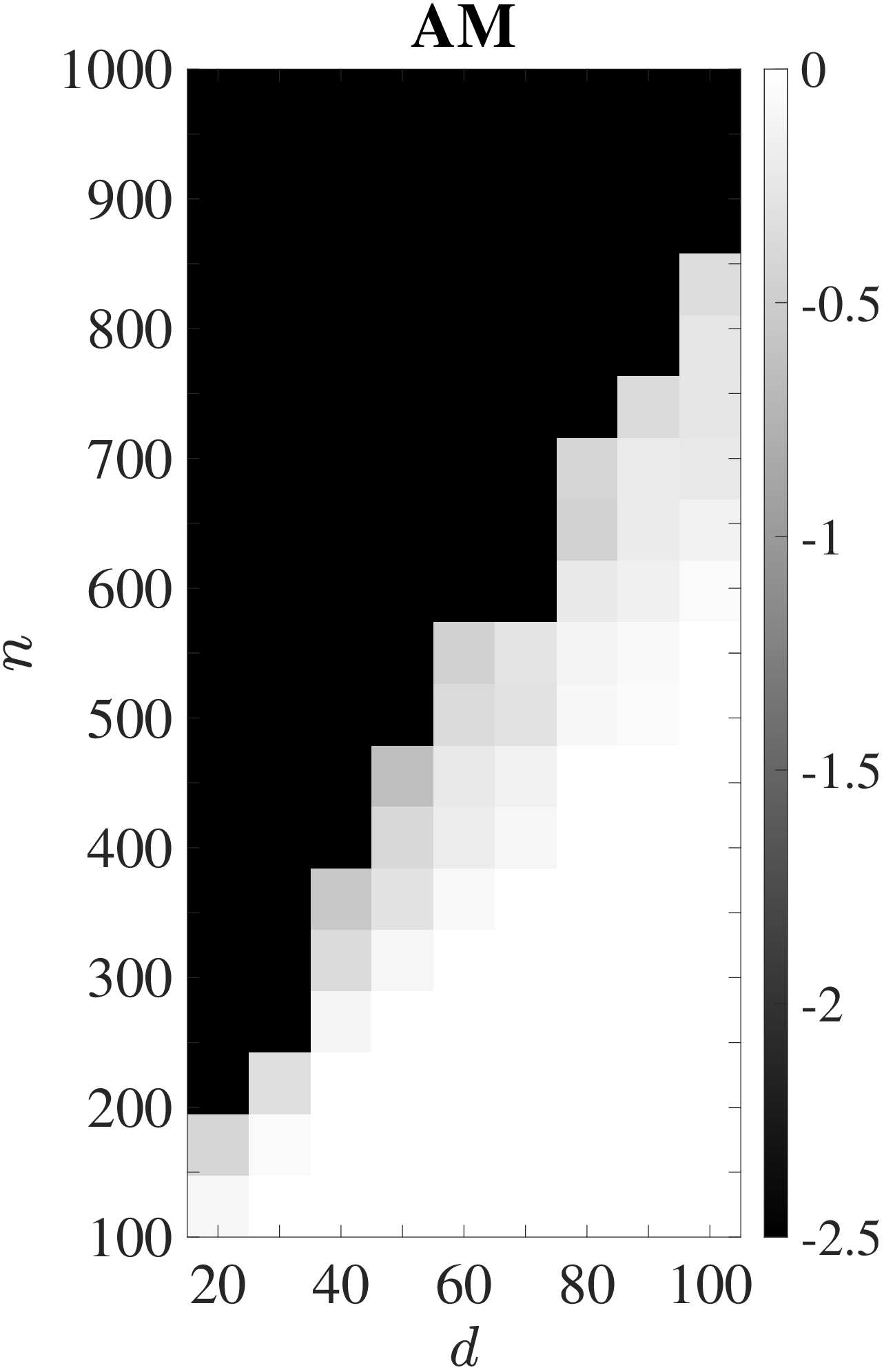}
    \end{minipage}
    \hfill
    \begin{minipage}{0.32\textwidth}
        \centering
        \includegraphics[scale=0.2]{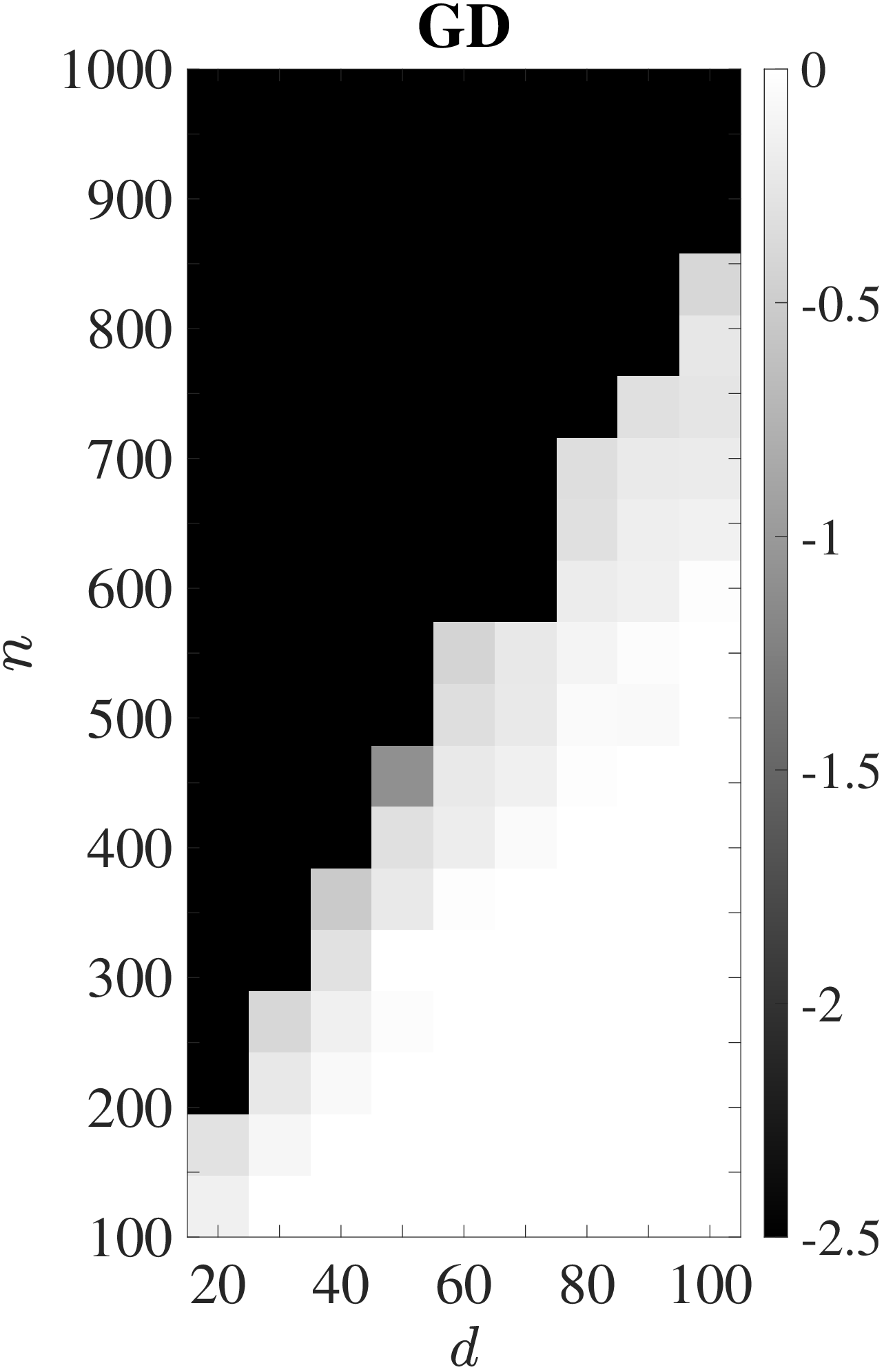}
    \end{minipage}
    \hfill
    \begin{minipage}{0.32\textwidth}
        \centering
        \includegraphics[scale=0.2]{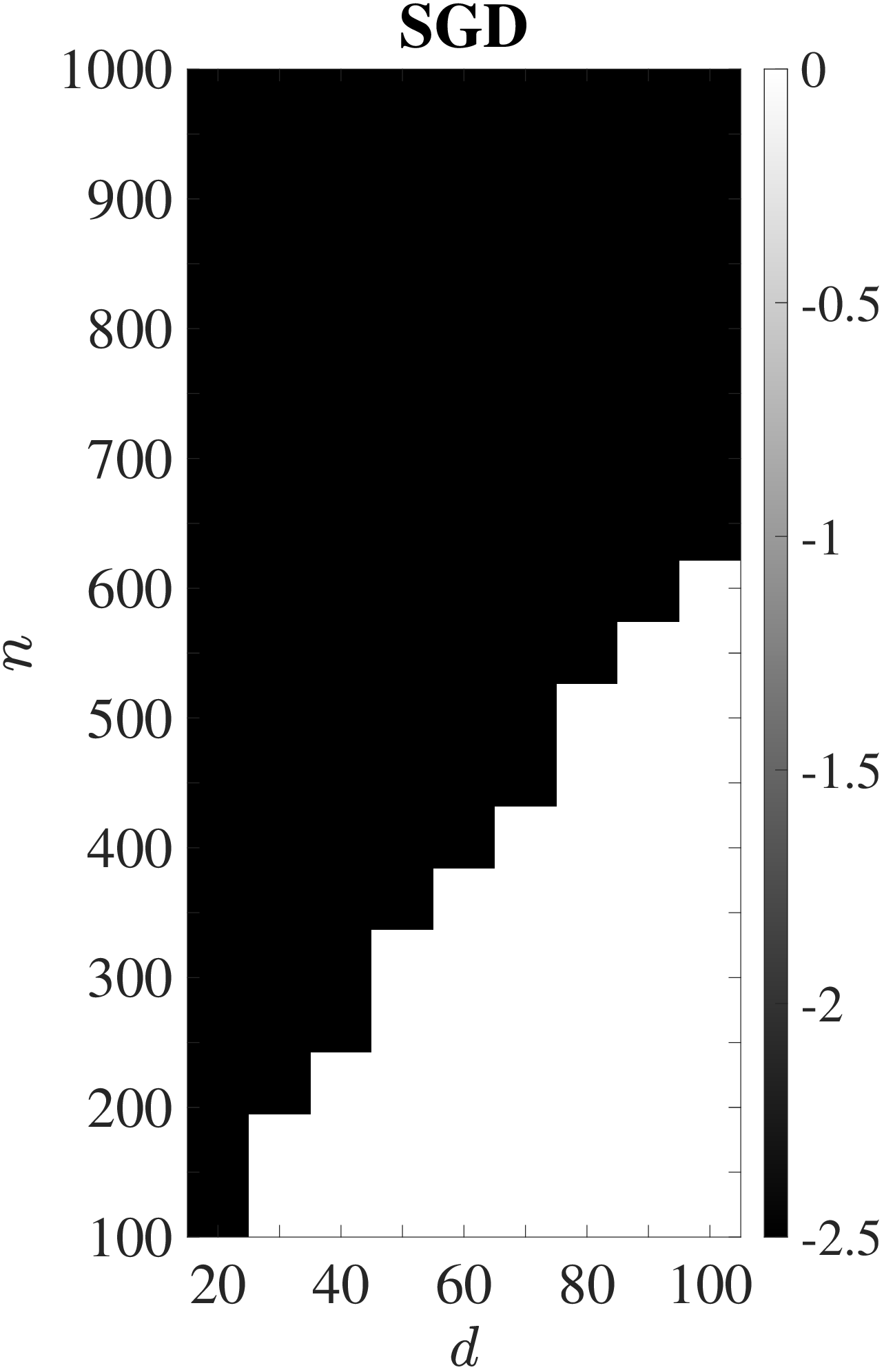}
    \end{minipage}
    \hfill
    \vspace{2em} \\
    \hfill    
    \begin{minipage}{0.32\textwidth}
        \centering
        \includegraphics[scale=0.2]{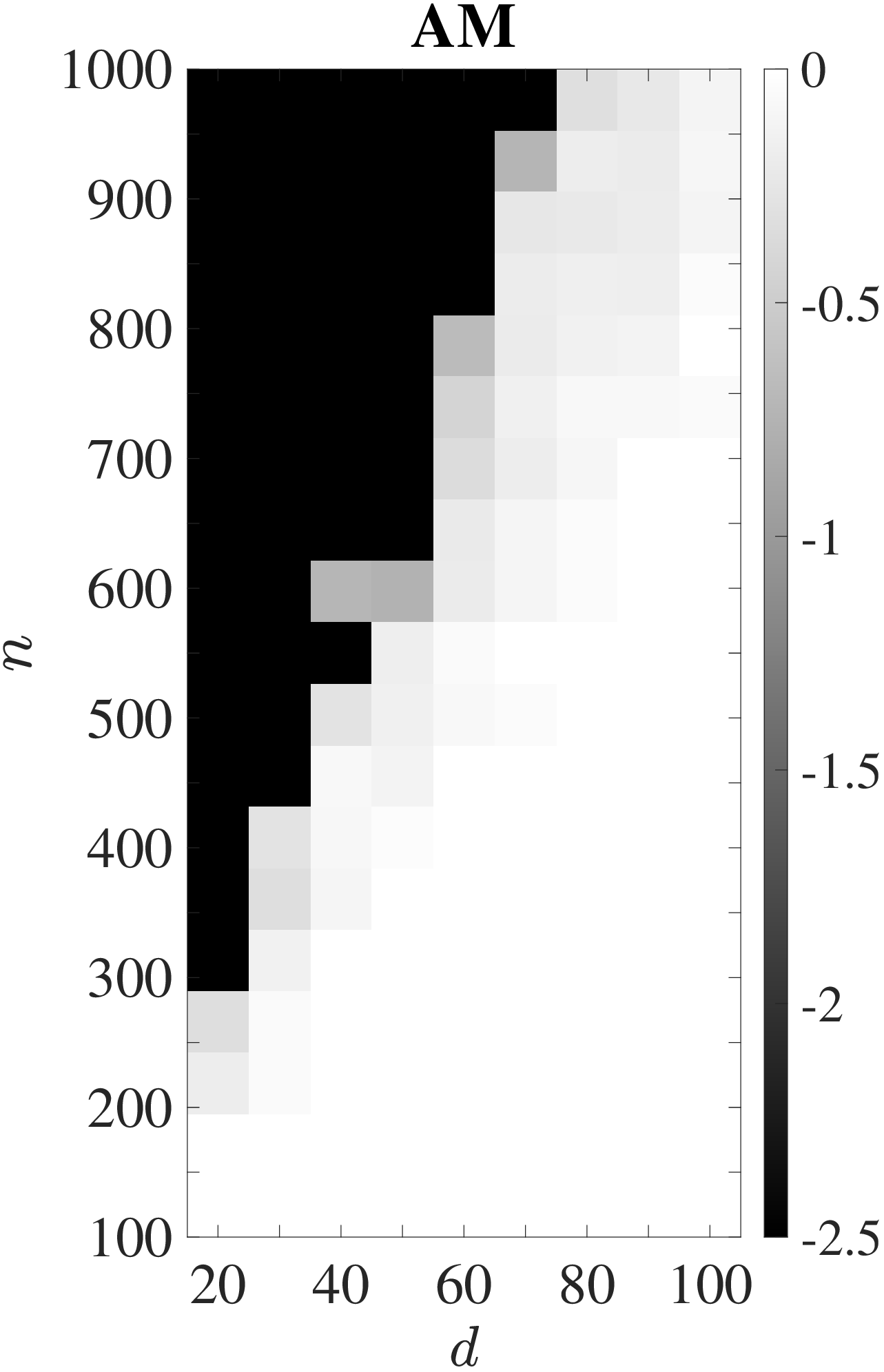}
    \end{minipage}
    \hfill
    \begin{minipage}{0.32\textwidth}
        \centering
        \includegraphics[scale=0.2]{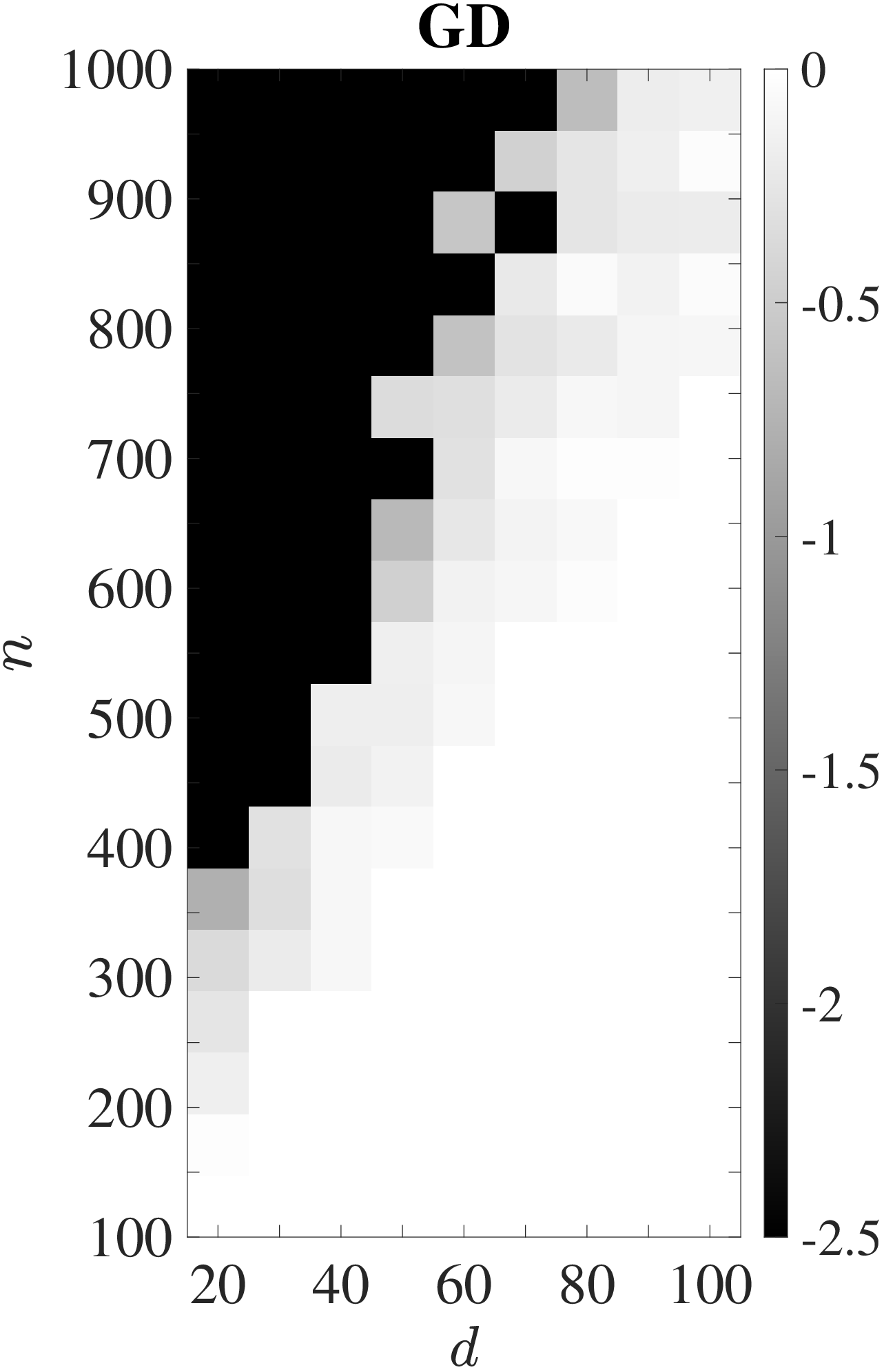}
    \end{minipage}
    \hfill
    \begin{minipage}{0.32\textwidth}
        \centering
        \includegraphics[scale=0.2]{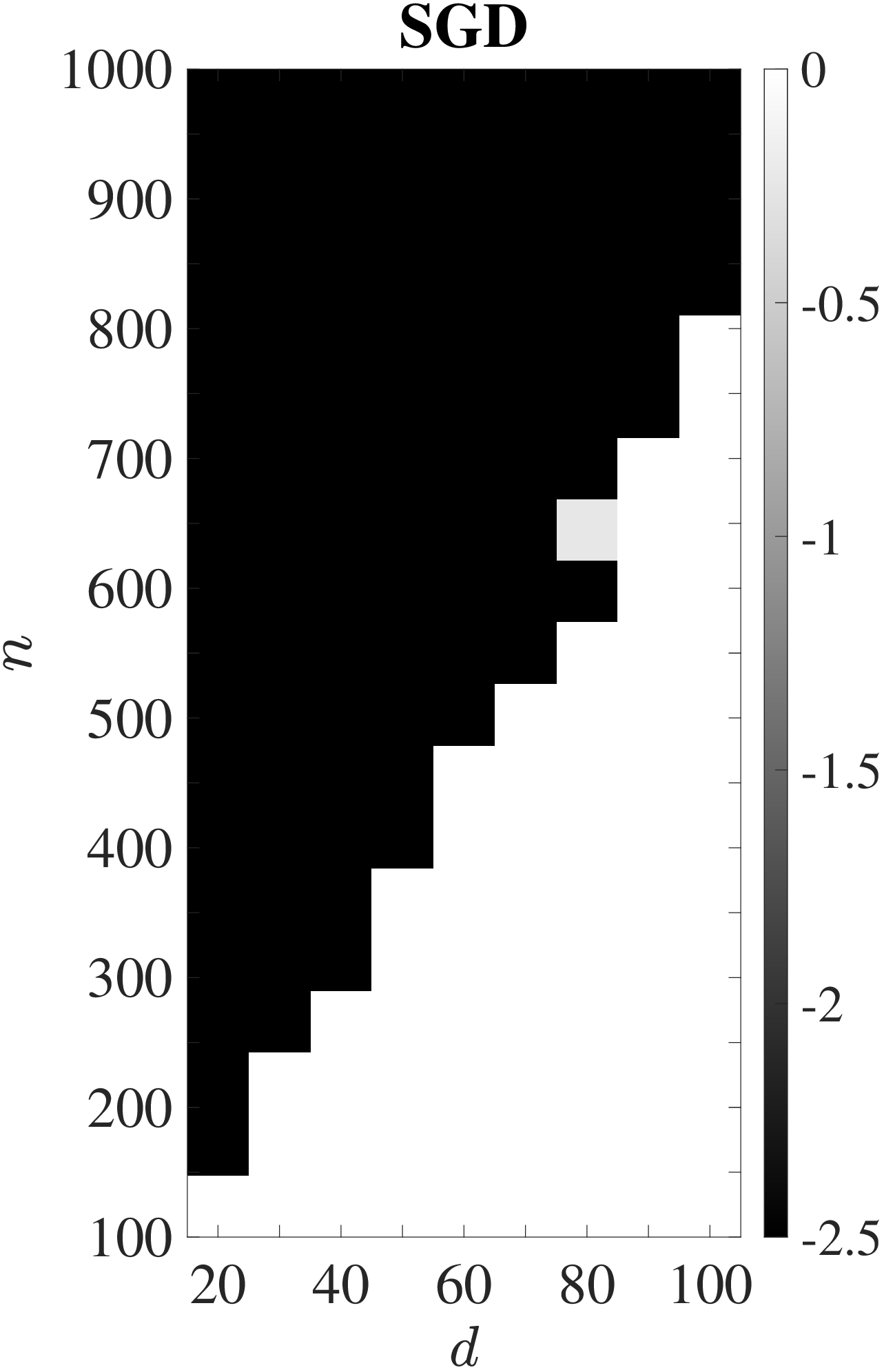}
    \end{minipage}
    \hfill    
    \caption{Phase transition of estimation error per the number of observations $n$ and the ambient dimension $d$ in the noiseless case (The number of linear models $k$ and the batch size $m$ are set to $3$ and $64$, respectively). {  The first row and the second row respectively show the median and the $90$th percentile of estimation errors in $50$ trials.}}
    \label{fig:k3}
\end{figure}


\begin{figure}[H]
    \centering
    \hfill
    \begin{minipage}{0.32\textwidth}
        \centering
        \includegraphics[scale=0.2]{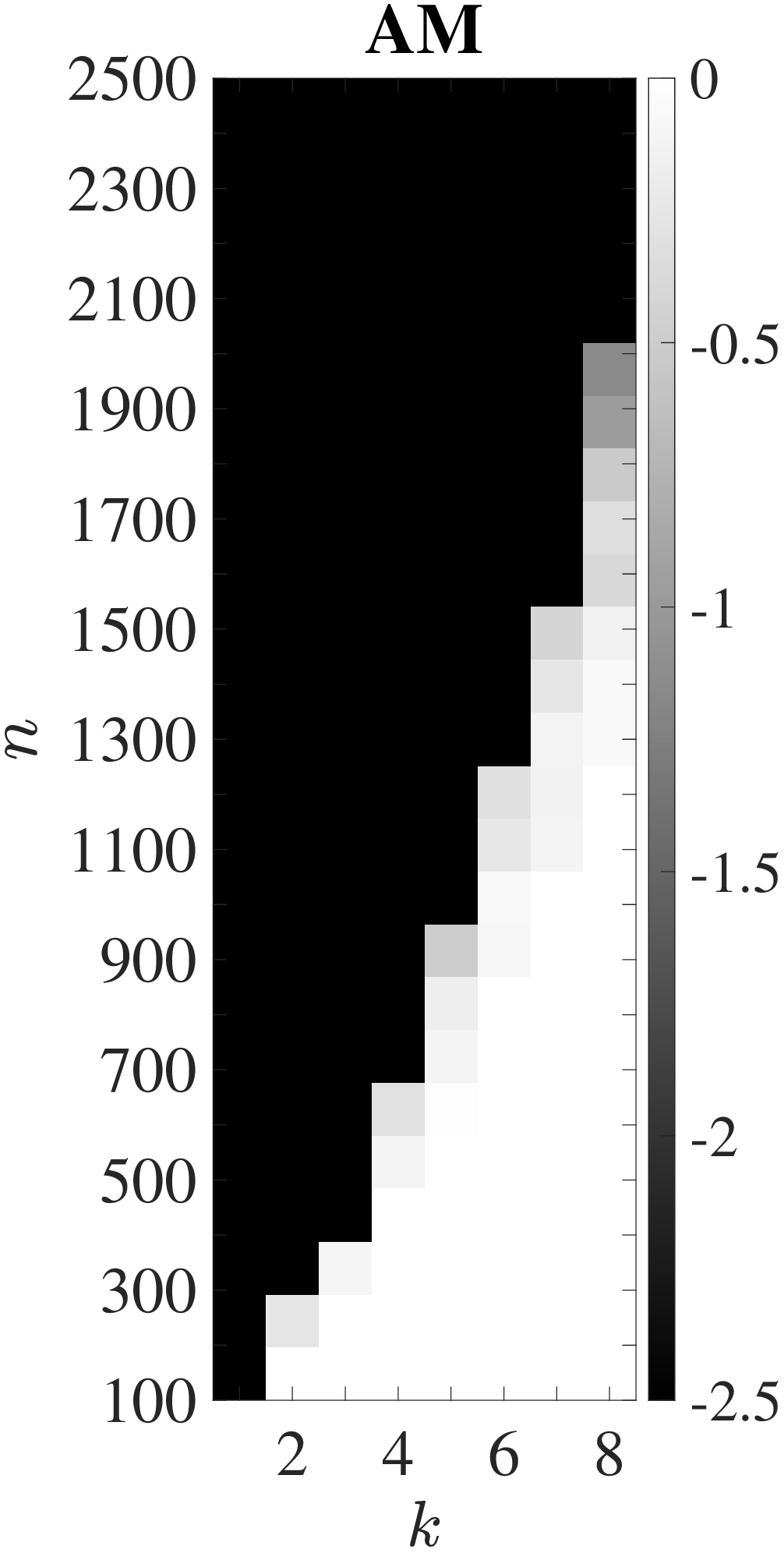}
    \end{minipage}
    \hfill
    \begin{minipage}{0.32\textwidth}
        \centering
        \includegraphics[scale=0.2]{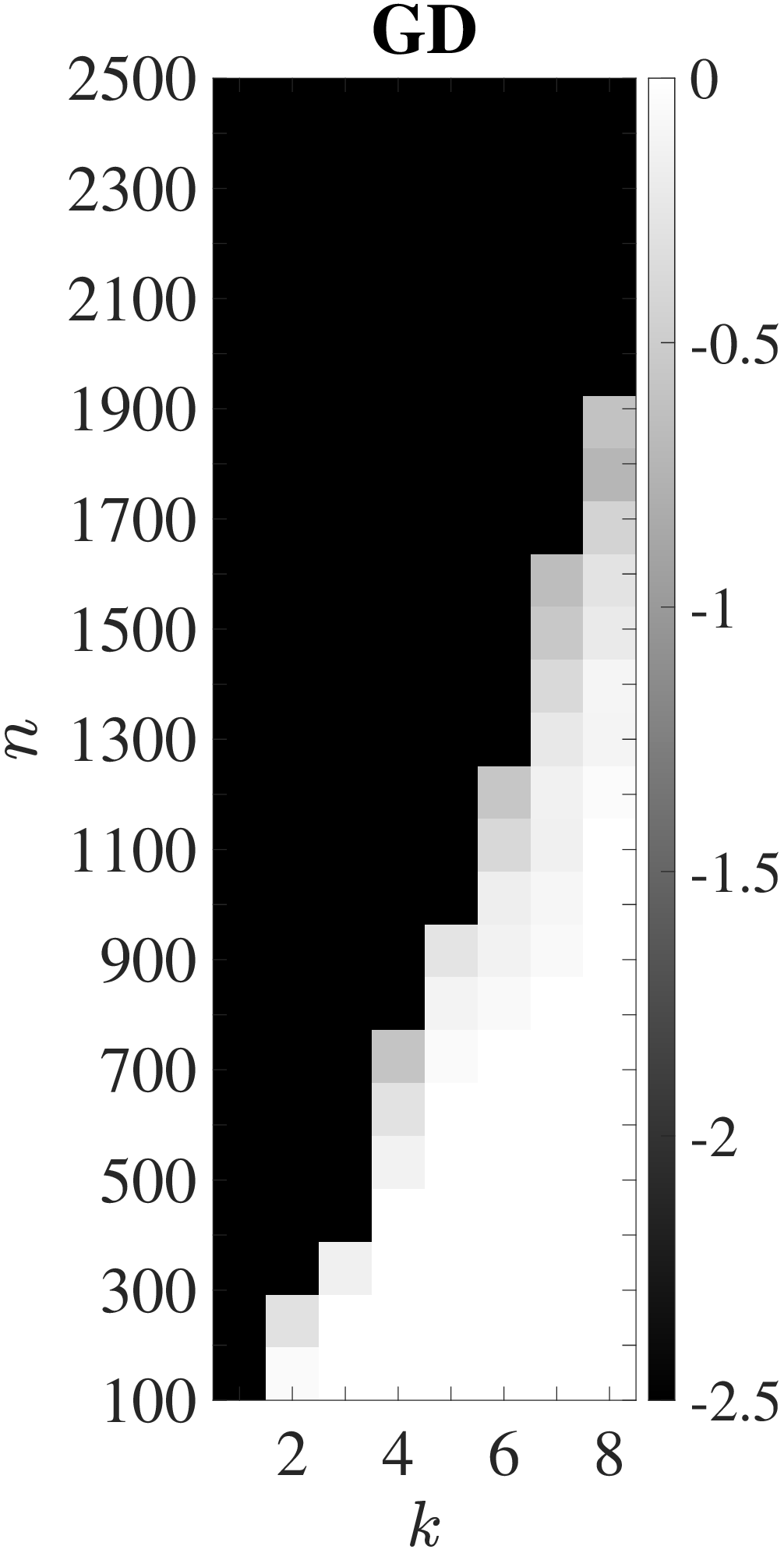}
    \end{minipage}
    \hfill
    \begin{minipage}{0.32\textwidth}
        \centering
        \includegraphics[scale=0.2]{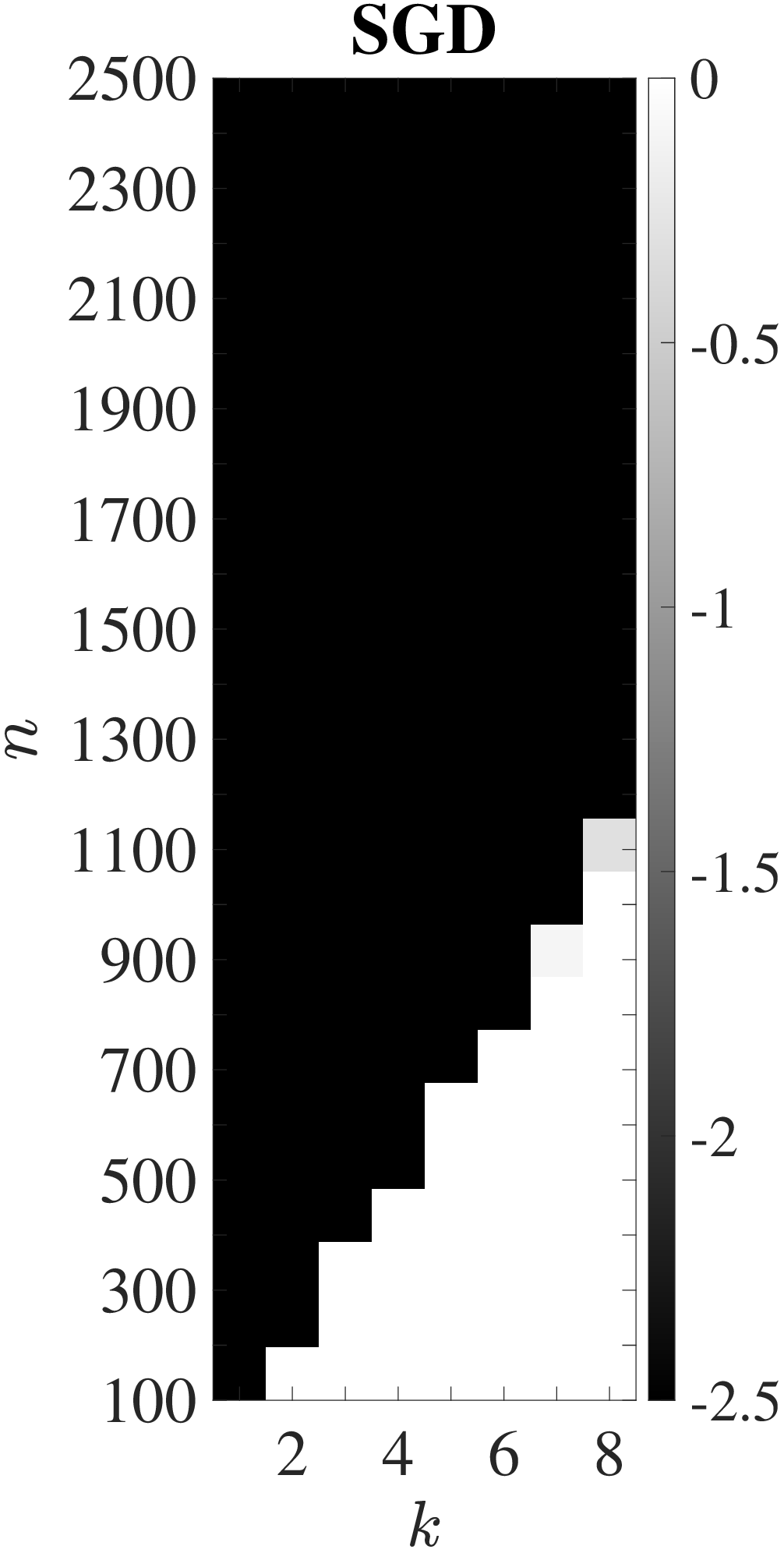}
    \end{minipage}
    \hfill \\
    \vspace{2em}
    \hfill
    \begin{minipage}{0.32\textwidth}
        \centering
        \includegraphics[scale=0.2]{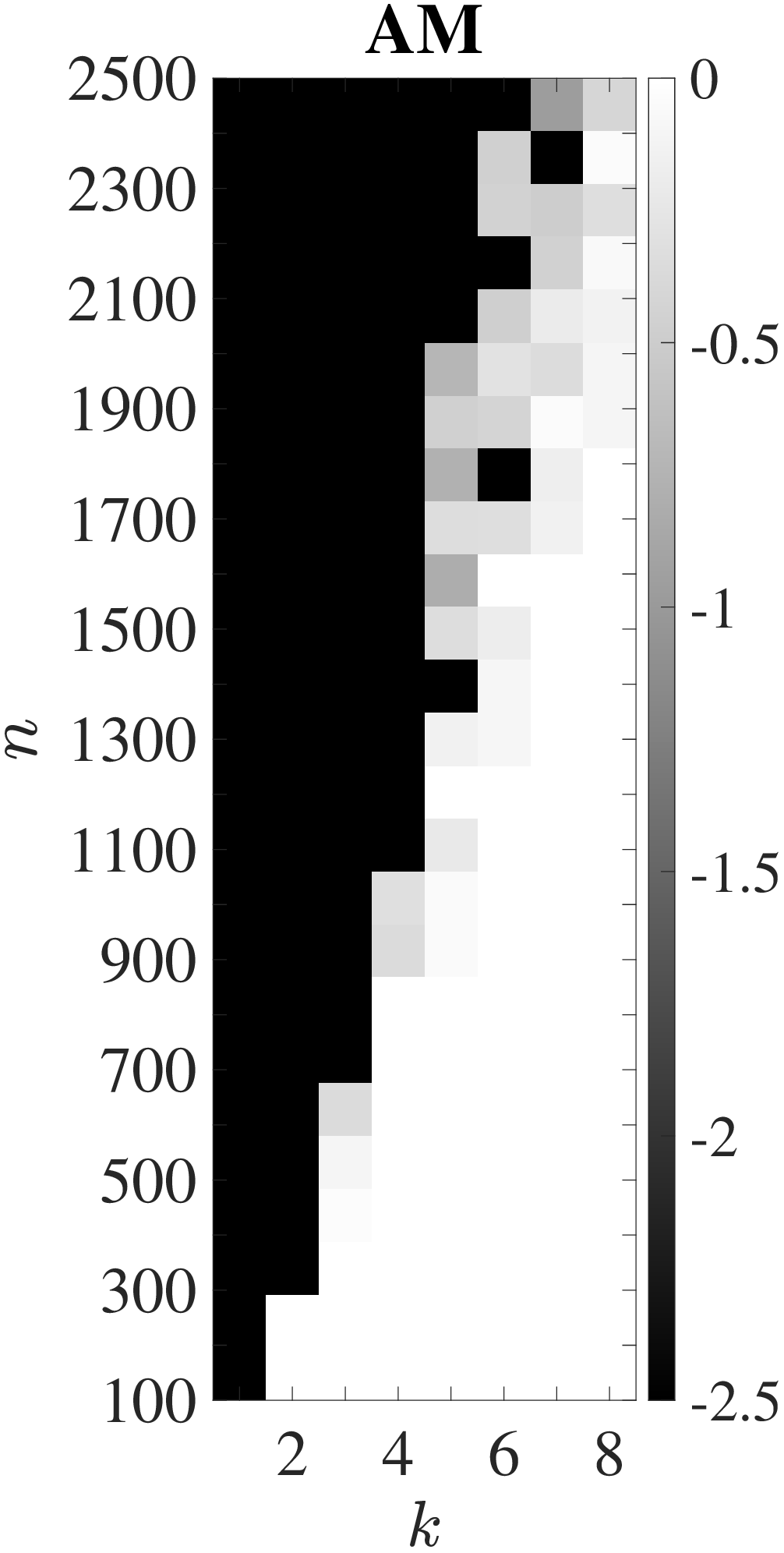}
    \end{minipage}
    \hfill
    \begin{minipage}{0.32\textwidth}
        \centering
        \includegraphics[scale=0.2]{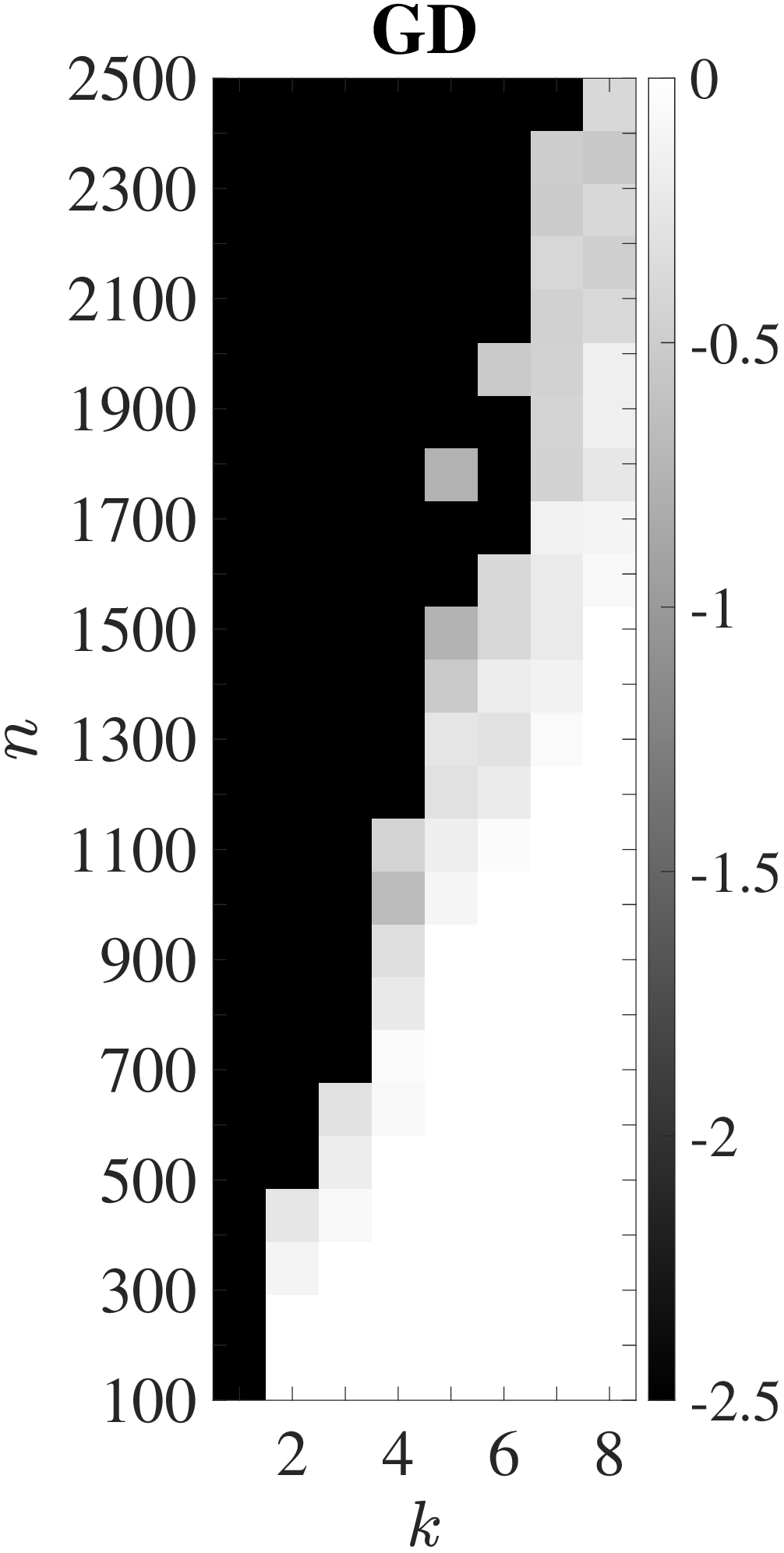}
    \end{minipage}
    \hfill
    \begin{minipage}{0.32\textwidth}
        \centering
        \includegraphics[scale=0.2]{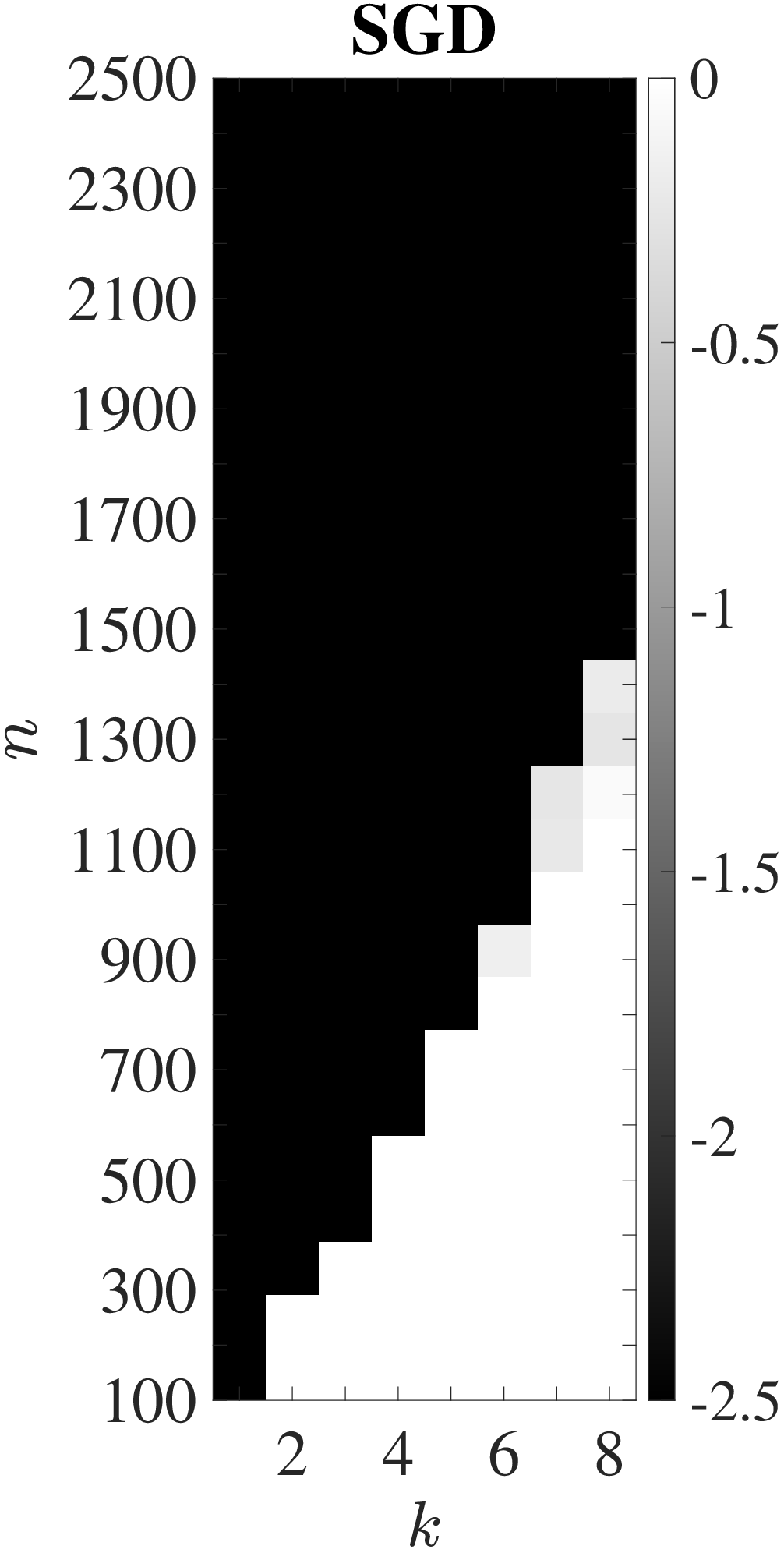}
    \end{minipage}
    \hfill
    \caption{Phase transition of estimation error per number of observations $n$ and number of linear models $k$ in the noiseless case (The ambient dimension $d$ and mini-batch size $m$ are set to $50$ and $64$ respectively).  The first row and the second row respectively show the median and the $90$th percentile of estimation errors in $50$ trials.}
    \label{fig:p50}
\end{figure}


\Cref{fig:k3,fig:p50} illustrate the empirical phase transition by the three estimators through Monte Carlo simulations. 
{The median and the $90$th percentile of $50$ random trials are displayed.} 
In these figures, the transition occurs when the sample size $n$ becomes larger than a threshold that depends on the ambient dimension $d$ and the number of linear models $k$. 
\Cref{fig:k3} shows that the threshold for both estimators increases linearly with $d$ for fixed $k$. 
This observation is consistent with the sample complexity by \Cref{thm:main_noise} and \Cref{thm:main_mini_batch}. 
A complementary view is presented in \Cref{fig:p50} for varying $k$ and fixed $d$. 
The thresholds in \Cref{fig:p50} for GD and SGD are almost linear in $k$ when $d$ is fixed to $50$, which scales slower than the corresponding sample complexity results in \Cref{thm:main_noise} and \Cref{thm:main_mini_batch}. 
A similar discrepancy between theoretical and empirical phase transitions has been observed for AM \cite[Appendix~L]{ghosh2019max}.
We also observe that mini-batch SGD outperforms GD and AM with a lower threshold for phase transition. 
It has been shown that the inherent random noise in the gradient helps the estimator to escape saddle points or local minima \cite{jin2017escape,daneshmand2018escaping}. 
This explains why SGD recovers the parameters with fewer samples than GD. 
{We also note that the relative performance among the three estimators remain similar in both the median and the $90$th percentile. This shows that SGD for noiseless max-affine regression does not suffer from a large variance, which corroborates the result in Corollary~\ref{cor:minibatch_highprob}.}

\begin{figure}[htbp]
    \centering
    \begin{subfigure}[b]{0.45\textwidth}
        \includegraphics[scale=0.35]{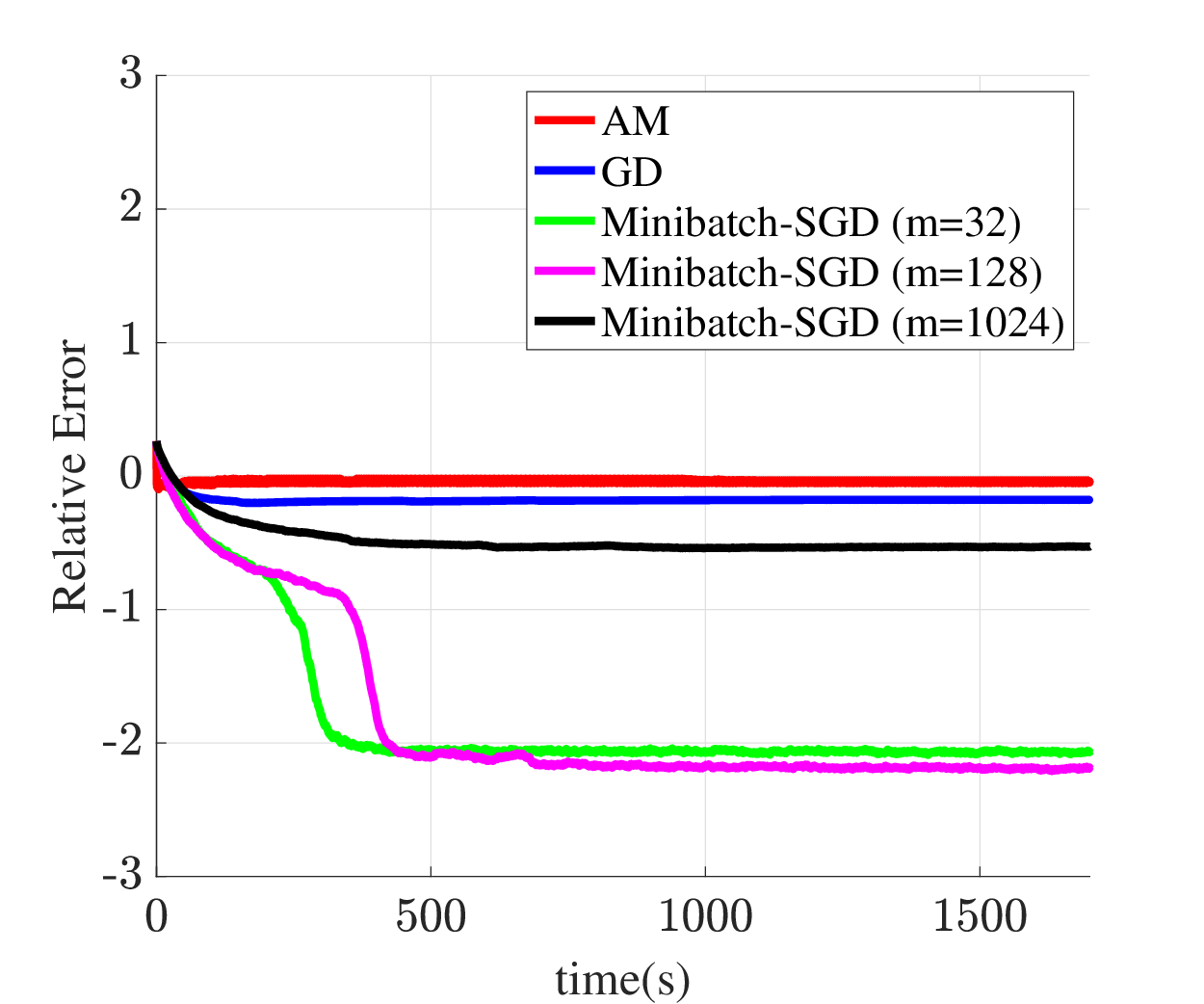}
        \caption{$n=1,500$}
        \label{fig:minibatch}
    \end{subfigure}
    \hfill 
    \begin{subfigure}[b]{0.45\textwidth}
    \includegraphics[scale=0.35]{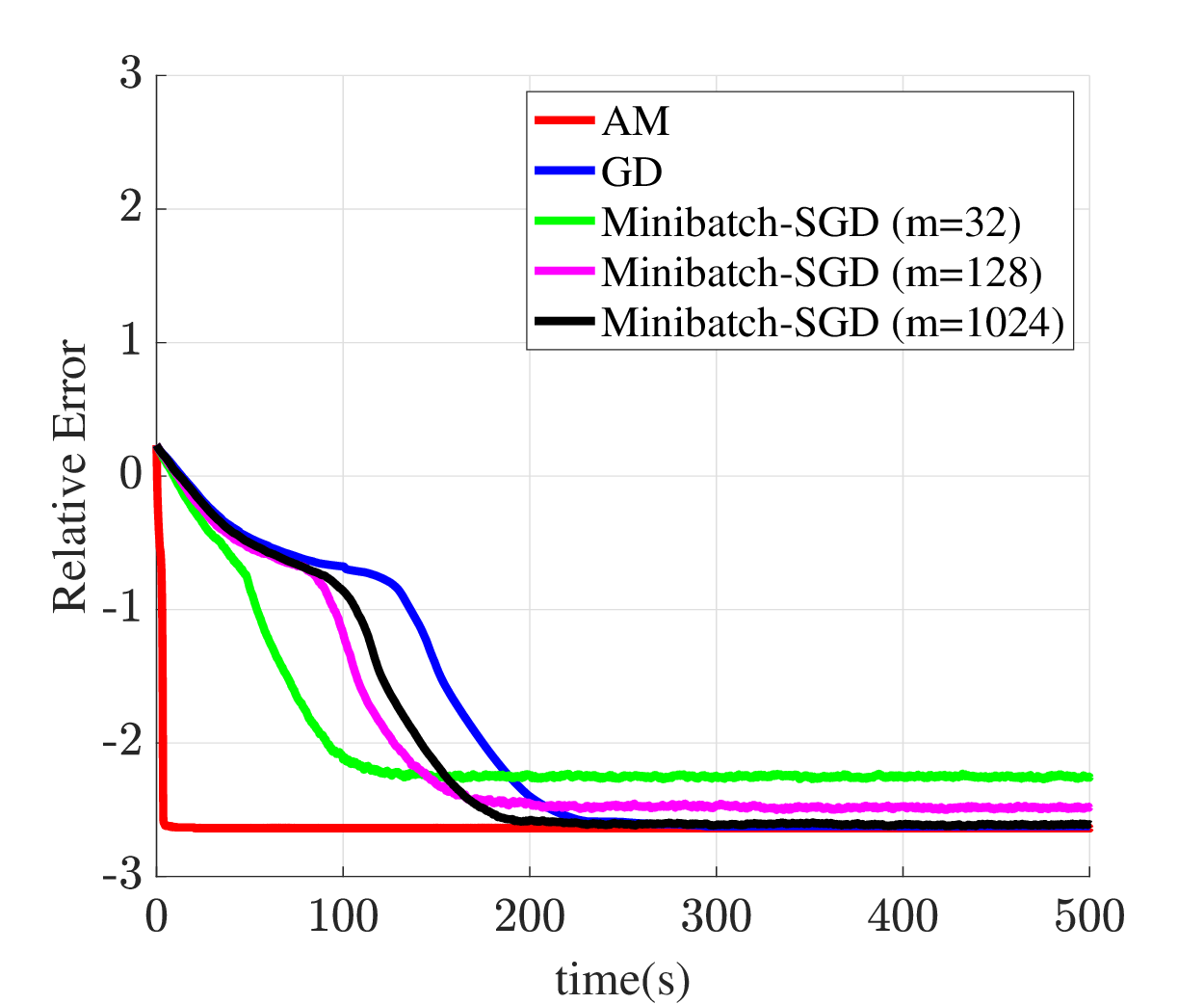}
        \caption{$n=3,000$}
        \label{fig:minibatch2}
    \end{subfigure}
    \caption{Convergence of estimators for max-affine regression under additive white Gaussian noise of variance $\sigma^2=0.01$ ($k=8$ and $d=50$).}
    \label{fig:minibatch_T}
\end{figure}



\begin{figure}[th]
\centering
\includegraphics[scale=0.35]{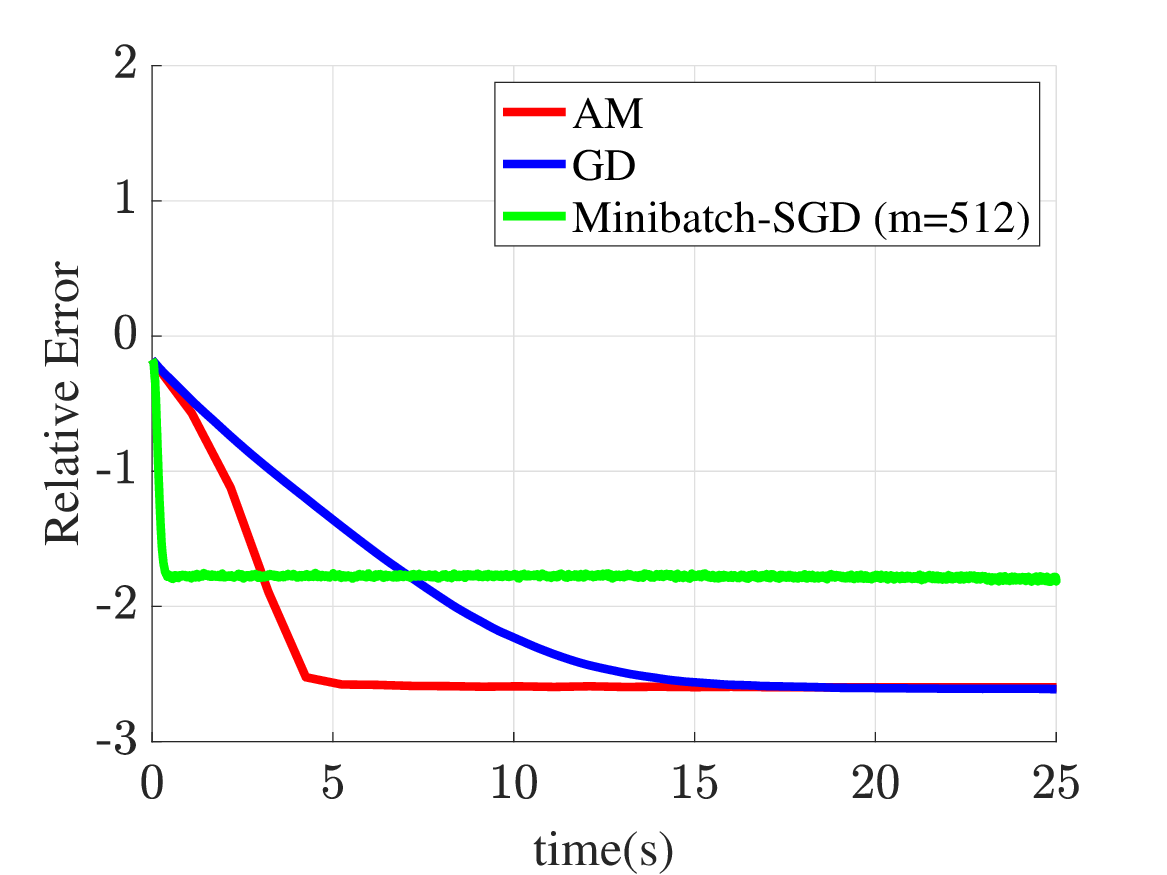}
\caption{Convergence of estimators for max-affine regression under additive white Gaussian noise of variance $\sigma^2=0.01$ ($k=3$, $d=500$, and $n=8,000$).
}
\label{fig:linearconvergence2}
\end{figure}



{\Cref{fig:minibatch_T,fig:linearconvergence2} study the estimation error by mini-batch SGD under zero-mean Gaussian noise with standard deviation $\sigma=0.1$ in three different scenarios. In \Cref{fig:minibatch_T}, we focus on observing how the batch size $m$ affects the convergence speed and the estimation error. 
{\Cref{fig:minibatch} considers the scenario where the spectral method provides a poor initialization due to a small number of observations.} 
Consequently, GD and AM fail to provide a low estimation error. In contrast, mini-batch SGD with a small batch size ($m = 32$ or $m = 128$) relative to the total number of samples ($n = 1,500$) converges to a small estimation error ($< 10^{-2}$). 
{In other words, there exists a trade-off between the convergence speed and the estimation error determined by the batch size $m$.} 
SGD with $m = 128$ converges slower to a smaller error than SGD with $m = 32$. 
This corroborates the theoretical result in \Cref{thm:main_mini_batch}.
However, as the batch size $m$ further increases to $m = 1,024$ close to $n = 1,500$, SGD starts to fail like GD and AM. 
Again, this phenomenon is explained by the fact that the noisy gradient in SGD avoids saddle points and local minima efficiently \cite{jin2017escape,daneshmand2018escaping}. 

\Cref{fig:minibatch2} illustrates the comparison in a high-sample regime where
the number of samples is twice larger than that for \Cref{fig:minibatch}. 
In this case, both GD and AM converge to a smaller error than SGD. 
Moreover, AM converges faster than the other algorithms in the run time, which is explained by the following two reasons.  
First, as discussed in \Cref{sec:performance}, AM converges faster than GD and SGD in the iteration count with a smaller constant for linear convergence. 
{ Second, due to the small ambient dimension ($d=50$), the gain in the per-iteration cost of SGD $O(kmd)$ over that of AM $O(knd^2)$ is not significant. 

Lastly, \Cref{fig:linearconvergence2}, compares the convergence of the estimators in the presence of noise when $d$, $k$, and $n$ are set as in \Cref{fig:linearconvergence}. 
{On one hand, SGD converges faster than AM with a significantly lower per-iteration cost $O(kmd)$ than $O(knd^2)$ due to the large ambient dimension ($d = 500$) and small batch size ($m = 512$ compared to $n = 8,000$). 
On the other hand, SGD yields a larger error than the other two estimators. 
The estimation error bound of SGD by \Cref{thm:main_mini_batch} behaves similarly in this case.}
}

\section{Discussion}
\label{sec:discussion}

We have established local convergence analysis of GD and SGD for max-affine regression under a relaxed covariate model with $\sigma$-sub-Gaussian noise. 
The covariate distribution characterized by the sub-Gaussianity and the anti-concentration generalizes beyond the standard Gaussian model. 
It has been shown that suitably initialized GD and SGD converge linearly below a non-asymptotic error bound, which is comparable to the analogous result on AM. 
Notably, when applied to noiseless max-affine regression, SGD empirically outperforms GD and AM in both sample complexity and convergence speed. 

Under a special case of the Gaussian covariate model, the spectral method by Ghosh et al. \cite{ghosh2021max} can provide the desired initial estimate. 
It is of great interest to extend their theory on the spectral method to the relaxed covariate model. 
Moreover, the extension of the theoretical result on GD and SGD to robust regression, where a subset of samples is corrupted as outliers, is also an intriguing future direction. 


\section*{Acknowledgement}
The authors thank Sohail Bahmani for helpful discussions. 

%% file: appendix.tex
\section{Tools}
\label{sec:tools}

This section collects a set of standard results on concentration inequalities, which will be used in the proofs of Theorem~\ref{thm:main_noise}. 
The following lemma provides the concentration of extreme singular values of sub-Gaussian matrices. 

\begin{lemma}[{\cite[Theorem~4.6.1]{vershynin2018high}}]
\label{lem:expect0}
Let $\{\mb x_i\}_{i=1}^n$ be independent isotropic $\eta$-sub-Gaussian random vectors in $\mathbb{R}^d$. Then there exists an absolute constant $C > 0$ such that 
\[
\P\left( \norm{\frac{1}{n}\sum_{i=1}^{n} \mb x_i \mb x_i^\top - \mb I_p} > \eta^2 \max(\epsilon,\epsilon^2) \right) \leq \delta \quad \text{where} \quad \epsilon = \sqrt{\frac{C(d + \log(2/\delta))}{n}}.
\]
\end{lemma}

\begin{remark}
It has been shown that Lemma~\ref{lem:expect0} continues to hold when $\mb x_i$ is substituted by $\mb \xi = [\mb x_i;\ 1]$ \cite{ghosh2019max}. Indeed, multiplying a random sign to the last coordinate of $\mb \xi_i$ does not modify the outer product $\mb \xi_i \mb \xi_i^\top$ whereas $\mb \xi_i$ remains a sub-Gaussian vector.     
\end{remark}

\noindent Furthermore, we also use the results from the standard Vapnik–Chervonenkis (VC) theory stated in the following lemmas. 
\begin{lemma}[{\cite[Theorem~2]{vapnik2015uniform}}]
\label{lem:uniform}
Let $\mathcal{V}$ be a collection of subsets of a set $\mathcal{X}$ and $\{\mb x_i\}_{i=1}^n$ be $n$ independent copies of a random variable $\mb x \in \mathcal{X}$. 
Then it holds for all $\epsilon>0$ and $n\geq2/\epsilon^2$ that
\[
\P\left(\sup_{V\in\mathcal{V}}\left|\frac{1}{n}\sum_{i=1}^{n}\bbone_{\{\mb x_i\in V\}}-\P(\mb x\in V)\right|\geq \epsilon\right)\leq4\Pi_{\mathcal{V}}(2n)\exp(-n\epsilon^2/16),
\]
where $\Pi_{\mathcal{V}}(n)$ denotes the growth function defined by
\begin{equation*}
\label{eq:def_growthfunction}
\Pi_{\mathcal{V}}(n):=\max_{\mb x_1,\ldots,\mb x_n\in\mathcal{X}}\left|\left\{\left(\bbone_{\{\mb x_1\in V\}},\ldots,\bbone_{\{\mb x_n\in V\}}\right):V\in \mathcal{V}\right\}\right|.
\end{equation*}
\end{lemma}

\begin{lemma}[{\cite[Corollary~3.18]{mohri2018foundations}}]
\label{lem:growthfunction}
Let $\mathcal{V}$ be a collection of subsets having VC dimension $d$. Then, for all $n\geq d$, the growth function of $\mathcal{V}$ is upper-bounded by
\[
\Pi_{\mathcal{V}}(n)\leq\left(\frac{en}{d}\right)^d.
\]
\end{lemma}

\noindent The VC dimension of the $k$-fold intersection has been known in the literature (e.g. see \cite{blumer1989learnability}). 
We will use the following lemma for the result for the intersection of size two. 
Since it was given as an exercise in \cite{mohri2018foundations}, we provide a proof for the sake of completeness. 

\begin{lemma}[{\cite[Equation~(3.53)]{mohri2018foundations}}]
\label{lem:intersection}
Let $\mathcal{V}$ and $\mathcal{W}$ be collections of subsets of a common set. Then their intersection given by $\mathcal{V}\cap\mathcal{W}:=\left\{V\cap W : V\in\mathcal{V},\,W\in\mathcal{W}\right\}$ satisfies that 
\[
\Pi_{\mathcal{V}\cap\mathcal{W}}(n)\leq \Pi_{\mathcal{V}}(n) \Pi_{\mathcal{W}}(n), \quad \forall n \in \mathbb{N}.
\]
\end{lemma}
\begin{proof}
For any $V \cap W \in \mathcal{V} \cap \mathcal{W}$, we have
\begin{align*}
\left(\bbone_{\{\mb x_1\in V\cap W\}},\ldots,\bbone_{\{\mb x_n\in V\cap W\}}\right)
=\left(\bbone_{\{\mb x_1\in V\}},\ldots,\bbone_{\{\mb x_n\in V\}}\right)\odot\left(\bbone_{\{\mb x_1\in W\}},\ldots,\bbone_{\{\mb x_n\in W\}}\right),
\end{align*} 
where $\odot$ denotes the pointwise product. 
Therefore, the claim follows from the definition of the growth function. 
\end{proof}

\begin{lemma}
\label{lem:bnd_growth_C}
Let $\mathcal{P}_k$ be the collection of all polytopes constructed by the intersection of $k$ half spaces in $\mathbb{R}^d$. 
Then the growth function of $\mathcal{P}_k$ satisfies
\begin{equation}
\label{eq:bnd_growth_intersec}
\Pi_{\mathcal{P}_k}(n)\leq\left(\frac{en}{d+1}\right)^{k(d+1)}.
\end{equation}
\end{lemma}

\begin{proof}
Let $\mathcal{H}_j$ be the collection of all half spaces in $\mathbb{R}^d$ for $j \in [k]$. 
Then, by the construction of $\mathcal{P}_k$, we have $\mathcal{P}_k=\cap_{j=1}^k \mathcal{H}_j$.
Therefore, by inductive application of \Cref{lem:intersection}, the growth function of $\mathcal{P}_k$ satisfies
\begin{equation}
\label{eq:bound_growthfunction}
\Pi_{\mathcal{P}_k}(n)\leq\prod_{j=1 }^k\Pi_{\mathcal{H}_{j}}(n).
\end{equation}
Furthermore, since the VC dimensions of half spaces in $\mathbb{R}^d$ is $d+1$ (e.g. see \cite[Section~3]{mohri2018foundations}), \Cref{lem:growthfunction} implies
\begin{equation}
\label{eq:bnd_H}
\Pi_{\mathcal{H}_{j}}(n)\leq\left(\frac{en}{d+1}\right)^{d+1}, \quad \forall j \in [k].
\end{equation} 
The assertion is obtained by plugging in \eqref{eq:bnd_H} into \eqref{eq:bound_growthfunction}.
\end{proof}

\noindent Finally, the following corollary is a direct consequence of Lemmas~\ref{lem:uniform}, \ref{lem:growthfunction}, and \ref{lem:intersection}. 

\begin{corollary}
\label{cor:concent_emp_prob}
Let $\delta \in (0,1)$ and $\mathcal{P}_k$ be the collection of all polytopes constructed by the intersection of $k$ half-spaces in $\mathbb{R}^d$. 
Suppose that $\{\mb x_i\}_{i=1}^n$ are independent copies of a random vector $\mb x \in \mathbb{R}^d$. 
Then it holds with probability at least $1-\delta$ that
\begin{align}
\sup_{Z\in\mathcal{P}_{k}}\left|\frac{1}{n}\sum_{i=1}^{n}\bbone_{\{\mb x_i\in Z\}}-\P(\mb x\in Z)\right|
\leq 
4\sqrt{\frac{\log(4/\delta)+2k(d+1)\log(2en/(d+1))}{n}}. 
\end{align}
\end{corollary}

\section{Supporting lemmas}

In this section, we list lemmas to prove Theorem~\ref{thm:main_noise}. These lemmas are borrowed from \cite{tan2019phase} and \cite{ghosh2019max}. We improve on a subset of these results derived with a streamlined proof. 

\subsection{Worst-case extreme eigenvalues of partial sum of outer products of covariates}
\label{sec:worst-case}

A partial sum of the outer products of covariates, $\sum_{i \in \mathcal{I}} \mb \xi_i \mb \xi_i^\top$ appears frequently in the proof. 
The summation indices in $\mathcal{I}$ often depend on covariates. 
The following lemma by Tan and Vershynin \cite{tan2019phase} provides a tail bound on the worst-case largest eigenvalue of $\sum_{i \in \mathcal{I}} \mb \xi_i \mb \xi_i^\top$ when the cardinality of $\mathcal{I}$ is bounded from above. 
\begin{lemma}[{\cite[Theorem~5.7]{tan2019phase}}]
\label{lem:upperbound}
Let $\delta\in(0,1/e)$, $\alpha\in(0,1)$, and $\mb \xi_{i}=[\mb x_i,1]\in\mathbb{R}^{d+1}$ for $i\in[n]$.  Suppose that Assumption~\ref{main_assumption} holds. 
Then it holds with probability at least $1-\delta$ that 
\[
\sup_{\mathcal{I}:|\mathcal{I}|\leq \alpha n}\lambda_{1}\left(\sum_{i\in\mathcal{I}}\mb \xi_{i}\mb \xi_{i}^\T\right)
\leq C_4 (\eta^2\vee1)\sqrt{\alpha} n
\] 
for some absolute constant $C_4 > 0$, provided 
\begin{equation}
\label{eq:samp_tanver}
n\geq \left(d\vee\frac{\log(1/\delta)}{\alpha}\right).
\end{equation}
\end{lemma}

\begin{remark}
In the original result, Tan and Vershynin assumed that $\{\mb \xi_i\}_{i=1}^n$ are isotropic $\eta$-sub-Gaussian random vectors \cite[Theorem~5.7]{tan2019phase}. 
Later, Ghosh et al. \cite{ghosh2019max} showed that the result also applies to the setting in Lemma~\ref{lem:upperbound} through the following argument. 
The outer product $\mb \xi_i \mb \xi_i^\top$ remains the same as one multiplies a random sign to the last entry of $\mb \xi_i$ which makes the random vector $\tilde{\eta}$-sub-Gaussian with $\tilde{\eta} = \max(\eta,1)$. 
\end{remark}

\noindent Moreover, Ghosh et al. also derived analogous lower tail bound on the smallest eigenvalue when the index set $\mathcal{I}$ exceeds a threshold \cite[Lemma~7]{ghosh2019max}. 
Their proof strategy adopted an epsilon-net approximation and a union bound argument. 
Our lemma below, derived by using the small-ball method \cite{koltchinskii2015bounding}, provides a streamlined proof and a sharper bound. 

\begin{lemma}
\label{lem:lwb_trunc}
Let $\alpha, \delta \in(0,1)$ and $\mb \xi_{i}=[\mb x_i,1]\in\mathbb{R}^{d+1}$ for $i\in[n]$.  
Suppose that Assumption~\ref{main_assumption2} holds. 
Then there exists an absolute constant $C > 0$ such that if 
\begin{equation}
\label{eq:lwb_trunc_sample}
n\geq C\alpha^{-2}(d\log(n/d)\vee\log(1/\delta))
\end{equation}
then it holds with probability at least $1-\delta$ that 
\begin{equation}
\label{lem:sb:lb}
\inf_{\mc I \subset [n] : |\mc I|\geq\alpha n} \lambda_{d+1}\left( \sum_{i\in \mc I} \mb \xi_i \mb \xi_i^\top \right) \geq \frac{2n}{\gamma}\left(\frac{\alpha}{4}\right)^{1+\zeta^{-1}}.
\end{equation}
\end{lemma}

\noindent {
We compare Lemma~\ref{lem:lwb_trunc} to the previous result by Ghosh et al. \cite[Lemma~$7$]{ghosh2019max} when the parameter $\gamma$ is treated as a fixed constant. 
They demonstrated that the worst-case minimum eigenvalue in the left-hand side of \eqref{lem:sb:lb} satisfies $\Omega(n \alpha^{1+2\zeta^{-1}})$ if $n \geq \alpha^{-1} \max(4p, \zeta^{-1}(d+1))$. 
On one hand, their requirement in the sample complexity is less stringent than that in \eqref{eq:lwb_trunc_sample}. 
On the other hand, the lower bound in \eqref{lem:sb:lb} is tighter than theirs by a factor of $\alpha^{\zeta^{-1}}$. 
When these two results are applied to derive Theorem~\ref{thm:main_noise} with $\alpha$ substituted by $\pi_{\min}$, the resulting sample complexity $\tilde{O}(\pi_{\min}^{-4{(1+\zeta^{-1})}}d)$ by Lemma~\ref{lem:lwb_trunc} is smaller than $\tilde{O}(\pi_{\min}^{-4(1+2\zeta^{-1})}d)$ by \cite[Lemma~$7$]{ghosh2019max}. 
The gain due to Lemma~\ref{lem:lwb_trunc} is $\pi_{\min}^{-4\zeta^{-1}}$, which is no less than $k^{4\zeta^{-1}}$. 
For example, if the covariates are Gaussian $\zeta = 1/2$, then the gain is $k^8$. 
}



\begin{proof}
Let $T > 0$ be an arbitrarily fixed threshold. 
If 
\begin{equation}
\label{eq:condition1}
N(\mb v) :=\sum_{i=1}^{n}\bbone_{\{\langle \mb \xi_i, \mb  v\rangle^2>T\}} > n-\frac{\alpha n}{2}
\end{equation}
then it follows that
\begin{equation*}
\label{eq:mainarg2}
\frac{1}{n} \sum_{i\in \mc I} \langle\mb \xi_i,\mb v\rangle^2
\geq \frac{\alpha T}{2}, \quad \forall \mc I \subset [n] : |\mc I| \geq \alpha n.
\end{equation*}
Therefore, it suffices to show that \eqref{eq:condition1} holds for all $\mb v \in \mathbb{S}^{d}$ with probability $1-\delta$. 
Let $\mathcal{H}$ denote the collection of half-spaces in $\mathbb{R}^d$ given by $\{\mb x\in\mathbb{R}^d : \mb x^\T\mb u>\sqrt{T}-w\}$ for all $\mb v = [\mb u ;\ w] \in \mathbb{S}^d$. 
Since the VC dimension of all half-spaces in $\mathbb{R}^d$ is at most $d+1$, by Lemmas \ref{lem:uniform} and \ref{lem:growthfunction}, it holds with probability at least $1-\delta/2$ that 
\begin{equation}
\label{eq:VCarg1}
\frac{1}{n}N(\mb v)\geq \frac{1}{n}\E N(\mb v)-C' \sqrt{\frac{d\log(n/d)+\log(1/\delta)}{n}}, \quad \forall \mb v \in \mathbb{S}^{d}, 
\end{equation}
where $C' > 0$ is an absolute constant. 

Moreover, it follows from {Assumption~\ref{main_assumption2}} that 
\begin{equation}
\label{eq:smallball1}
\frac{1}{n} \E N(\mb v) = \P\left(|\langle\mb x,\mb u\rangle+w|^2>T\right)\geq 1-\left(T\gamma\right)^\zeta.   
\end{equation} 
By plugging in \eqref{eq:smallball1} into \eqref{eq:VCarg1}, we obtain that
\[
\frac{1}{n} N(\mb v)\geq 1-(T\gamma)^\zeta-C'\sqrt{\frac{d\log(n/d)+\log(1/\delta)}{n}}, \quad \forall \mb w \in \mathbb{S}^{d}.
\]
Then \eqref{eq:condition1} is satisfied for all $\mb v \in \mathbb{S}^{d}$ when $T=\frac{1}{\gamma}\left(\frac{\alpha}{4}\right)^{\zeta^{-1}}$ and $C = (4C')^2$. 
This completes the proof.
\end{proof}

\subsection{Local estimates}

In this section, we present local tail bounds which arise in the proof of the main result. 
The following lemma, obtained as a direct consequence of the triangle inequality and the definition of $\kappa$ in \eqref{eq:defkappa}, provides a basic inequality that will be used frequently throughout this section. 

\begin{lemma}
\label{lem:triangle_claim}
Suppose that $\mb \beta \in \mathcal{N}(\mb \beta^\star)$, where $\mathcal{N}(\mb \beta^\star)$ is defined as in \eqref{eq:defnbr}. Then we have
\[
\|(\mb \beta_j-\mb \beta_{j'}) - (\mb \beta_j^\star - \mb \beta_{j'}^\star)\|_2
\leq 
2\rho \|(\mb \beta_j^\star - \mb \beta_{j'}^\star)_{1:d}\|_2, \quad \forall j\neq j' \in [k].
\]
\end{lemma}

\begin{proof}
Since $\mb \beta\in \mathcal{N}(\mb \beta^\star)$, by the triangle inequality, we have
\[
\|(\mb \beta_j-\mb \beta_{j'}) - (\mb \beta_j^\star - \mb \beta_{j'}^\star)\|_2
\leq 
\|\mb \beta_j-\mb \beta_j^\star\|_2+\|\mb \beta_{j'}-\mb \beta_{j'}^\star\|_2
\leq 
2\kappa\rho,\quad\forall j,j'\in [k].
\]
Furthermore, it follows from the definition of $\kappa$ in \eqref{eq:defkappa} that 
\[
\kappa \leq \| (\mb \beta_j^\star - \mb \beta_{j'}^\star)_{1:d} \|_2, \quad \forall j\neq j' \in [k].
\]
Then the assertion follows. 
\end{proof}

\noindent We also use the following lemma by Ghosh et al. \cite{ghosh2019max}, which is a consequence of Assumptions~\ref{main_assumption} and \ref{main_assumption2} respectively for the sub-Gaussianity and anti-concentration. 

\begin{lemma}[{\cite[Lemma~17]{ghosh2019max}}]
\label{lem:bnd_vol}
Suppose that $\mb x \in \mathbb{R}^d$ satisfies Assumptions~\ref{main_assumption} and \ref{main_assumption2}. 
If 
\[
\|\mb v-\mb v^\star\|_2 \leq \frac{1}{2} \|(\mb v^\star)_{1:d}\|_2,
\]
then 
\[
\P\left( 
\langle [\mb x; 1],\mb v^\star\rangle^2\leq\langle [\mb x; 1],\mb v-\mb v^\star\rangle^2
\right)
\lesssim \left( \left(\frac{\|\mb v-\mb v^\star\|_2}{\|(\mb v^\star)_{1:d}\|_2}\right)^2 \cdot \log\left(\frac{2\|(\mb v^\star)_{1:d}\|_2}{\|\mb v-\mb v^\star\|_2}\right)\right)^\zeta.
\]
\end{lemma}

Intuitively, when the parameter vector $\mb \beta$ belongs to a small neighborhood of the ground-truth, the partition sets $\left(\mathcal{C}_j\right)_{j=1}^k$ by $\mb \beta$ and $\left(\mathcal{C}_j^\star\right)_{j=1}^k$ by the ground-truth $\mb \beta^\star$ will be similar. 
The next lemmas quantify the empirical measure on the event of $\mb x \in \mathcal{C}_j\cap\mathcal{C}_{j'}^\star$ for distinct indices $j$ and $j'$, and quadratic forms given as a partial summation indexed by the indicator functions on this event. 

\begin{lemma}
\label{lem:lwb_number_ball}
Let $\left(\mathcal{C}_j\right)_{j=1}^k$ and $\left(\mathcal{C}_j^\star\right)_{j=1}^k$ be defined as in \eqref{eq:def_calCj} and \eqref{eq:def_calCjstar} respectively by $\mb \beta$ and $\mb \beta^\star$. 
Furthermore, let $\pi_{\min}$ be defined as in \eqref{eq:def_pimin_pimax} by $\mb \beta^\star$. 
Suppose that $\mb x \in \mathbb{R}^d$ and $\{\mb x_i\}_{i=1}^n$ satisfy Assumptions~\ref{main_assumption} and \ref{main_assumption2}, and that the parameter $\rho$ of $\mathcal{N}(\mb \beta^\star)$ in \eqref{eq:defnbr} satisfies \eqref{eq:choice_rho1}
for some numerical constant $R >0$. 
Then there exists an absolute constant $C$ such that if 
\begin{equation}
\label{eq:sample_intersection}
n\geq C \pi_{\min}^{-2}\cdot\left(kd\log(n/d)\vee\log(1/\delta)\right)
\end{equation}
then with probability at least $1-\delta$ \begin{equation}
\label{eq:statement}
\begin{aligned}
\frac{1}{n}\sum_{i=1}^{n}\bbone_{\{\mb x_i\in \mathcal{C}_j\cap \mathcal{C}_j^\star\}}
\geq\frac{\pi_{\min}}{4}
\end{aligned}
\end{equation}
holds for all $j \in [k]$, $\mb \beta \in \mathcal{N}(\mb \beta^\star)$, and $\mb \beta^\star \in \mathbb{R}^{d+1}$. 
\end{lemma}

\begin{proof} 
Note that the left-hand side of \eqref{eq:statement} is an empirical measure on the event $\mb x\in \mathcal{C}_{j} \cap \mathcal{C}_{j}^\star$. 
We first derive a lower bound on its expectation, which is written as
\begin{align}
\P\left(\mb x\in \mathcal{C}_{j}, \mb x \in \mathcal{C}_{j}^\star\right)
&= \P\left(\mb x\in \mathcal{C}_{j} | \mb x \in \mathcal{C}_{j}^\star\right) 
\cdot \P\left(\mb x \in \mathcal{C}_{j}^\star\right) \nonumber \\
&= \left(1 - \P\left(\mb x \not\in \mathcal{C}_{j} | \mb x \in \mathcal{C}_{j}^\star\right) \right)
\cdot \P\left(\mb x \in \mathcal{C}_{j}^\star\right).\label{eq:sbreslt1}
\end{align}
Then, by the construction of $\left(\mathcal{C}_j\right)_{j=1}^k$ in \eqref{eq:def_calCj}, we have
\begin{align*}
& \P\left(\mb x \not\in \mathcal{C}_{j} | \mb x \in \mathcal{C}_{j}^\star\right) \nonumber \\
&= \frac{\P(\mb x \not\in \mathcal{C}_{j}, \mb x \in \mathcal{C}_{j}^\star)}{\P(\mb x \in \mathcal{C}_{j}^\star)} \nonumber\\
&\leq \frac{1}{\P(\mb x \in \mathcal{C}_{j}^\star)} \sum_{j' \neq j} \P\left(\langle [\mb x;\ 1], \mb \beta_{j'} \rangle \geq \langle [\mb x;\ 1], \mb \beta_{j} \rangle, \langle [\mb x;\ 1], \mb \beta_{j}^\star \rangle \geq \langle [\mb x;\ 1], \mb \beta_{j'}^\star \rangle\right) \nonumber\\
&\leq \frac{1}{\P(\mb x \in \mathcal{C}_{j}^\star)} \sum_{j' \neq j} \P\left(\langle [\mb x;\ 1], \mb v_{j,j'} \rangle \langle [\mb x;\ 1], \mb v_{j,j'}^\star \rangle \leq 0\right) \\
&\leq \frac{1}{\P(\mb x \in \mathcal{C}_{j}^\star)} \sum_{j' \neq j}  \P\left( 
\langle [\mb x; 1],\mb v_{j,j'}^\star\rangle^2\leq\langle [\mb x; 1],\mb v_{j,j'}-\mb v_{j,j'}^\star\rangle^2
\right),
\end{align*}
where the second inequality holds since $\mb v_{j,j'} = \mb \beta_j - \mb \beta_{j'}$ and $\mb v_{j,j'}^\star = \mb \beta_j^\star - \mb \beta_{j'}^\star$, and the last inequality follows from the fact that $ab \leq 0$ implies $|b| \leq |a-b|$ for $a,b \in \mathbb{R}$. 
Recall that $\mb \beta\in\mathcal{N}(\mb \beta^\star)$ implies $\|\mb v_{j,j'} - \mb v_{j,j'}^\star\|_2 \leq 2\rho \|(\mb v_{j,j'}^\star)_{1:d}\|_2$ due to Lemma~\ref{lem:triangle_claim}. 
Furthermore, one can choose the numerical constant $R > 0$ in \eqref{eq:choice_rho1} sufficiently small (but independent of $k$ and $p$) so that $2\rho \leq 0.1$. 
Then it follows that 
\begin{align}
\P(\mb x\not\in \mathcal{C}_{j'}\,|\,\mb x\in \mathcal{C}_{j'}^\star)
&\overset{\mathrm{(i)}}{\lesssim} \frac{k}{\P(\mb x\in \mathcal{C}_j^\star)} \left(\frac{ \|\mb v_{j,j'}-\mb v_{j,j'}^\star\|_2^2}{\|(\mb v_{j,j'}^\star)_{1:d}
\|_2^2}\log\left(\frac{2\|(\mb v_{j,j'}^\star)_{1:d}\|_2}{\|\mb v_{j,j'}-\mb v_{j,j'}^\star\|_2}\right)\right)^{\zeta}\nonumber\\
&\overset{\mathrm{(ii)}}{\leq} \frac{k}{\P(\mb x\in \mathcal{C}_j^\star)} \left( (2\rho)^2 \log\left(\frac{1}{\rho}\right)\right)^{\zeta}\nonumber\\
&\overset{\mathrm{(iii)}}{\leq}\frac{k}{\pi_{\min}} {\left(\frac{ R^2 \pi_{\min}^{2\zeta^{-1}(1+\zeta^{-1})}}{k^{2\zeta^{-1}}}\right)^{\zeta}} \nonumber \\
&\leq
\frac{R^{2\zeta} \pi_{\min}^{1+2\zeta^{-1}}}{k},\label{eq:subreslt0}
\end{align}  
where (i) follows from Lemma~\ref{lem:bnd_vol}; (ii) holds since $a \log^{1/2}(2/a)$ is monotone increasing for $a \in (0,1]$; (iii) follows from the fact that $a\leq\frac{b}{2}\log^{-1/2}(1/b)$ implies $a\log^{1/2}(2/a)\leq b$ for $b\in(0,0.1]$. 
Since $\pi_{\min} \leq \frac{1}{k}$, once again $R > 0$ can be made sufficiently small so that the right-hand side of \eqref{eq:subreslt0} is at most $\frac{1}{2}$. 
Then plugging in this upper bound by \eqref{eq:subreslt0} into \eqref{eq:sbreslt1} yields 
\begin{equation}
\label{eq:lowerbound_prob_intersec}
\P(\mb x\in \mathcal{C}_{j'}\cap \mathcal{C}_{j'}^\star)\geq\frac{1}{2}\cdot\P(\mb x\in \mathcal{C}_{j'}^\star).
\end{equation} 

It remains to show the concentration of the left-hand side of \eqref{eq:statement} around the expectation. 
Recall that $\mathcal{C}_j$ and $\mathcal{C}_j^\star$ are constructed as the intersection of at most $k$ half-spaces. Then $\mathcal{C}_j \cap \mathcal{C}_j^\star$ belongs to the set $\mathcal{P}_{2k}$ defined in Lemma~\ref{lem:bnd_growth_C} and, hence, we have
\begin{align*}
\sup_{\begin{subarray}{c} j \in [k], \mb \beta \in \mathcal{N}(\mb \beta^\star) \\ \mb \beta^\star \in \mathbb{R}^{d+1} \end{subarray}}\left|\frac{1}{n}\sum_{i=1}^{n}\bbone_{\{\mb x_i\in \mathcal{C}_j\cap \mathcal{C}_j^\star\}}-\P(\mb x\in\mathcal{C}_j\cap \mathcal{C}_j^\star)\right|
\leq\sup_{\mathcal{Z}\in\mathcal{P}_{2k}}\left|\frac{1}{n}\sum_{i=1}^{n}\bbone_{\{\mb x_i\in\mathcal{Z}\}}-\P(\mb x\in \mathcal{Z})\right|.
\end{align*}
Therefore, it follows from Corollary~\ref{cor:concent_emp_prob} that with probability at least $1-\delta$
\begin{equation}
\begin{aligned}
\frac{1}{n}\sum_{i=1}^{n}\bbone_{\{\mb x_i\in \mathcal{C}_j\cap \mathcal{C}_j^\star\}}
\geq \P(\mb x\in \mathcal{C}_j\cap \mathcal{C}_j^\star)-4\sqrt{\frac{\log(4/\delta)+2k(d+1)\log(2en/(d+1))}{n}}
\end{aligned}
\label{eq:bnd_sfinal}
\end{equation} 
holds for all $j \in [k]$, $\mb \beta \in \mathcal{N}(\mb \beta^\star)$, and $\mb \beta^\star \in \mathbb{R}^{d+1}$. 
The first summand in the right-hand side of \eqref{eq:bnd_sfinal} is bounded from below as in \eqref{eq:lowerbound_prob_intersec}. 
Then choosing $C$ in \eqref{eq:sample_intersection} large enough makes the second summand less than half of the lower bound in \eqref{eq:lowerbound_prob_intersec}. 
This completes the proof. 
\end{proof}

Next, the following lemma provides a slightly improved upper bound compared to the analogous previous result \cite[Lemma~6]{ghosh2019max}. 
Moreover, Lemma~\ref{lem:bnd_partial} is derived by using the VC theory and provides a streamlined and shorter proof compared to previous work \cite{ghosh2019max}. 

\begin{lemma}
\label{lem:bnd_partial}
Suppose that Assumptions~\ref{main_assumption} and \ref{main_assumption2} hold, and that $\rho$ satisfies \eqref{eq:choice_rho1} for some numerical constant $R > 0$. Let $\delta\in(0,1/e)$. There exists an absolute constant $C$ such that if
\begin{equation}
\label{eq:samplecom_tailbounds}
n \geq C k^4\pi_{\min}^{-4(1+\zeta^{-1})} (\log(k/\delta)\vee d\log(n/d))
\end{equation} 
then with probability at least $1-\delta$
\begin{equation}
\label{eq:bnd_partial}
\begin{aligned}
\frac{1}{n}\sum_{i=1}^n\bbone_{\{\mb x_i \in \mathcal{C}_j\cap \mathcal{C}_{j'}^\star\}}\langle [\mb x_i ;\ 1], \mb v_{j,j'}^\star\rangle^2
\leq
\frac{2}{5\gamma k} \left(\frac{\pi_{\min}}{16}\right)^{1+\zeta^{-1}} \|\mb v_{j,j'}-\mb v_{j,j'}^\star\|_2^2
\end{aligned}
\end{equation}
holds for all $j \in [k]$, $\mb \beta \in \mathcal{N}(\mb \beta^\star)$, and $\mb \beta^\star \in \mathbb{R}^{d+1}$ where $\mb v_{j,j'} = \mb \beta_j - \mb \beta_{j'}$ and $\mb v_{j,j'}^\star = \mb \beta_j^\star - \mb \beta_{j'}^\star$. 
\end{lemma}

The previous result \cite[Lemma~6]{ghosh2019max} showed that with probability at least $1-\delta$ the left-hand side of \eqref{eq:bnd_partial} is bounded from above by $\tilde{O}(({\pi_{\min}^{1+\zeta^{-1}}/}{k}) \log^{\zeta/2+1}({k}/({\pi_{\min}}^{1+\zeta^{-1}})))$ if $n \geq O(\max(p,\log(1/\delta)))$. 
In contrast, Lemma~\ref{lem:bnd_partial} provides a smaller upper bound by a logarithmic factor at the cost of increased sample complexity. 
However, the condition in \eqref{eq:samplecom_tailbounds} is implied by another sufficient condition from another step of the analysis; hence, it does not affect the main result in Theorem~\ref{thm:main_noise}. \\


\begin{proof}
By the definition of $\left(\mathcal{C}_j\right)_{j=1}^k$ in \eqref{eq:def_calCj}, it holds for any $j \neq j'$ that 
\begin{equation}
\begin{aligned}
\mb x_i \in \mathcal{C}_j\cap \mathcal{C}_{j'}^\star 
&\iff \langle \mb \xi_i, \mb \beta_j \rangle \geq \langle \mb \xi_i, \mb \beta_{j'} \rangle, ~ \langle \mb \xi_i, \mb \beta_{j'}^\star \rangle \geq \langle \mb \xi_i, \mb \beta_j^\star \rangle \\
&\iff \langle \mb \xi_i, \mb v_{j,j'} \rangle \geq 0, ~ \langle \mb \xi_i, \mb v_{j,j'}^\star\rangle \leq 0 \\
&\implies \langle \mb \xi_i, \mb v_{j,j'} \rangle \langle \mb \xi_i, \mb v_{j,j'}^\star \rangle \leq 0.
\end{aligned}
\label{eq:inclusion}
\end{equation}
Furthermore, by Lemma~\ref{lem:triangle_claim}, every $\mb \beta\in\mathcal{N}(\mb \beta^\star)$ satisfies $\|\mb v_{j,j'} - \mb v_{j,j'}^\star\|_2 \leq 2\rho \|(\mb v_{j,j'}^\star)_{1:d}\|_2$. 
Therefore, it suffices to show that with probability at least $1-\delta$
\begin{equation}
\label{eq:bnd_partial2}
\begin{aligned}
\frac{1}{n}\sum_{i=1}^n\bbone_{\{\langle \mb \xi_i, \mb v \rangle \langle \mb \xi_i, \mb v^\star \rangle \leq 0\}} \langle \mb \xi_i, \mb v^\star\rangle^2
\leq 
\frac{2}{5\gamma k} \left(\frac{\pi_{\min}}{16}\right)^{1+\zeta^{-1}} \|\mb v - \mb v^\star\|_2^2
\end{aligned}
\end{equation}
holds for all $(\mb v, \mb v^\star) \in \mathcal{M}$, where 
\begin{equation*}
\mathcal{M} := \{ (\mb v, \mb v^\star) \in \mathbb{R}^{d+1} \times \mathbb{R}^{d+1} : \|\mb v - \mb v^\star\| \leq 2\rho \|(\mb v)_{1:d}\|_2 \}.    
\end{equation*}

Since $ab \leq 0$ implies $|b| \leq |a-b|$ for $a,b \in \mathbb{R}$, each summand in the left-hand side of \eqref{eq:bnd_partial2} is upper-bounded by
\begin{equation*}
\begin{aligned}
\bbone_{\{\langle\mb \xi_i,\mb v\rangle\langle\mb \xi_i,\mb v^\star\rangle\leq0\}} \langle \mb \xi_i, \mb v^\star \rangle^2
&\leq \bbone_{\{\langle\mb \xi_i,\mb v^\star\rangle^2\leq\langle\mb \xi_i,\mb v-\mb v^\star\rangle^2\}}\langle \mb \xi_i, \mb v^\star \rangle^2 \\
&\leq \bbone_{\{\langle\mb \xi_i,\mb v^\star\rangle^2\leq\langle\mb \xi_i,\mb v-\mb v^\star\rangle^2\}}\langle \mb \xi_i, \mb v-\mb v^\star \rangle^2.
\end{aligned}    
\end{equation*}
Before we proceed to the next step, for brevity, we introduce a shorthand notation given by
\begin{equation}
\label{eq:def_setS}
\mathcal{S}_{\mb v, \mb v^\star} := \{\mb \xi\in\mathbb{R}^{d+1}:\langle\mb \xi,\mb v-\mb v^\star\rangle^2\geq\langle\mb \xi,\mb v^\star\rangle^2\}.
\end{equation}
Then the left-hand side of \eqref{eq:bnd_partial2} is bounded from above as
\[
\frac{1}{n}\sum_{i=1}^n\bbone_{\{\langle \mb \xi_i, \mb v \rangle \langle \mb \xi_i, \mb v^\star \rangle \leq 0\}} \langle \mb \xi_i, \mb v^\star\rangle^2
\leq 
\frac{1}{n}\sum_{i=1}^n \bbone_{\{\mb \xi_i \in \mathcal{S}_{\mb v, \mb v^\star}\}}\langle \mb \xi_i, \mb v-\mb v^\star \rangle^2.
\]

Next, we derive a tail bound on the empirical measure $\frac{1}{n}\sum_{i=1}^n \bbone_{\{\mb \xi_i \in \mathcal{S}_{\mb v, \mb v^\star}\}}$ on the event for $\mb \xi \in \mathcal{S}_{\mb v, \mb v^\star}$.
Let $\mathcal{P}_2$ denote the collection of all polytopes given by the intersections of two half-spaces. 
Then $\mathcal{S}_{\mb v, \mb v^\star}$ belongs to $\mathcal{P}_2 \cup \mathcal{P}_2$. 
It follows from Lemma~\ref{lem:bnd_growth_C} and \cite[Theorem~A]{csikos2018optimal} that 
\begin{equation}
\label{eq:VCbound_unionintersect}
\Pi_{\mathcal{P}_2\cup\mathcal{P}_2}(n)\leq\left(\frac{en}{C'(d+1)}\right)^{C'(d+1)}
\end{equation} 
for some absolute constant $C'$. 
Therefore, by Lemma~\ref{lem:uniform} and \eqref{eq:VCbound_unionintersect}, we obtain that
\begin{align}
\sup_{(\mb v, \mb v^\star) \in \mathcal{M}} \left|\frac{1}{n}\sum_{i=1}^{n}\bbone_{\{\mb \xi_i\in \mathcal{S}_{\mb v, \mb v^\star}\}}-\P(\mb \xi\in\mathcal{S}_{\mb v, \mb v^\star})\right|
\lesssim 
\sqrt{\frac{\log(1/\delta)+d\log(n/d)}{n}}
\label{eq:first_result2} 
\end{align} 
holds with probability at least $1-\frac{\delta}{2}$. 

Similar to \eqref{eq:subreslt0}, we obtain an upper bound on the probability by using Lemma~\ref{lem:bnd_vol} as follows:
\begin{align}
\sup_{(\mb v, \mb v^\star) \in \mathcal{M}} \P(\mb \xi\in \mathcal{S}_{\mb v, \mb v^\star}) 
&\leq C_1 \left( (2\rho)^2 \log\left(\frac{1}{\rho}\right)\right)^{\zeta}\nonumber\\
&\leq C_1 {\left(\frac{ R^2 \pi_{\min}^{2\zeta^{-1}(1+\zeta^{-1})}}{k^{2\zeta^{-1}}}\right)^{\zeta}} \nonumber \\
&\leq
\underbrace{\frac{C_1 R^{2\zeta} \pi_{\min}^{2+2\zeta^{-1}}}{k^2}}_{\alpha} \label{eq:ub_exp_Svvstar}
\end{align} 
where $C_1 > 0$ is an absolute constant. 
By choosing the numerical constant $C > 0$ in \eqref{eq:samplecom_tailbounds} sufficiently large, we obtain from \eqref{eq:first_result2} and \eqref{eq:ub_exp_Svvstar} that
\begin{equation}
\label{eq:Kvbnd}
\mathbb{P}\left( \sup_{(\mb v, \mb v^\star) \in \mathcal{M}} \frac{1}{n}\sum_{i=1}^n \bbone_{\{\mb \xi_i \in \mathcal{S}_{\mb v, \mb v^\star}\}} > { \frac{\alpha }{2}} \right) \leq \frac{\delta}{2}.
\end{equation}

Furthermore, one can choose the numerical constant $R > 0$ small enough so that $\alpha \in (0,1)$. 
Then, since \eqref{eq:samplecom_tailbounds} and \eqref{eq:choice_rho1} imply \eqref{eq:samp_tanver}, by Lemma~\ref{lem:upperbound}, it holds with probability at least $1-\delta/2$ that 
\begin{equation}
\sup_{\mathcal{I} : |\mathcal{I}| \leq \frac{\alpha n}{2}} \left\| \sum_{i \in \mathcal{I}} \mb \xi_i \mb \xi_i^\top \right\|
\lesssim (\eta^2\vee1)\sqrt{\alpha} n. 
\label{eq:lem:bnd_partial:largest_ev}
\end{equation}

Finally, by combining the results in \eqref{eq:Kvbnd} and  \eqref{eq:lem:bnd_partial:largest_ev}, we obtain that with probability at least $1-\delta$
\begin{align*}
\frac{1}{n}\sum_{i=1}^n\bbone_{\{\langle \mb \xi_i, \mb v \rangle \langle \mb \xi_i, \mb v^\star \rangle \leq 0\}} \langle \mb \xi_i, \mb v^\star\rangle^2
& \leq 
\sup_{\mathcal{I} : |\mathcal{I}| \leq \frac{\alpha n}{2}} \frac{1}{n} \sum_{i \in \mathcal{I}} \langle \mb \xi_i, \mb v - \mb v^\star \rangle^2 \\
& \leq \sup_{\mathcal{I} : |\mathcal{I}| \leq \frac{\alpha n}{2}} \left\| \frac{1}{n} \sum_{i \in \mathcal{I}} \mb \xi_i \mb \xi_i^\top \right\| \cdot \| \mb v - \mb v^\star \|_2^2 \\
& \leq C_2(\eta^2\vee1)R^{\zeta} \left({\frac{\pi_{\min}^{(1+\zeta^{-1})}}{k}}\right)\cdot\| \mb v - \mb v^\star \|_2^2
\end{align*}
holds for all $(\mb v, \mb v^\star) \in \mathcal{M}$, where $C_2$ is an absolute constant. 
By choosing $R > 0$ sufficiently small so that 
\[
C_2(\eta^2\vee1)R^{\zeta}\leq\frac{2}{5\gamma}\left(\frac{1}{16}\right)^{1+\zeta^{-1}},
\] 
we obtain the assertion in \eqref{eq:bnd_partial2}. 
\end{proof}

\section{Proof of \Cref{thm:main_noise}}
\label{sec:thmproof}

The loss function $\ell(\mb \beta)$ is decomposed as
\begin{align*}
\ell(\mb \beta)&=\frac{1}{2n}\left(\max_{j\in[k]}\langle\mb \xi_i,\mb \beta_j\rangle-\max_{j\in[k]}\langle\mb \xi_i,\mb \beta_j^\star\rangle-z_i\right)^2\\
&=\underbrace{\frac{1}{2n}\sum_{i=1}^{n}\left(\max_{j\in[k]}\langle\mb \xi_i,\mb \beta_j\rangle-\max_{j\in[k]}\langle\mb \xi_i,\mb \beta_j^\star\rangle\right)^2}_{\ell^{\mathrm{clean}}(\mb \beta)} \\
&-\underbrace{\left(\frac{1}{n}\sum_{i=1}^{n}z_i\left(\max_{j\in[k]}\langle\mb \xi_i,\mb \beta_j\rangle-\max_{j\in[k]}\langle\mb \xi_i,\mb \beta_j^\star\rangle\right)-\frac{1}{2n}\sum_{i=1}^{n}z_i^2\right)}_{\ell^{\mathrm{noise}}(\mb \beta)}.
\end{align*}
Then the partial gradient of $\ell(\mb \beta)$ with respect to $\mb \beta_l$ is written as
\begin{equation}
\label{eq:gradient_clean_def}
\begin{aligned}
    \nabla_{\mb \beta_l}\ell(\mb \beta)&=\frac{1}{n}\sum_{i=1}^{n}\bbone_{\{\mb x_i\in\mathcal{C}_l\}}\left(\max_{j\in[k]}\langle\mb \xi_i,\mb \beta_j\rangle-\max_{j\in[k]}\langle\mb \xi_i,\mb \beta_j^\star\rangle-z_i\right)\mb \xi_i\\
    &=\underbrace{\frac{1}{n}\sum_{i=1}^{n}\bbone_{\{\mb x_i\in\mathcal{C}_l\}}\left(\max_{j\in[k]}\langle\mb \xi_i,\mb \beta_j\rangle-\max_{j\in[k]}\langle\mb \xi_i,\mb \beta_j^\star\rangle\right)\mb \xi_i}_{ \nabla_{\mb \beta_l}\ell^{\mathrm{clean}}(\mb \beta)}-\underbrace{\frac{1}{n}\sum_{i=1}^{n}z_i\bbone_{\{\mb x_i\in\mathcal{C}_l\}}\mb \xi_i}_{\nabla_{\mb \beta_l}\ell^{\mathrm{noise}}(\mb \beta)}
\end{aligned}
\end{equation}
where $\mathcal{C}_1, \dots, \mathcal{C}_k$ are determined by $\mb \beta$ as in \eqref{eq:def_calCj}. 

In the remainder of the proof, we will use the following shorthand notation to denote the pairwise difference of parameter vectors and the probability measure on the largest partition by the ground-truth model: 
\begin{align*}
\mb v_{j,j'} :=\mb \beta_{j}-\mb \beta_{j'}, \quad 
\mb v_{j,j'}^\star
:=\mb \beta_j^\star-\mb \beta_{j'}^\star, \quad \text{and} \quad \pi_{\max} :=\max_{j\in[k]}\P\left(\mb x\in \mathcal{C}_j^\star\right).
\end{align*}
Below we show that the following lemmas hold under the condition in \eqref{eq:samplcomp_noise}. 
The proof is provided in Appendix~\ref{sec:proof:lem:lwb_gradient}.

\begin{lemma}
\label{lem:lwb_gradient}
Under the hypothesis of Theorem~\ref{thm:main_noise}, if \eqref{eq:samplcomp_noise} is satisfied, then with probability at least $1-\delta$ the following inequalities hold for all $j \in [k]$, $\mb \beta^\star \in \mathbb{R}^{k(d+1)}$, and $\mb \beta^t \in \mathcal{N}(\mb \beta^\star)$:  
\begin{equation}
\label{eq:res:lem:lwb_gradient}
\langle\nabla_{\mb \beta_j}\ell^{\mathrm{clean}}(\mb \beta^t),\mb \beta_j^t-\mb \beta_j^\star\rangle
\geq   
\frac{2}{\gamma}\left(\frac{\pi_{\min}}{16}\right)^{1+\zeta^{-1}}
\left( \|\mb \beta_j^t-\mb \beta_j^\star\|_2^2 - \frac{1}{10k} \sum_{j':j'\neq j}\left\|\mb v_{j,j'}^t-\mb v_{j,j'}^\star\right\|_2^2 \right),
\end{equation}
\begin{equation}
\label{eq:res:lem:upbnd_normgrad}
\|\nabla_{\mb \beta_j}\ell^{\mathrm{clean}}(\mb \beta^t)\|_2^2
\lesssim \left({\pi_{\max}}+\pi_{\min}^{2(1+\zeta^{-1})}\right)\left\|\mb \beta_j^t-\mb \beta_j^\star\right\|_2^2 + \frac{\pi_{\min}^{2(1+\zeta^{-1})}}{k^2} \sum_{j':j'\neq j}\left\|\mb v_{j,j'}^t-\mb v_{j,j'}^\star\right\|_2^2, 
\end{equation}
and
\begin{equation}
\label{eq:res:lem:upbnd_noisegrad}
\left\|
\nabla_{\mb \beta_j}\ell^{\mathrm{noise}}(\mb \beta^t)
\right\|_2
\lesssim 
\frac{\sigma \sqrt{kd\log(n/d) + \log(1/\delta)}}{\sqrt{n}}.
\end{equation}
\end{lemma}

The remainder of the proof shows that the assertion of the theorem is obtained from \eqref{eq:res:lem:lwb_gradient}, \eqref{eq:res:lem:upbnd_normgrad} and \eqref{eq:res:lem:upbnd_noisegrad} via the following three steps. \\

\noindent\textbf{Step~1:} 
We prove by induction that all iterates remain within the neighborhood $\mathcal{N}(\mb \beta^\star)$. 
Suppose that $\mb \beta^t \in \mathcal{N}(\mb \beta^\star)$ holds for a fixed $t \in \mbb{N}$. 
By the triangle inequality, for any $j\in[k]$, the next iterate $\mb \beta^{t+1}$ satisfies
\begin{align}
\|\mb \beta_j^{t+1}-\mb \beta_j^\star\|_2
&=\|\mb \beta_j^{t}-\mu\nabla_{\mb \beta_j}\ell(\mb \beta^t)-\mb \beta_j^\star\|_2\nonumber\\
&\leq\underbrace{\|\mb \beta_j^{t}-\mu\nabla_{\mb \beta_j}\ell^{\mathrm{clean}}(\mb \beta^t)-\mb \beta_j^\star\|_2}_{A_\mathrm{clean}}
+\underbrace{\mu\|\nabla_{\mb \beta_j}\ell^{\mathrm{noise}}(\mb \beta^t)\|_2}_{A_\mathrm{noise}}\label{eq:eql_noise}.
\end{align}
Then it remains to show 
\begin{equation}
\label{eq:objective_step1}
\|\mb \beta_j^{t+1}-\mb \beta_j^\star\|_2\leq A_{\mathrm{clean}}+A_{\mathrm{noise}}\leq \kappa\rho,\quad\forall j\in[k].
\end{equation}

Note that the first summand in the right-hand side of \eqref{eq:eql_noise} satisfies
\begin{align*}
A_\mathrm{clean}^2
=\|\mb \beta_j^{t}-\mb \beta_j^\star\|_2^2-2\mu\langle\nabla_{\mb \beta_j}\ell^{\mathrm{clean}}(\mb \beta^t),\mb \beta_j^{t}-\mb \beta_j^\star\rangle+\mu^2\|\nabla_{\mb \beta_j}\ell^{\mathrm{clean}}(\mb \beta^t)\|_2^2. 
\end{align*} 
Therefore, it follows from \eqref{eq:res:lem:lwb_gradient} and \eqref{eq:res:lem:upbnd_normgrad} that
\begin{align}
A_\mathrm{clean}^2 &\leq\left\|\mb \beta_j^{t}-\mb \beta_j^\star\right\|_2^2 
- \frac{4\mu}{\gamma}\left(\frac{1}{16}\right)^{1+\zeta^{-1}} \pi_{\min}^{1+\zeta^{-1}} \left(\|\mb \beta_j^{t}-\mb \beta_j^\star\|_2^2-\frac{1}{10k}\sum_{j':j'\neq j}\left\|\mb v_{j,j'}^t-\mb v_{j,j'}^\star\right\|_2^2\right) \nonumber \\
&\quad+\mu^2 C_1\left(\left({\pi_{\max}}+\pi_{\min}^{2(1+\zeta^{-1})}\right)\left\|\mb \beta_j^{t}-\mb \beta_j^\star\right\|_2^2+\frac{\pi_{\min}^{2(1+\zeta^{-1})}}{k^2}\sum_{j':j'\neq j}\left\|\mb v_{j,j'}^t-\mb v_{j,j'}^\star\right\|_2^2\right) \nonumber \\
&=\left(1 - \frac{4}{\gamma}\left(\frac{1}{16}\right)^{1+\zeta^{-1}} \mu\pi_{\min}^{1+\zeta^{-1}} + C_1\mu^2\left({\pi_{\max}}+\pi_{\min}^{2(1+\zeta^{-1})}\right)\right)\|\mb \beta_j^t-\mb \beta_j^\star\|_2^2 \nonumber \\
&\quad+\left(\frac{\frac{2}{\gamma}\left(\frac{1}{16}\right)^{1+\zeta^{-1}}\mu \pi_{\mathrm{min}}^{1+\zeta^{-1}}}{5k}+\frac{C_1 \mu^2 \pi_{\min}^{2(1+\zeta^{-1})}}{k^2}\right)\sum_{j'^*:j'\neq j}\left\|\mb v_{j,j}^t-\mb v_{j,j'}^\star\right\|_2^2. \label{eq:proof_rec_prox_ub1}
\end{align} 
We set the step size $\mu$ to be
\begin{equation}
\label{eq:muchoose}
\mu=\frac{\omega\pi_{\min}^{1+\zeta^{-1}}}{\tau}
\end{equation}
where $\omega$ is a constant that will be specified later and $\tau$ is given by
\begin{equation}
\label{eq:choose_tau}
\tau:={\pi_{\max}} + \pi_{\min}^{2(1+\zeta^{-1})}.
\end{equation}
Putting the choices of $\mu$ and $\tau$ respectively by \eqref{eq:muchoose} and \eqref{eq:choose_tau} into \eqref{eq:proof_rec_prox_ub1} yields
\begin{equation}
\label{eq:estimate_bnd1}
\begin{aligned}
A_\mathrm{clean}^2 & \leq\left(1-\frac{  \frac{4}{\gamma}\left(\frac{1}{16}\right)^{1+\zeta^{-1}} \omega \pi_{\min}^{2(1+\zeta^{-1})}}{\tau}+\frac{C_1\omega^2\pi_{\min}^{2(1+\zeta^{-1})}\left({\pi_{\max}}+\pi_{\min}^{2(1+\zeta^{-1})}\right)}{\tau^2}\right)\|\mb \beta_j^t-\mb \beta_j^\star\|_2^2 
\\
& \quad+\left(\frac{\frac{2}{\gamma}\left(\frac{1}{16}\right)^{1+\zeta^{-1}} \omega\pi_{\min}^{2(1+\zeta^{-1})}}{5 \tau k}+\frac{C_1 \omega^2\pi_{\min}^{4(1+\zeta^{-1})}}{\tau^2 k^2}\right)\sum_{j':j'\neq j}\left\|\mb v_{j,j'}^t-\mb v_{j,j'}^\star\right\|_2^2 \\
&\leq\left(1-\frac{ \frac{4}{\gamma}\left(\frac{1}{16}\right)^{1+\zeta^{-1}} \omega \pi_{\min}^{2(1+\zeta^{-1})}}{\tau}+\frac{C_1\omega^2\pi_{\min}^{2(1+\zeta^{-1})}}{\tau}\right)\|\mb \beta_j^t-\mb \beta_j^\star\|_2^2 \\
& \quad+\left(\frac{\frac{2}{\gamma}\left(\frac{1}{16}\right)^{1+\zeta^{-1}} \omega\pi_{\min}^{2(1+\zeta^{-1})}}{5\tau }+\frac{C_1w^2\pi_{\min}^{2(1+\zeta^{-1})}}{\tau} \right) \max_{1 \leq j\neq j' \leq k}\left\|\mb v_{j,j'}^t-\mb v_{j,j'}^\star\right\|_2^2.
\end{aligned} 
\end{equation}
Next, since $\mb \beta^t \in \mathcal{N}(\mb \beta^\star)$, by the definition of $\mathcal{N}(\mb \beta^\star)$ in \eqref{eq:defnbr}, we have
\begin{equation}
\label{eq:rec_cond0}
\max_{j \in [k]} \|\mb \beta_j^t - \mb \beta_j^\star\|_2 \leq \kappa\rho.
\end{equation}
Furthermore, by Lemma~\ref{lem:triangle_claim}, we also have
\begin{equation}
\label{eq:rec_cond}
\max_{1\leq j\neq j'\leq k}{\left\|\mb v_{j,j'}^{t}-\mb v_{j,j'}^\star\right\|_2} 
\leq 
2 \kappa\rho.
\end{equation} 
Then plugging in \eqref{eq:rec_cond0} and \eqref{eq:rec_cond} into \eqref{eq:estimate_bnd1} yields
\begin{equation}
\begin{aligned}
(\kappa\rho)^{-2} A_\mathrm{clean}^2
&\leq 1-\frac{\pi_{\min}^{2(1+\zeta^{-1})} \omega}{\tau}\left(\frac{2}{\gamma}\left(\frac{1}{16}\right)^{1+\zeta^{-1}} \left(2-\frac{4}{5}\right)+C_1\omega \left(1+4\right) \right)\\
&\leq 1-\frac{\pi_{\min}^{2(1+\zeta^{-1})}}{\tau} \cdot {\omega \left(\frac{\frac{12}{\gamma}\left(\frac{1}{16}\right)^{1+\zeta^{-1}}}{5} +5\omega C_1 \right)}\\
&\leq 1-\frac{\pi_{\min}^{2(1+\zeta^{-1})}}{\tau} \cdot \omega\underbrace{\left( \frac{\frac{12}{\gamma}\left(\frac{1}{16}\right)^{1+\zeta^{-1}}}{5}\right)}_{c_0},
\label{eq:estiamte_bnd2}
\end{aligned}
\end{equation}
which is rewritten as 
\begin{equation}
A_\mathrm{clean}^2 \\
\leq (\kappa\rho)^{2} \left(1 - \frac{c_0\omega \pi_{\min}^{2(1+\zeta^{-1})}}{\tau}\right).
\label{eq:estiamte_bnd3}
\end{equation}
For fixed $\gamma$ and $\zeta$, $c_0$ is a positive numerical constant. 
Due to the choice of $\tau$ by \eqref{eq:choose_tau}, we have
\[
\frac{\pi_{\min}^{2(1+\zeta^{-1})}}{\tau}
= \frac{\pi_{\min}^{2(1+\zeta^{-1})}}{\pi_{\max} + \pi_{\min}^{2(1+\zeta^{-1})}} < 1,
\]
Furthermore, one can choose $\omega > 0$ sufficiently small so that $\omega c_0<1$. 
Then the upper bound in the right-hand side of \eqref{eq:estiamte_bnd3} is valid as a positive number. 

If $A_{\mathrm{noise}}$ is upper-bounded as
\begin{equation}
\label{eq:suffice_noise1}
A_{\mathrm{noise}}\leq\kappa\rho\frac{c_0\omega\pi_{\min}^{2(1+\zeta^{-1})}}{2\tau},
\end{equation}
then, by the elementary inequality $1-\sqrt{1-\alpha}\geq\alpha/2$ that holds for any $\alpha\in(0,1)$, we have
\begin{equation}
\label{eq:suffice_noise2}
A_{\mathrm{noise}}\leq\kappa\rho\left(1-\sqrt{1-\frac{c_0\omega\pi_{\min}^{2(1+\zeta^{-1})}}{\tau}}\right).     
\end{equation}
Then \eqref{eq:estiamte_bnd3} and \eqref{eq:suffice_noise2} yield \eqref{eq:objective_step1}. 
Therefore, it suffices to show that \eqref{eq:suffice_noise1} holds. 

{
Due to the inequality in \eqref{eq:res:lem:upbnd_noisegrad}, we have 
\[
\left\|
\nabla_{\mb \beta_j}\ell^{\mathrm{noise}}(\mb \beta^t)
\right\|_2
\lesssim \frac{\sigma \sqrt{kd\log(n/d) + \log(1/\delta)}}{\sqrt{n}}, \quad \forall j \in [k]. 
\] 
By the choice of $\mu$ in \eqref{eq:muchoose}, we obtain an upper bound on $A_{\mathrm{noise}}$ given by
\begin{equation}
\label{eq:bnd_noise_norm1}
A_{\mathrm{noise}}=\mu\left\|
\nabla_{\mb \beta_j}\ell^{\mathrm{noise}}(\mb \beta^t)
\right\|_2
\lesssim
\frac{\omega\pi_{\min}^{1+\zeta^{-1}}}{\tau} \cdot \frac{\sigma\sqrt{kd\log(n/d)+\log(1/\delta)}}{\sqrt{n}}.
\end{equation}
The condition in \eqref{eq:samplcomp_noise} implies
\begin{equation}
\label{eq:sample_noise1}
n \geq C\cdot\frac{\sigma^2  \pi_{\min}^{-2(1+\zeta^{-1})} \left( kd\log(n/d)+\log(1/\delta) \right)}{\kappa^{2} \rho^{2} }. 
\end{equation}
One can choose the absolute constant $C > 0$ in \eqref{eq:samplcomp_noise} and \eqref{eq:sample_noise1} as large enough so that \eqref{eq:sample_noise1} and \eqref{eq:bnd_noise_norm1} imply \eqref{eq:suffice_noise1}. 
This completes the induction argument in Step 1.\\

}


\noindent\textbf{Step~2:} 
Next we show that all iterates also satisfy
\begin{align}
\label{eq:recursion_decrease}
\left\|\mb \beta^{t+1}-\mb \beta^\star\right\|_2
\leq \sqrt{1-\nu} \left\|\mb \beta^t-\mb \beta^\star\right\|_2
+C'\mu\sigma\sqrt{\frac{k \left(kd\log(n/d)+\log(1/\delta)\right)}{n}}.
\end{align}
We use the fact that $\mb \beta^t \in \mathcal{N}(\mb \beta^\star)$, which has been shown in Step 1.
By the update rule of gradient descent and the triangle inequality, the left-hand side of \eqref{eq:recursion_decrease} satisfies
\begin{align}
\|\mb \beta^{t+1}-\mb \beta^\star\|_2
&= \|\mb \beta^t-\mu\nabla_{\mb \beta}\ell(\mb \beta^t)-\mb \beta^\star\|_2 \nonumber \\
& \leq 
\|\mb \beta^t-\mu\nabla_{\mb \beta}\ell^{\mathrm{clean}}(\mb \beta^t)-\mb \beta^\star\|_2
+\mu
\|\nabla_{\mb \beta}\ell^{\mathrm{noise}}(\mb \beta^t)\|_2 \nonumber \\
&=\underbrace{\sqrt{\sum_{j=1}^{k}\|\mb \beta_j^t-\mb \beta_j^\star-\mu\nabla_{\mb \beta_j}\ell^{\mathrm{clean}}(\mb \beta^t)\|_2^2}}_{B_\mathrm{clean}}
+\underbrace{\sqrt{{\mu^2 \sum_{j=1}^{k}\|\nabla_{\mb \beta_j}\ell^{\mathrm{noise}}(\mb \beta^t)\|_2^2}}}_{B_\mathrm{noise}}.\label{eq:decomposed_noise}
\end{align}
Below we derive an upper bound on each of the summands on the right-hand side of \eqref{eq:decomposed_noise}. 
First we show that
\begin{align}
\label{eq:recursion_decrease_clean}
B_\mathrm{clean}^2
\leq (1-\nu) \sum_{j=1}^{k}\left\|\mb \beta_j^t-\mb \beta_j^\star\right\|_2^2.
\end{align}
Since $\mb \beta^t \in \mathcal{N}(\mb \beta^\star)$, the inequality in \eqref{eq:recursion_decrease_clean} holds if there exist constants $\mu, \lambda \in (0,1)$ such that
\begin{equation}
\label{eq:regulcond_maxlinear}
\sum_{j=1}^{k}\langle\nabla_{\mb \beta_j}\ell^{\mathrm{clean}}(\mb \beta^t),\mb \beta_j - \mb \beta_j^\star\rangle\geq\frac{\mu}{2}\sum_{j=1}^{k}\|\nabla_{\mb \beta_j}\ell^{\mathrm{clean}}(\mb \beta^t)\|_2^2+\frac{\lambda}{2}\sum_{j=1}^{k}\|\mb \beta_j^t - \mb \beta_j^\star\|_2^2, \quad \forall \mb \beta^t \in \mathcal{N}(\mb \beta^{\star}).
\end{equation}
Indeed, the condition in \eqref{eq:regulcond_maxlinear} and $\mb \beta^t \in \mathcal{N}(\mb \beta^{\star})$ imply 
\begin{align}
    B_\mathrm{clean}^2
    &=\sum_{j=1}^{k}\|\mb \beta_j^t-\mu\nabla_{\mb \beta_j}\ell^{\mathrm{clean}}(\mb \beta^t)-\mb \beta_j^\star\|_2^2 \nonumber \\
    &=\sum_{j=1}^{k}\|\mb \beta_j^t - \mb \beta_j^\star\|_2^2+\sum_{j=1}^{k}\mu^2 \|\nabla_{\mb \beta_j}\ell^{\mathrm{clean}}(\mb \beta^t)\|_2^2-2\mu\sum_{j=1}^{k}\langle\mb \beta_j^t - \mb \beta_j^\star,\nabla_{\mb \beta_j}\ell^{\mathrm{clean}}(\mb \beta^t)\rangle\nonumber\\
    &\leq\sum_{j=1}^{k}\|\mb \beta_j^t - \mb \beta_j^\star\|_2^2-\mu\lambda\sum_{j=1}^{k}\|\mb \beta_j^t - \mb \beta_j^\star\|_2^2\nonumber\\
    &=(1-\mu\lambda)\sum_{j=1}^{k}\|\mb \beta_j^t - \mb \beta_j^\star\|_2^2. \label{eq:noiseless_contraction}
\end{align}

Next we show that \eqref{eq:regulcond_maxlinear} holds. 
Due to \eqref{eq:res:lem:lwb_gradient} and the elementary inequality $\|\mb a+\mb b\|_2^2\leq2\|\mb a\|_2^2+2\|\mb b\|_2^2$, it holds for all $j\in[k]$ that
\begin{equation}
\label{eq:subrslt_LHS}
\begin{aligned}
& \langle\nabla_{\mb \beta_j}\ell^{\mathrm{clean}}(\mb \beta^t),\mb \beta_j^t-\mb \beta_j^\star\rangle \\
& \geq \frac{2}{\gamma}\left(\frac{1}{16}\right)^{1+\zeta^{-1}}\pi_{\min}^{1+\zeta^{-1}} \left( \|\mb \beta_j^t-\mb \beta_j^\star\|_2^2 - \frac{1}{5k} \sum_{j':j'\neq j}\left(\left\|\mb \beta_j^t-\mb \beta_j^\star\right\|_2^2+\left\|\mb \beta_{j'}^t-\mb \beta_{j'}^\star\right\|_2^2\right) \right).
\end{aligned}
\end{equation} 
By taking the summation of \eqref{eq:subrslt_LHS} over $j \in [k]$, we obtain
\begin{align}
\sum_{j=1}^k\langle\nabla_{\mb \beta_j}\ell^{\mathrm{clean}}(\mb \beta^t),\mb \beta_j^t-\mb \beta_j^\star\rangle
\geq \frac{ \frac{6}{\gamma}\left(\frac{1}{16}\right)^{1+\zeta^{-1}}\pi_{\min}^{1+\zeta^{-1}}}{5} \sum_{j=1}^{k}\|\mb \beta_j^t-\mb \beta_j^\star\|_2^2. \label{eq:lwb_result}
\end{align}
Furthermore, by using \eqref{eq:res:lem:upbnd_normgrad} and the elementary inequality $\|\mb a+\mb b\|_2^2\leq2\|\mb a\|_2^2+2\|\mb b\|_2^2$ again, we obtain
\begin{equation}
\label{eq:result_up}
\begin{aligned}
\|\nabla_{\mb \beta_j}\ell^{\mathrm{clean}}(\mb \beta^t)\|_2^2&\leq C_1\left({\pi_{\max}}+\pi_{\min}^{2(1+\zeta^{-1})}\right)\|\mb \beta_j^t-\mb \beta_j^\star\|_2^2\\
&+\frac{2 C_1 \pi_{\min}^{2(1+\zeta^{-1})}}{k^2}\sum_{j':j'\neq j}\left(\|\mb \beta_j^t-\mb \beta_j^\star\right\|_2^2+\left\|\mb \beta_{j'}^t-\mb \beta_{j'}^\star\|_2^2\right).
\end{aligned}
\end{equation}
Summing the equation in \eqref{eq:result_up} over $j \in [k]$ yields  

\begin{equation}
\label{eq:RHSbnd1}
\begin{aligned}
\sum_{j=1}^k \|\nabla_{\mb \beta_j}\ell^{\mathrm{clean}}(\mb \beta^t)\|_2^2
&\leq 
C_1\left( {\pi_{\max}}+\pi_{\min}^{2(1+\zeta^{-1})} + \frac{4 (k-1)\pi_{\min}^{2(1+\zeta^{-1})}}{k^2} \right)
\sum_{j=1}^k \left\|\mb \beta_j^t-\mb \beta_j^\star\right\|_2^2\\
&\leq C_1\left( {\pi_{\max}}+\pi_{\min}^{2(1+\zeta^{-1})} + {4\pi_{\min}^{2(1+\zeta^{-1})}} \right)
\sum_{j=1}^k \left\|\mb \beta_j^t-\mb \beta_j^\star\right\|_2^2.
\end{aligned}
\end{equation} 
By combining \eqref{eq:lwb_result} and \eqref{eq:RHSbnd1} with $\mu$ as in \eqref{eq:muchoose}, we obtain a sufficient condition for \eqref{eq:regulcond_maxlinear} given by

\begin{equation}
\label{eq:RClwup}
\frac{ \frac{6}{\gamma}\left(\frac{1}{16}\right)^{1+\zeta^{-1}} \pi_{\min}^{1+\zeta^{-1}}}{5}
\geq
\frac{\omega\pi_{\min}^{1+\zeta^{-1}} C_1 \left( {\pi_{\max}}+5\pi_{\min}^{2(1+\zeta^{-1})}\right)}{2\left( {\pi_{\max}}+\pi_{\min}^{2(1+\zeta^{-1})}\right)}
+ \frac{\lambda}{2}.
\end{equation}
By choosing $\omega > 0$ small enough, \eqref{eq:RClwup} is satisfied when $\lambda$ is chosen as
\begin{equation}
\label{eq:lambdachoose}
\lambda= \min(c_2 \pi_{\min}^{1+\zeta^{-1}}, 1)
\end{equation} 
for an absolute constant $c_2 > 0$.
Hence, we have shown that the condition in \eqref{eq:regulcond_maxlinear} holds with $\mu$ and $\lambda$ specified by \eqref{eq:muchoose} and \eqref{eq:lambdachoose}. 

Next we consider the second summand on the right-hand side of \eqref{eq:decomposed_noise}. 
The inequality in \eqref{eq:res:lem:upbnd_noisegrad} implies 
\begin{align}
\label{eq:noise_bnd_B}
B_\mathrm{noise}^2 
= 
\mu^2 \sum_{j=1}^{k}\left\|\nabla_{\mb \beta_j}\ell^{\mathrm{noise}}(\mb \beta^t)\right\|_2^2\lesssim
\frac{\mu^2\sigma^2 k(kd\log(n/d) + \log(1/\delta))}{n}.
\end{align}

Finally, plugging in \eqref{eq:noiseless_contraction} and \eqref{eq:noise_bnd_B} into \eqref{eq:decomposed_noise} provides the assertion in \eqref{eq:recursion_decrease}. This completes the proof of Step~$2$. \\

\noindent\textbf{Step~3:} 
We finish the proof of Theorem~\ref{thm:main_noise} by applying the results in Step 1 and Step 2. 
Plugging in the expression of $\nu=\mu\lambda$ with $\mu$ and $\lambda$ as in \eqref{eq:muchoose} and \eqref{eq:lambdachoose} provides 
\begin{align*}
\|\mb \beta^{t}-\mb \beta^\star\|_2
&\leq\left(1-\mu\lambda\right)^{t/2}\|\mb \beta^0-\mb \beta^\star\|_2+C_2\cdot \frac{\mu\sigma}{1-\sqrt{1-\mu\lambda}} \cdot \sqrt{\frac{k\left(kd\log(n/d)+\log(1/\delta)\right)}{n}}\\
&\overset{\mathrm{(a)}}{\leq}\left(1-\mu\lambda\right)^{t/2}\|\mb \beta^0-\mb \beta^\star\|_2+C_2 \cdot \frac{2\sigma}{\lambda}\cdot \sqrt{\frac{k\left(kd\log(n/d)+\log(1/\delta)\right)}{n}}\\
&
\overset{\mathrm{(b)}}{\leq}\left(1-\mu\lambda\right)^{t/2}\|\mb \beta^0-\mb \beta^\star\|_2+C_3 \cdot \frac{\sigma}{\pi_{\max}}\cdot\sqrt{\frac{k\left(kd\log(n/d)+\log(1/\delta)\right)}{n}}\\
&
\overset{\mathrm{(c)}}{\leq}\left(1-\mu\lambda\right)^{t/2}\|\mb \beta^0-\mb \beta^\star\|_2+C_3 \cdot \sigma k\sqrt{\frac{k\left(kd\log(n/d)+\log(1/\delta)\right)}{n}}, 
\end{align*} 
where $\mathrm{(a)}$ follows from the elementary inequality $\sqrt{1-t}<1-t/2$ for any $t\in(0,1)$; $\mathrm{(b)}$ holds by the choice of $\tau$ in \eqref{eq:choose_tau}; $\mathrm{(c)}$ holds since $\pi_{\max}^{-1}\leq k$.

\subsection{Proof of Lemma~\ref{lem:lwb_gradient}}
\label{sec:proof:lem:lwb_gradient}

We show that each of \eqref{eq:res:lem:lwb_gradient}, \eqref{eq:res:lem:upbnd_normgrad}, and \eqref{eq:res:lem:upbnd_noisegrad} holds with probability at least $1-\delta/3$. { We also note that for simplicity, we proceed on the proofs using $\mb \beta$ and $\mb v_{j,j'}$. Therefore, the assertions in \eqref{eq:res:lem:lwb_gradient}, \eqref{eq:res:lem:upbnd_normgrad}, and \eqref{eq:res:lem:upbnd_noisegrad} can be completed by substituting $\mb \beta$ and $\mb v_{j,j'}$ with $\mb \beta^t$ and $\mb v_{j,j'}^t$ respectively.}

\noindent\textbf{Proof of \eqref{eq:res:lem:lwb_gradient}: }
We show that \eqref{eq:res:lem:lwb_gradient} holds with high probability under the following condition 
\begin{equation}
\label{eq:cond:lem:lwb_gradient}
n \geq C_1 \left(\log(k/\delta) \vee d\log(n/d)\right)k^4\pi_{\min}^{-4(1+\zeta^{-1})},
\end{equation} which is implied by the assumption in \eqref{eq:samplcomp_noise}. 
We proceed with the proof under the following three events, each of which holds with probability at least $1-\delta/9$. 
First, since \eqref{eq:cond:lem:lwb_gradient} implies \eqref{eq:samplecom_tailbounds}, by Lemma~\ref{lem:bnd_partial}, it holds with probability at least $1 - \delta/9$ that
\begin{equation}
\label{eq:bnd_vjj'}
\begin{aligned}
& \frac{1}{n}\sum_{j':j'\neq j}\sum_{i=1}^n \bbone_{\{\mb x_i \in \mathcal{C}_j\cap \mathcal{C}_{j'}^\star\}}\langle \mb \xi_i, \mb v_{j,j'}^\star\rangle^2 \\ 
& \leq \frac{2}{5\gamma k} \left(\frac{\pi_{\min}}{16}\right)^{1+\zeta^{-1}} \sum_{j':j'\neq j}\|\mb v_{j,j'}-\mb v_{j,j'}^\star\|_2^2, \quad \forall j \in [k], ~ \forall \mb \beta\in\mathcal{N}(\mb \beta^\star), ~ \forall \mb \beta^\star \in \mathbb{R}^{d+1}.
\end{aligned}
\end{equation}
Moreover, since \eqref{eq:cond:lem:lwb_gradient} also implies \eqref{eq:sample_intersection}, by Lemma~\ref{lem:lwb_number_ball}, it holds with probability at least $1-\delta/3$ that
\begin{equation}
\label{eq:lwb_number}
\frac{1}{n}\sum_{i=1}^{n}\bbone_{\{\mb x_i\in \mathcal{C}_j\cap \mathcal{C}_j^\star\}} \geq \frac{\pi_{\min}}{4}, \quad \forall j \in [k], ~ \forall \mb \beta \in \mathcal{N}(\mb \beta^\star), ~ \forall \mb \beta^\star \in \mathbb{R}^{d+1}. 
\end{equation} 
Lastly, since \eqref{eq:cond:lem:lwb_gradient} is a sufficient condition to invoke Lemma~\ref{lem:lwb_trunc} with $\alpha=\pi_{\min}/4$, it holds with probability at least $1-\delta/9$ that
\begin{equation}
\label{eq:bnd_smallball}
\begin{aligned}
\inf_{\mc I \subset [n] : |\mc I|\geq\frac{\pi_{\min} n}{4}} \lambda_{d+1}\left(\frac{1}{n}\sum_{i\in \mc I} \mb \xi_i \mb \xi_i^\top \right) \geq  \frac{2}{\gamma}\left(\frac{\pi_{\min}}{16}\right)^{1+\zeta^{-1}}.
\end{aligned}
\end{equation}
Therefore, we have shown that \eqref{eq:bnd_vjj'}, \eqref{eq:lwb_number}, and \eqref{eq:bnd_smallball} hold with probability at least $1-\delta/3$. 
The remainder of the proof is conditioned on the event that $\{\mb \xi_i\}_{i=1}^n$ satisfy \eqref{eq:bnd_vjj'}, \eqref{eq:lwb_number}, and \eqref{eq:bnd_smallball}. 

Let $\mb \beta^\star \in \mathbb{R}^{d+1}$, $\mb \beta \in \mathcal{N}(\mb \beta^\star)$, and $j \in [k]$ be arbitrarily fixed. 
For brevity, we will use the shorthand notation $\mb h_j := \mb \beta_j - \mb \beta_j^\star$. 
Then the left-hand side of \eqref{eq:res:lem:lwb_gradient} is rewritten as
\begin{align*}
\langle\nabla_{\mb \beta_j}\ell^{\mathrm{clean}}(\mb \beta),\mb h_j\rangle
&= \frac{1}{n}\sum_{i=1}^n\bbone_{\{\mb x_i \in \mathcal{C}_j\}} \left(\langle \mb \xi_i, \mb \beta_j \rangle-\max_{j\in[k]}\langle \mb \xi_i, \mb \beta_j^\star \rangle\right) \langle \mb \xi_i, \mb h_j \rangle \\
&= \frac{1}{n} \sum_{j'=1}^k \sum_{i=1}^n \bbone_{\{\mb x_i \in \mathcal{C}_j\cap\mathcal{C}_j^\star\}} \langle \mb \xi_i, \mb \beta_j - \mb \beta_{j'}^\star \rangle\langle\mb \xi_i,\mb h_j\rangle\\
&=\frac{1}{n} \sum_{i=1}^n \bbone_{\{\mb x_i \in \mathcal{C}_j\cap\mathcal{C}_j^\star\}} \langle \mb \xi_i, \mb h_j \rangle^2+\frac{1}{n}\sum_{j':j'\neq j}\sum_{i=1}^n\bbone_{\{\mb x_i\in\mathcal{C}_j\cap\mathcal{C}_{j'}^\star\}}\langle\mb \xi_i,\mb \beta_j-\mb \beta_{j'}^\star\rangle\langle\mb \xi_i,\mb h_j\rangle.
\end{align*}
By the inequality of arithmetic and geometric means, we have
\begin{align*}
\langle \mb \xi_i, \mb \beta_j - \mb \beta_{j'}^\star \rangle \langle \mb \xi_i, \mb h_j \rangle
& = \langle \mb \xi_i, \mb \beta_j - \mb \beta_j^\star + \mb \beta_j^\star - \mb \beta_{j'}^\star \rangle \langle \mb \xi_i, \mb h_j \rangle \\
& = \langle \mb \xi_i, \mb h_j + \mb v_{j,j'}^\star \rangle \langle \mb \xi_i, \mb h_j \rangle \\
& \geq \frac{\langle \mb \xi_i, \mb h_j \rangle^2}{2} - \frac{\langle \mb \xi_i, \mb v_{j,j'}^\star \rangle^2}{2} \geq - \frac{\langle \mb \xi_i, \mb v_{j,j'}^\star \rangle^2}{2}.
\end{align*} 
Therefore, we obtain 
\begin{align}
\langle\nabla_{\mb \beta_j}\ell^{\mathrm{clean}}(\mb \beta),\mb h_j\rangle \geq \underbrace{\frac{1}{n}\sum_{i=1}^n\bbone_{\{\mb x_i \in \mathcal{C}_j\cap \mathcal{C}_j^\star\}}\langle \mb \xi_i, \mb h_j \rangle^2}_{\mathrm{(*)}} 
- \underbrace{\frac{1}{2n}\sum_{j':j'\neq j}\sum_{i=1}^n\bbone_{\{\mb x_i \in \mathcal{C}_j\cap \mathcal{C}_{j'}^\star\}}\langle \mb \xi_i, \mb v_{j,j'}^\star \rangle^2}_{\mathrm{(**)}}.
\label{eq:midrslt0}
\end{align}
By \eqref{eq:lwb_number} and \eqref{eq:bnd_smallball}, the first summand in the right-hand side of \eqref{eq:midrslt0} is bounded from below as
\begin{equation}
\label{eq:lwb_fsterm}
\begin{aligned}
(*) 
&\geq \frac{2}{\gamma}\left(\frac{\pi_{\min}}{16}\right)^{1+\zeta^{-1}}\|\mb h_j\|_2^2.
\end{aligned}
\end{equation}
Moreover, due to \eqref{eq:bnd_vjj'}, $(**)$ is bounded from above as
\begin{equation}
\label{eq:thirdbnd}
\begin{aligned}
(**) \leq  \frac{1}{5\gamma k} \left(\frac{\pi_{\min}}{16}\right)^{1+\zeta^{-1}} \sum_{j':j'\neq j}\|\mb v_{j,j'}-\mb v_{j,j'}^\star\|_2^2.
\end{aligned}
\end{equation}
Then, plugging in \eqref{eq:lwb_fsterm} and \eqref{eq:thirdbnd} into \eqref{eq:midrslt0} provides 
\begin{equation*}
\begin{aligned}
&\langle\nabla_{\mb \beta_j}\ell(\mb \beta),\mb h_j\rangle\\
&\geq\frac{2}{\gamma}\left(\frac{\pi_{\min}}{16}\right)^{1+\zeta^{-1}}\|\mb h_j\|_2^2-\frac{1}{5\gamma}\left(\frac{1}{16}\right)^{1+\zeta^{-1}}\left(\frac{\pi_{\min}^{1+\zeta^{-1}}}{k}\right)\sum_{j':j'\neq j}\|\mb v_{j,j'}-\mb v_{j,j'}^\star\|_2^2\\
&=\frac{2}{\gamma}\left(\frac{\pi_{\min}}{16}\right)^{1+\zeta^{-1}}\left(\|\mb h_j\|_2^2-\frac{1}{10k}\sum_{j':j'\neq j}\left\|\mb v_{j,j'}-\mb v_{j,j'}^\star\right\|_2^2\right).
\end{aligned} 
\end{equation*} This completes the proof.\\


\noindent\textbf{Proof of \eqref{eq:res:lem:upbnd_normgrad}: }
The proof is based on the condition
\begin{equation}
\label{eq:cond:lem:upbnd_normgrad}
n\geq C_2\left(\log(k/\delta)\vee d\log(n/d)\right)k^4\pi_{\min}^{-4(1+\zeta^{-1})},
\end{equation} which is implied by \eqref{eq:samplcomp_noise}. 
We will proceed under the following four events, each of which holds with probability at least $1-\delta/12$. 
First, since \eqref{eq:cond:lem:upbnd_normgrad} implies \eqref{eq:samplecom_tailbounds}, by Lemma~\ref{lem:bnd_partial}, \eqref{eq:bnd_vjj'} holds with probability at least $1 - \delta/12$.  
Next, since $\left(\mathcal{C}_j^\star\right)_{j=1}^k$ are included in the set of intersection of $k$ half-spaces in $\mathbb{R}^d$, by Corollary~\ref{cor:concent_emp_prob} and \eqref{eq:cond:lem:upbnd_normgrad}, it holds with probability at least $1-\delta/12$ that 
\begin{equation}
\label{eq:hoff}
\frac{1}{n}\sum_{i=1}^{n}\bbone_{\{\mb x_i\in \mathcal{C}_j^\star\}}\leq 2\P\left(\mb x\in \mathcal{C}_j^\star\right),\quad\forall j\in[k]. 
\end{equation}
We also consider the event given by
\begin{equation}
\label{eq:unif_conv}
\sum_{i=1}^n \bbone_{\left\{\mb x_i\in \mathcal{C}_j\cap\mathcal{C}_j^\star\right\}} \leq 2n c\left(\frac{\pi_{\min}^{2(1+\zeta^{-1})}}{k^2}\right),\quad\forall j\neq j', ~ \forall\mb \beta\in\mathcal{N}(\mb \beta^\star)
\end{equation} 
for some numerical constant $c\in(0,1)$. Note that \eqref{eq:cond:lem:upbnd_normgrad} is a sufficient condition to invoke Lemma~\ref{lem:bnd_partial} with probability at least $1-\delta/12$. 
Therefore, all intermediate steps in the proof of Lemma~\ref{lem:bnd_partial} hold. 
In particular, due to the inclusion argument in \eqref{eq:inclusion}, $\mb x_i\in\mathcal{C}_j\cap\mathcal{C}_{j'}^\star$ implies $\mb \xi_i = [\mb x_i ; 1] \in\mathcal{S}_{\mb v_{j,j'},\mb v_{j,j'}^\star}$ for any $j\neq j'$, where $\mathcal{S}_{\mb v_{j,j'},\mb v_{j,j'}^\star}$ is defined in \eqref{eq:def_setS}. 
Then, \eqref{eq:Kvbnd} with $\alpha$ as in \eqref{eq:ub_exp_Svvstar} implies \eqref{eq:unif_conv}.
The last event is defined by
\begin{equation}
\label{eq:tanver_bound}
\begin{aligned}
& \max_{\begin{subarray}{c} \mathcal{I} \subset [n] \\ |\mathcal{I}|\leq 2\alpha n \end{subarray}} \lambda_{\max}\left(\frac{1}{n}\sum_{i\in\mathcal{I}}\mb \xi_i\mb \xi_i^\T\right) \leq C_4 (\eta^2 \vee 1)\sqrt{\alpha},
\quad\forall \alpha \in \left\{\frac{c \pi_{\min}^{2(1+\zeta^{-1})}}{k^2}\right\} \cup \left\{\P(\mb x\in\mathcal{C}_j^\star)\right\}_{j=1}^k.
\end{aligned}
\end{equation} 
By \eqref{eq:cond:lem:upbnd_normgrad}, Lemma~\ref{lem:upperbound}, and the union bound over $j \in [k]$, \eqref{eq:tanver_bound} holds with probability at least $1-\delta/12$. 
Thus far we have shown that \eqref{eq:bnd_vjj'}, \eqref{eq:hoff}, \eqref{eq:unif_conv}, and \eqref{eq:tanver_bound} hold with probability at least $1-\delta/3$. 
We proceed conditioned on the event that $\{\mb \xi_i\}_{i=1}^n$ satisfy these conditions. 

Let $\mb \beta^\star \in \mathbb{R}^{d+1}$, $\mb \beta \in \mathcal{N}(\mb \beta^\star)$, and $j \in [k]$ be arbitrarily fixed. 
Then the partial gradient of $\ell^\mathrm{clean}(\mb \beta)$ with respect to the $j$th block $\mb \beta_j \in \mathbb{R}^{d+1}$ of $\mb \beta\in \mathbb{R}^{k(d+1)}$ is written as
\begin{align}
    \nabla_{\mb \beta_j}\ell^{\mathrm{clean}}(\mb \beta)
    &=\frac{1}{n} \sum_{i=1}^n \bbone_{\{\mb x_i\in \mathcal{C}_j\}} \left(\langle\mb \xi_i,\mb \beta_j\rangle-\max_{j\in[k]}\langle\mb \xi_i,\mb \beta_j^\star\rangle\right)\mb \xi_i\nonumber\\
    &=\frac{1}{n}\sum_{j'\in[k]} \sum_{i=1}^n \bbone_{\{\mb x_i\in \mathcal{C}_j \cap \mathcal{C}_{j'}^\star \}} \left(\langle\mb \xi_i,\mb \beta_j\rangle-\langle\mb \xi_i,\mb \beta_{j'}^\star\rangle\right)\mb \xi_i\nonumber\\
    &=\frac{1}{n} \sum_{i=1}^n \bbone_{\{\mb x_i\in \mathcal{C}_j \cap \mathcal{C}_{j'}^\star \}} \langle\mb \xi_i,\mb \beta_j-\mb \beta_j^\star\rangle\mb \xi_i+\frac{1}{n}\sum_{j':j'\neq j} \sum_{i=1}^n \bbone_{\{\mb x_i\in \mathcal{C}_j \cap \mathcal{C}_{j'}^\star \}} \langle\mb \xi_i,\mb \beta_j-\mb \beta_{j'}^\star\rangle\mb \xi_i\label{start_result}.
\end{align}
By using the identity $\langle\mb \xi_i,\mb \beta_j-\mb \beta_{j'}^\star\rangle=\langle\mb \xi_i,\mb \beta_j-\mb \beta_j^\star+\mb \beta_j^\star-\mb \beta_{j'}^\star\rangle$, \eqref{start_result} is rewritten as
\begin{equation}
\label{eq:grad_clean}
\nabla_{\mb \beta_j}\ell^{\mathrm{clean}}(\mb \beta)=\frac{1}{n} \sum_{i=1}^n\bbone_{\{\mb x_i\in\mathcal{C}_j\}}\langle\mb \xi_i,\mb \beta_j-\mb \beta_j^\star\rangle\mb \xi_i+\frac{1}{n}\sum_{j':j'\neq j} \sum_{i=1}^n \bbone_{\{\mb x_i\in \mathcal{C}_j \cap \mathcal{C}_{j'}^\star \}} \langle\mb \xi_i,\mb \beta_j^\star-\mb \beta_{j'}^\star\rangle\mb \xi_i.
\end{equation}
Then it follows from \eqref{eq:grad_clean} that
{
\begin{align}
&\left\|\nabla_{\mb \beta_j}\ell^{\mathrm{clean}}(\mb \beta)\right\|_2^2 \nonumber \\
&\overset{\mathrm{(i)}}{\leq} 2 \left\|\frac{1}{n}\sum_{i=1}^n\bbone_{\{\mb x_i\in\mathcal{C}_j\}}\langle\mb \xi_i,\mb \beta_j-\mb \beta_j^\star\rangle\mb \xi_i\right\|_2^2
+ 2 \left\|\frac{1}{n}\sum_{j':j'\neq j} \sum_{i=1}^n \bbone_{\{\mb x_i\in \mathcal{C}_j \cap \mathcal{C}_{j'}^\star \}} \langle\mb \xi_i,\mb \beta_j^\star-\mb \beta_{j'}^\star\rangle\mb \xi_i\right\|_2^2 \nonumber \\
& \overset{\mathrm{(ii)}}{\leq} 2 \cdot \norm{\frac{1}{n}\sum_{i=1}^{n} \bbone_{\{\mb x_i\in \mathcal{C}_j\}} \mb \xi_i \mb \xi_i^\top} \cdot \frac{1}{n} \sum_{i=1}^n \bbone_{\{\mb x_i\in\mathcal{C}_j\}} \langle\mb \xi_i,\mb \beta_j-\mb \beta_j^\star\rangle^2\nonumber
\\
&\quad + 2 \cdot \sum_{j':j'\neq j} \norm{\frac{1}{n} \sum_{i=1}^{n} \bbone_{\{\mb x_i\in\mathcal{C}_j\cap\mathcal{C}\}}\mb \xi_i \mb \xi_i^\top} \cdot \frac{1}{n} \sum_{i=1}^n \bbone_{\{\mb x_i\in \mathcal{C}_j \cap \mathcal{C}_{j'}^\star \}} \langle\mb \xi_i,\mb \beta_j^\star-\mb \beta_{j'}^\star\rangle^2 \nonumber\\
&\leq 2 \cdot \underbrace{\norm{\frac{1}{n}\sum_{i=1}^{n} \bbone_{\{\mb x_i\in \mathcal{C}_j\}} \mb \xi_i \mb \xi_i^\top}^2}_{\mathrm{(a)}}  \cdot\|\mb \beta_j-\mb \beta_j^\star\|_2^2\nonumber\\
&\quad + 2 \cdot  \underbrace{\max_{j':j'\neq j}\norm{\frac{1}{n}\sum_{i=1}^{n} \bbone_{\{\mb x_i\mathcal{C}_j \cap \mathcal{C}_{j'}^\star\}} \mb \xi_i \mb \xi_i^\top}}_{\mathrm{(b)}}\cdot\underbrace{\frac{1}{n}\sum_{j':j'\neq j}\sum_{i=1}^n \bbone_{\{\mb x_i\in \mathcal{C}_j \cap \mathcal{C}_{j'}^\star \}} \langle\mb \xi_i,\mb \beta_j^\star-\mb \beta_{j'}^\star\rangle^2}_{\mathrm{(c)}}, 
\label{eq:decomposed_AB}
\end{align}
where (i) holds since $\|\mb a+\mb b\|_2^2\leq2\|\mb a\|_2^2+2\|\mb b\|_2^2$ and (ii) holds since $\mathcal{C}_j\cap \mathcal{C}_{l}^\star$ and $\mathcal{C}_j\cap \mathcal{C}_{l'}^\star$ are disjoint for any $l\neq l'\in[k]$. An upper bound on (b) is provided by \eqref{eq:bnd_vjj'}. It remains to derive upper bounds on (a) and (c).

First, we derive an upper bound on (a). By the triangle inequality, we have
\begin{equation}
\label{eq:ub_a_lem:normgrad}
\sqrt{\mathrm{(a)}} 
\leq \sum_{j'=1}^k \left\| \sum_{i=1}^{n}\bbone_{\{\mb x_i\in \mathcal{C}_j\cap \mathcal{C}_{j'}^\star\}}\mb \xi_i\mb \xi_i^\T \right\|.
\end{equation}
For the summand indexed by $j'=j$, due to the set inclusion $\mathcal{C}_j\cap \mathcal{C}_j^\star \subset \mathcal{C}_j^\star$, we obtain that
\begin{align*}
\sum_{i=1}^{n}\bbone_{\{\mb x_i\in \mathcal{C}_j\cap \mathcal{C}_j^\star\}}\mb \xi_i\mb \xi_i^\T
\preceq
\sum_{i=1}^{n}\bbone_{\{\mb x_i\in \mathcal{C}_j^\star\}}\mb \xi_i\mb \xi_i^\T.
\end{align*}
Therefore, by \eqref{eq:hoff} and \eqref{eq:tanver_bound}, we have  
\begin{equation}
\label{eq:reslt_A}
\begin{aligned}
\left\| \frac{1}{n}
\sum_{i=1}^n \bbone_{\{\mb x_i \in\mathcal{C}_j^\star\}} \mb \xi_i\mb \xi_i^\T \right\|
&\leq \max_{\mathcal{I}:|\mathcal{I}|\leq2n\P(\mb x\in\mathcal{C}_j^\star)} \left\| \frac{1}{n}
\sum_{i\in\mathcal{I}}\mb 
\xi_i\mb \xi_i^\T\right\| \\
&\lesssim(\eta^2\vee1)\sqrt{\P(\mb x\in \mathcal{C}_j^\star)} \\
& \leq(\eta^2\vee 1)\sqrt{\pi_{\max}},
\end{aligned}
\end{equation} 
where the last inequality holds by the definition of $\pi_{\max}$.
Similarly, by \eqref{eq:unif_conv} and \eqref{eq:tanver_bound}, we have
\begin{equation}
\label{eq:eigenbound2}
\left\| \sum_{i=1}^{n}\bbone_{\{\mb x_i\in \mathcal{C}_j\cap \mathcal{C}_{j'}^\star\}}\mb \xi_i\mb \xi_i^\T\right\|
\lesssim
(\eta^2\vee1)\sqrt{c}\left(\frac{\pi_{\min}^{1+\zeta^{-1}}}{k}\right), \quad \forall j' \neq j.
\end{equation}
Then by plugging in \eqref{eq:reslt_A} and \eqref{eq:eigenbound2} to \eqref{eq:ub_a_lem:normgrad}, we obtain 
\begin{equation*}
\mathrm{(a)} 
\lesssim \left({\pi_{\max}}+\pi_{\min}^{2(1+\zeta^{-1})}\right) \left\|\mb \beta_j-\mb \beta_j^\star\right\|_2^2
\end{equation*} for an absolute constant $C_1$.
Finally, since an upper bound on (b) is given by \eqref{eq:eigenbound2}, plugging in the obtained upper bounds to \eqref{eq:decomposed_AB} provides the assertion. \\
}

\noindent\textbf{Proof of \eqref{eq:res:lem:upbnd_noisegrad}: }
By the variational characterization of the Euclidean norm and the triangle inequality, we have
\begin{align}
\left\|\nabla_{\mb \beta_j}\ell^{\mathrm{noise}}(\mb \beta)\right\|_2
&= \sup_{[\mb u ;\ w] \in B_2^{d+1}}\left|\frac{1}{n}\sum_{i=1}^{n}z_i \bbone_{\{\mb x_i\in \mathcal{C}_j\}} (\langle \mb x_i, \mb u \rangle+w) \right| \nonumber \\
&\leq 
\underbrace{\sup_{\mb u \in B_2^p} \left|\frac{1}{n}\sum_{i=1}^{n}z_i \bbone_{\{\mb x_i\in \mathcal{C}_j\}} \langle \mb x_i, \mb u \rangle \right|}_{(\mathrm{A})}
+ 
\underbrace{\sup_{|w| \leq 1} \left|\frac{1}{n}\sum_{i=1}^{n}z_i \bbone_{\{\mb x_i\in \mathcal{C}_j\}} w \right|}_{(\mathrm{B})},
\label{eq:bnd_noise}
\end{align} 
where $B_2^d$ denotes the unit ball in $\ell_2^d$. 
Note that (A) and (B) depend on $\mb \beta$ only through $\mathcal{C}_j$, which are determined by $\mb \beta$ according to \eqref{eq:def_calCj}.
For any $\mb \beta$ and any $j \in [k]$, the corresponding $\mathcal{C}_j$ is given as the intersection of up to $k$ affine spaces. 
Therefore, it suffices to maximize $\left\|\nabla_{\mb \beta_j}\ell^{\mathrm{noise}}(\mb \beta)\right\|_2$ over $\mathcal{C}_j \in \mathcal{P}_{k-1}$ for a fixed $j$, where $\mathcal{P}_{k-1}$ is defined in the statement of \Cref{lem:bnd_growth_C}.

We proceed under the event that the following inequalities hold:
\begin{equation}
\label{eq:bnd_noise_event1}
\left\|\frac{1}{n} \sum_{i=1}^n\mb x_i \mb x_i^\T  \right\|
\leq 1 + \epsilon
\end{equation}
and
\begin{equation}
\label{eq:bnd_noise_event2}
\left|\frac{1}{n}\sum_{i=1}^{n}\bbone_{\{\mb x_i \in \mathcal{C}_j\}}-\mathbb{P}(\mb x \in \mathcal{C}_j)\right|
\leq \epsilon, 
\quad \forall \mathcal{C}_j \in \mathcal{P}_{k-1}
\end{equation}
for some constant $\epsilon$, which we specify later. The remainder of the proof is given conditioned on $(\mb x_i)_{i=1}^n$ satisfying \eqref{eq:bnd_noise_event1} and \eqref{eq:bnd_noise_event2}.

First, we derive an upper bound on (A) in \eqref{eq:bnd_noise}. 
Note that (A) corresponds to the supremum of the random process
\[
Z_{\mb u}:=\frac{1}{n} \sum_{i=1}^n z_i \bbone_{\{\mb x_i\in \mathcal{C}_j\}} \langle \mb x_i, \mb u \rangle
\]
over $\mb u \in B_2^p$. 
The sub-Gaussian increment satisfies
\begin{align*}
\|Z_{\mb u}-Z_{\mb u'}\|_{\psi_2}
& \lesssim 
\frac{\sigma}{\sqrt{n}}\sqrt{\frac{1}{n}\sum_{i=1}^{n}\bbone_{\{\mb x_i\in \mathcal{C}_j\}} \langle \mb x_i, \mb u-\mb u' \rangle^2} \\
& \leq 
\frac{\sigma}{\sqrt{n}}
\left\| \frac{1}{n}\sum_{i=1}^{n} \bbone_{\{\mb x_i\in \mathcal{C}_j\}} \mb x_i \mb x_i^\T \right\|^{1/2} \cdot \|\mb u-\mb u'\|_2 \\
&\leq 
\frac{\sigma}{\sqrt{n}}
\left\| \frac{1}{n}\sum_{i=1}^{n} \mb x_i \mb x_i^\T \right\|^{1/2} \cdot \|\mb u-\mb u'\|_2 \\
& \leq 
\frac{\sigma \sqrt{1 + \epsilon}}{\sqrt{n}} \cdot \|\mb u-\mb u'\|_2,
\end{align*} 
where the third step follows from the inequality 
\[
\left\| \frac{1}{n}\sum_{i=1}^{n} \bbone_{\{\mb x_i\in \mathcal{C}_j\}} \mb x_i \mb x_i^\T \right\|
\leq \left\| \frac{1}{n}\sum_{i=1}^{n} \mb x_i \mb x_i^\T \right\|,
\] 
which holds deterministically, and the last step follows from \eqref{eq:bnd_noise_event1}. 
Then, by applying a version of Dudley's inequality \cite[Theorem~8.1.6]{vershynin2018high}, we obtain that
\begin{equation*}
\mathbb{P}\left(
\sup_{\mb u\in {B}_2^p}|Z_{\mb u}|
>
\frac{C_1\sigma\sqrt{1 + \epsilon}}{\sqrt{n}}\left(\int_{0}^{\infty}\sqrt{\log N(B_2^p, \norm{\cdot}_2,\eta)} d\eta+\sqrt{\log(1/\delta)}\right) 
\right) \leq \delta. 
\end{equation*} 
By the elementary upper bound on the covering number $N(B_2^p, \norm{\cdot}_2,\eta)\leq\left(3/\eta\right)^p$ (e.g. see \cite[Example~8.1.11]{vershynin2018high}) and the definition of (A) in \eqref{eq:bnd_noise}, we have
\begin{equation}
\label{eq:bnd_u_dudley}
(\mathrm{A})
\lesssim
\sqrt{\frac{\sigma^2(1 + \epsilon)(d+\log(1/\delta))}{n}},
\end{equation}
holds with probability $1-\delta/3$. 
Then we apply the union bound over $\mathcal{C}_j \in \mathcal{P}_{k-1}$. 
It follows from \eqref{eq:bnd_growth_intersec} that 
\[
\sup_{\mathcal{C}_j \in \mathcal{P}_{k-1}}(\mathrm{A})
\lesssim
\sqrt{\frac{\sigma^2(1 + \epsilon)(\log(1/\delta)+kd\log(n/d))}{n}}
\] 
holds with probability $1-\delta/9$. 

Next we derive an upper bound on (B) in \eqref{eq:bnd_noise}. 
Note that (B) is rewritten as the absolute value of 
\[
\varrho = \frac{1}{n} \sum_{i=1}^n z_i \bbone_{\{\mb x_i \in \mathcal{C}_j\}}.
\]
Conditioned on $(\mb x_i)_{i=1}^n$ satisfying \eqref{eq:bnd_noise_event2}, $\varrho$ is a sub-Gaussian random variable that satisfies $\mathbb{E} \varrho = 0$ and 
\[
\mathbb{E} \varrho^2 = \frac{\sigma^2}{n} \cdot \left( \frac{1}{n} \sum_{i=1}^n \bbone_{\{\mb x_i \in \mathcal{C}_j\}} \right) 
\leq \frac{\sigma^2(\mathbb{P}(\mb x \in \mathcal{C}_j) + \epsilon)}{n}. 
\]
The standard sub-Gaussian tail bound implies
\[
\mathbb{P}\left(
|\varrho| > \sqrt{\frac{C_2 \sigma^2(\mathbb{P}(\mb x \in \mathcal{C}_j) + \epsilon) \log(1/\delta)}{n}}
\right) \leq \delta.
\]
By taking the union bound over $\mathcal{C}_j \in \mathcal{P}_{k-1}$ and utilizing the inequality in \eqref{eq:bnd_growth_intersec}, we obtain that
\begin{align}
\sup_{\mathcal{C}_j \in \mathcal{P}_{k-1}}(\mathrm{B}) &\lesssim \sqrt{\frac{\sigma^2(\mathbb{P}(\mb x \in \mathcal{C}_j) + \epsilon) \left(kd\log(n/d)+\log(1/\delta)\right)}{n}}\nonumber\\
&\leq\sqrt{\frac{\sigma^2(1 + \epsilon) \left(kd\log(n/d)+\log(1/\delta)\right)}{n}}\label{eq:supB}
\end{align}
holds with probability $1-\delta/9$. 

Finally it remains to show that \eqref{eq:bnd_noise_event1} and \eqref{eq:bnd_noise_event2} hold with probability $1-\delta/3$ for $\epsilon$ satisfying
\[
\epsilon \lesssim \sqrt{\frac{kp(\log(n/d) + \log(1/\delta))}{n}}.
\]
This is obtained as a direct consequence of Lemmas~\ref{lem:expect0} and \ref{lem:uniform}. 
One can choose the absolute constant $C$ in \eqref{eq:samplcomp_noise} large enough so that $\epsilon<1$. 
Then the parameter $\epsilon$ in \eqref{eq:bnd_u_dudley} and \eqref{eq:supB} will be dropped. 
This completes the proof.

\section{Proof of Theorem~\ref{thm:main_mini_batch}}
\label{sec:proof:minibatch}

The proof will be similar to that for Theorem~\ref{thm:main_noise}. 
We will focus on the distinction due to the modification of the algorithm with random sampling. 
The partial subgradient in the update for the mini-batch stochastic gradient descent algorithm is given by
\[
\frac{1}{m}\sum_{i\in I_t}\nabla_{\mb \beta_l} \ell_{i}(\mb \beta^t)=\frac{1}{m}\sum_{i\in I_t}\underbrace{\bbone_{\{\mb x_{i}\in\mathcal{C}_l\}}\left(\max_{j\in[k]}\langle\mb \xi_{i},\mb \beta_j^t\rangle-\max_{j\in[k]}\langle\mb \xi_{i},\mb \beta_j^\star\rangle\right)\mb \xi_{i}}_{\nabla_{\mb \beta_{l}}\ell_i^{\mathrm{clean}}(\mb \beta^t)}-\frac{1}{m}\sum_{i\in I_t}\underbrace{z_i\bbone_{\{\mb x_i\in\mathcal{C}_l\}}\mb \xi_i}_{\nabla_{\mb \beta_l}\ell_i^{\mathrm{noise}}(\mb \beta^t)},\]
where $\mathcal{C}_1, \dots, \mathcal{C}_k$ are determined by $\mb \beta^t$ as in \eqref{eq:def_calCj}.

As shown in \Cref{sec:thmproof}, \eqref{eq:samplcomp_noise} invokes Lemma~\ref{lem:lwb_gradient} and hence \eqref{eq:res:lem:lwb_gradient} holds with probability $1-\delta/3$. Next, we show that under the condition \eqref{eq:samplcomp_noise}, the statements of the following lemma hold with probability $1-2\delta/3$. The proof is provided in Appendix \ref{lem:proof:normbnd_minibatch}.

\begin{lemma}
\label{lem:upbnd_normgrad_minibatch}
Suppose that the hypothesis of Theorem~\ref{thm:main_mini_batch} holds. 
If \eqref{eq:samplcomp_noise} is satisfied, then the following statement holds with probability at least $1-2\delta/3$: For all $j \in [k]$, $\mb \beta^\star \in \mathbb{R}^{k(d+1)}$, and $\mb \beta^t \in \mathcal{N}(\mb \beta^\star)$, we have
\begin{equation}
\label{eq:res:lem:upbnd_normgrad_minibatch}
\begin{aligned}
&\E_{I_t}\left\| \frac{1}{m}\sum_{i\in I_t}\nabla_{\mb \beta_j}\ell_{i}^{\mathrm{clean}}(\mb \beta^t)\right\|_2^2
\lesssim \\
& \left(1\vee\frac{d+\log(n/\delta)}{m}\right)\left(\left(\sqrt{\pi_{\max}}+\pi_{\min}^{1+\zeta^{-1}}\right)\left\|\mb \beta_j^t-\mb \beta_j^\star\right\|_2^2 +\frac{\pi_{\min}^{1+\zeta^{-1}}}{k} \sum_{j':j'\neq j}\left\|\mb v_{j,j'}^t-\mb v_{j,j'}^\star\right\|_2^2\right),
\end{aligned}
\end{equation}
and
\begin{equation}
\label{eq:res:lem:upbnd_noisegrad_minibatch}
\begin{aligned}
{\E_{I_t}\left\|
\frac{1}{m}\sum_{i\in I_t}\nabla_{\mb \beta_j}\ell_i^{\mathrm{noise}}(\mb \beta^t)\right\|_2^2}\lesssim\sigma^2 \Bigg(\frac{d+\log(n/\delta)}{m}\vee\frac{kd\log(n/d) + \log(1/\delta)}{n}\Bigg).
\end{aligned}
\end{equation}
\end{lemma}



Then we show that the assertion of the theorem follows from \eqref{eq:res:lem:lwb_gradient}, \eqref{eq:res:lem:upbnd_normgrad_minibatch}, and \eqref{eq:res:lem:upbnd_noisegrad_minibatch} via the following three steps.\\

\noindent\textbf{Step~1:} 
We show that every iterate remains within the neighborhood $\mathcal{N}(\mb \beta^\star)$ by the induction argument. Therefore, we illustrate that if we suppose $\mb \beta^t \in \mathcal{N}(\mb \beta^\star)$ holds for a fixed $t \in \mbb{N}$, we show $\mb \beta^{t+1} \in \mathcal{N}(\mb \beta^\star)$ in expectation. By the update rule of SGD with batch size $m$, the triangle inequality gives
\begin{align}
\E_{I_t}\|\mb \beta_j^{t+1}-\mb \beta_j^\star\|_2&\leq\underbrace{\E_{I_t}\left\|\mb \beta_j^{t}-\mu\frac{1}{m}\sum_{i\in I_t}\nabla_{\mb \beta_j}\ell_i^{\mathrm{clean}}(\mb \beta^t)-\mb \beta_j^\star\right\|_2}_{A_\mathrm{clean}}
+\underbrace{\mu\E_{I_t}\left\|\frac{1}{m}\sum_{i\in I_t}\nabla_{\mb \beta_j}\ell_i^{\mathrm{noise}}(\mb \beta^t)\right\|_2}_{A_\mathrm{noise}}\label{eq:eql_noise_minibatch}.
\end{align}
We will show that
\begin{equation}
\label{eq:objective_step1_minibatch}
\E_{I_t}\|\mb \beta_j^{t+1}-\mb \beta_j^\star\|_2\leq A_{\mathrm{clean}}+A_{\mathrm{noise}}\leq \kappa\rho,\quad\forall j\in[k].
\end{equation}
By applying Jensen's inequality, we can obtain an upper-bound $A_{\mathrm{clean}}$ in \eqref{eq:eql_noise_minibatch}:
\begin{align}
A_{\mathrm{clean}}^2
& \leq{\E_{I_t}\left\|\mb \beta_j^{t}-\mu\cdot\frac{1}{m}\sum_{i\in I_t}\nabla_{\mb \beta_j}\ell_{i}^{\mathrm{clean}}(\mb \beta^t)-\mb \beta_j^\star\right\|_2^2}\nonumber\\
&=\|\mb \beta_j^{t}-\mb \beta_j^\star\|_2^2-2\mu\E_{I_t}\left\langle\frac{1}{m}\sum_{i\in I_t}\nabla_{\mb \beta_j}\ell_{i}^{\mathrm{clean}}(\mb \beta^t)
,\mb \beta_j^{t}-\mb \beta_j^\star\right\rangle+\mu^2\E_{I_t}\left\|\frac{1}{m}\sum_{i\in I_t}\nabla_{\mb \beta_j}\ell_{i}(\mb \beta^t)\right\|_2^2. \label{eq:sub_bnd1_minibatch}
\end{align}
Due to the expectation, the second term in \eqref{eq:sub_bnd1_minibatch} simplifies to
\begin{equation}
\label{eq:apply_expectation}
\E_{I_t}\left\langle\frac{1}{m}\sum_{i\in I_t}\nabla_{\mb \beta_j}\ell_{i}^{\mathrm{clean}}(\mb \beta^t)
,\mb \beta_j^{t}-\mb \beta_j^\star\right\rangle=\langle\nabla_{\mb \beta_j}\ell^{\mathrm{clean}}(\mb \beta^t),\mb \beta_j^t-\mb \beta_j^\star\rangle,
\end{equation} 
where $\nabla_{\mb \beta_j}\ell^{\mathrm{clean}}(\mb \beta^t)$ is defined in \eqref{eq:gradient_clean_def}.
Then, \eqref{eq:res:lem:lwb_gradient} gives a lower bound on \eqref{eq:apply_expectation}. Furthermore, an upper bound on the third term in \eqref{eq:sub_bnd1_minibatch} is given by \eqref{eq:res:lem:upbnd_normgrad_minibatch}. Putting the bounds \eqref{eq:res:lem:lwb_gradient} and \eqref{eq:res:lem:upbnd_normgrad_minibatch} in \eqref{eq:sub_bnd1_minibatch} provides
\begin{align}
&A_{\mathrm{clean}}^2\leq\nonumber\\
&\left(1 - \frac{4}{\gamma}\left(\frac{1}{16}\right)^{1+\zeta^{-1}} \mu\pi_{\min}^{1+\zeta^{-1}} + C_1\mu^2\left(1 \vee \frac{d+\log(n/\delta)}{m}\right)\left(\sqrt{\pi_{\max}}+\pi_{\min}^{1+\zeta^{-1}}\right)\right)\|\mb \beta_j^t-\mb \beta_j^\star\|_2^2 \nonumber \\
&\quad+\left(\frac{\frac{2}{\gamma}\left(\frac{1}{16}^{1+\zeta^{-1}}\right)\mu \pi_{\mathrm{min}}^{1+\zeta^{-1}}}{5k}+C_1\left(1 \vee \frac{d+\log(n/\delta)}{m}\right)\frac{\mu^2 \pi_{\min}^{1+\zeta^{-1}}}{k}\right)\sum_{j'^*:j'\neq j}\left\|\mb v_{j,j}^t-\mb v_{j,j'}^\star\right\|_2^2. \label{eq:proof_rec_prox_ub1_minibatch}
\end{align} 
Let us choose the step size $\mu$ following
\begin{equation}
\label{eq:muchoose_minibatch}
\mu=\frac{\omega\pi_{\min}^{1+\zeta^{-1}}}{\tau}\cdot\left(1\wedge\frac{m}{d+\log(n/\delta)}\right)
\end{equation}
for a numerical constant $\omega$, which we specify later, and $\tau$ defined as 
\begin{equation}
\label{eq:choose_tau_minibatch}
\tau:=\sqrt{\pi_{\max}}+\pi_{\min}^{1+\zeta^{-1}}.
\end{equation}
Taking $\mu$ by \eqref{eq:muchoose_minibatch} and $\tau$ by \eqref{eq:choose_tau_minibatch} in \eqref{eq:proof_rec_prox_ub1_minibatch} yields
\begin{equation}
\label{eq:estimate_bnd1_minibatch}
\begin{aligned}
& A_{\mathrm{clean}}^2\\
&\leq\Bigg(1-\left(1\wedge\frac{m}{d+\log(n/\delta)}\right)\cdot\\
&\quad\quad\left(\frac{\frac{4}{\gamma}\left(\frac{1}{16}\right)^{1+\zeta^{-1}} \omega \pi_{\min}^{2(1+\zeta^{-1})}}{\tau}-\frac{C_1\omega^2\pi_{\min}^{2(1+\zeta^{-1})}\left(\sqrt{\pi_{\max}}+\pi_{\min}^{1+\zeta^{-1}}\right)}{\tau^2}\right)\Bigg)\|\mb \beta_j^t-\mb \beta_j^\star\|_2^2
\\
&+ \left(1\wedge\frac{m}{d+\log(n/\delta)}\right)\cdot\left(\frac{\frac{2}{\gamma}\left(\frac{1}{16}\right)^{1+\zeta^{-1}} \omega\pi_{\min}^{2(1+\zeta^{-1})}}{5 \tau k}+\frac{C_1 \omega^2\pi_{\min}^{3 (1+\zeta^{-1})}}{\tau^2 k}\right)\sum_{j':j'\neq j}\left\|\mb v_{j,j'}^t-\mb v_{j,j'}^\star\right\|_2^2 \\
&\leq\left(1-\left(1\wedge\frac{m}{d+\log(n/\delta)}\right)\cdot\left(\frac{\frac{4}{\gamma}\left(\frac{1}{16}\right)^{1+\zeta^{-1}} \omega \pi_{\min}^{2(1+\zeta^{-1})}}{\tau}-\frac{C_1\omega^2\pi_{\min}^{2(1+\zeta^{-1})}}{\tau}\right)\right)\|\mb \beta_j^t-\mb \beta_j^\star\|_2^2
\\
&+ \left(1\wedge\frac{m}{d+\log(n/\delta)}\right)\cdot\left(\frac{\frac{2}{\gamma}\left(\frac{1}{16}\right)^{1+\zeta^{-1}} \omega\pi_{\min}^{2(1+\zeta^{-1})}}{5 \tau }+\frac{C_1 \omega^2\pi_{\min}^{2(1+\zeta^{-1})}}{\tau}\right)\max_{j\neq j'}\left\|\mb v_{j,j'}^t-\mb v_{j,j'}^\star\right\|_2^2.
\end{aligned} 
\end{equation}
Due to $\mb \beta^t \in \mathcal{N}(\mb \beta^\star)$ defined in \eqref{eq:defnbr}, we have \eqref{eq:rec_cond0} and \eqref{eq:rec_cond} by \Cref{lem:triangle_claim}. Inserting \eqref{eq:rec_cond0} and \eqref{eq:rec_cond} into \eqref{eq:estimate_bnd1_minibatch} gives
\begin{align}
(\kappa\rho)^{-2} A_{\mathrm{clean}}^2&\leq 1-\frac{\pi_{\min}^{2(1+\zeta^{-1})}\omega}{\tau}\left(1\wedge\frac{m}{d+\log(n/\delta)}\right)\left(\frac{4}{\gamma}\left(\frac{1}{16}\right)^{1+\zeta^{-1}} \left(1-\frac{2}{5}\right)+C_1\omega \left(1+4\right) \right)\nonumber\\
&= 1-\frac{\pi_{\min}^{2(1+\zeta^{-1})}\omega}{\tau}\left(1\wedge\frac{m}{d+\log(n/\delta)}\right){\left(\frac{\frac{12}{\gamma}\left(\frac{1}{16}\right)^{1+\zeta^{-1}}}{5} +5\omega C_1 \right)}\nonumber\\
&\leq1-\frac{c_0\omega\pi_{\min}^{2(1+\zeta^{-1})}}{\tau}\left(1\wedge\frac{m}{d+\log(n/\delta)}\right),\label{eq:minibatch_Aclean_bound}
\end{align}
where $c_0$ is the numerical constant defined in \eqref{eq:estiamte_bnd2}. We represent \eqref{eq:minibatch_Aclean_bound} as
\begin{equation}
\label{eq:estiamte_bnd3_minibatch}
A_\mathrm{clean}^2
\leq (\kappa\rho)^{2} \left(1 - \frac{c_0\omega \pi_{\min}^{2(1+\zeta^{-1})}}{\tau}\cdot\left(1\wedge\frac{m}{d+\log(n/\delta)}\right)\right).
\end{equation} We note that by \eqref{eq:estiamte_bnd2}, $c_0$ is a positive absolute constant given $\gamma$ and $\zeta$. On the other hand, the choice of $\tau$ in \eqref{eq:choose_tau_minibatch} provides a bound
\[
\frac{\pi_{\min}^{2(1+\zeta^{-1})}}{\tau}=\frac{\pi_{\min}^{2(1+\zeta^{-1})}}{\sqrt{\pi_{\max}}+\pi_{\min}^{1+\zeta^{-1}}}<1.
\] Since $\left(1\wedge m/(d+\log(n/\delta)\right)<1$, one can set $\omega>0$ such that $\omega c_0<1$, which makes the upper bound in the right-hand side of \eqref{eq:estiamte_bnd3_minibatch} a positive scalar belonging in $(0,1)$.

By following the arguments in \eqref{eq:suffice_noise1} and \eqref{eq:suffice_noise2}, if 
\begin{equation}
\label{eq:suffice_noise1_minibatch}
A_{\mathrm{noise}}\leq\kappa\rho\left(\frac{c_0\omega\pi_{\min}^{2(1+\zeta^{-1})}}{2\tau}\right)\left(1\wedge\frac{m} {d+\log(n/\delta)}\right)
\end{equation} holds, we have
\begin{equation}
\label{eq:suffice_noise1_minibatch1}
A_{\mathrm{noise}}\leq\kappa\rho\left(1-\sqrt{1-\frac{c_0\omega\pi_{\min}^{2(1+\zeta^{-1})}}{\tau}\left(1\wedge\frac{m}{d+\log(n/\delta)}\right)}\right).
\end{equation} 
Since the upper bounds in \eqref{eq:estiamte_bnd3_minibatch} and  \eqref{eq:suffice_noise1_minibatch1} satisfies \eqref{eq:objective_step1_minibatch} it suffices to show \eqref{eq:suffice_noise1_minibatch}.

By \eqref{eq:res:lem:upbnd_noisegrad_minibatch}, we have
\[
\sqrt{\E_{I_t}\left\|\frac{1}{m}\sum_{i\in I_t}\nabla_{\mb \beta_j}\ell_{i}^{\mathrm{noise}}(\mb \beta^t)\right\|_2^2}\lesssim\sigma\sqrt{\left(\frac{d+\log(n/\delta)}{m}\vee
\frac{{kd\log(n/d) + \log(1/\delta)}}{{n}}\right)}
\]  for all $j\in[k]$.
After applying Jensen's inequality, we consider the choice of $\mu$ given in \eqref{eq:muchoose_minibatch}. Then, we have
\begin{equation}
\label{eq:Anoise_minibatch_bound}
\begin{aligned}
&A_{\mathrm{noise}}=\mu\E_{I_t}\left\|\frac{1}{m}\sum_{i\in I_t}\nabla_{\mb \beta_j}\ell_i^{\mathrm{noise}}(\mb \beta^t)\right\|_2\leq\mu\sqrt{\E_{I_t}\left\|\frac{1}{m}\sum_{i\in I_t}\nabla_{\mb \beta_j}\ell_i^{\mathrm{noise}}(\mb \beta^t)\right\|_2^2}\lesssim\\ &\frac{\sigma\omega\pi_{\min}^{1+\zeta^{-1}}}{\tau}\left(1\wedge\frac{m}{d+\log(n/\delta)}\right)\sqrt{\left(\frac{d+\log(n/\delta)}{m}\vee
\frac{{kd\log(n/d) + \log(1/\delta)}}{{n}}\right)}.
\end{aligned}
\end{equation}
Since \eqref{eq:samplcomp_noise} implies  \eqref{eq:sample_noise1}, we can choose a sufficiently large absolute constant $C>0$ in \eqref{eq:sample_noise1} such that \eqref{eq:sample_noise1} and \eqref{eq:Anoise_minibatch_bound} result in \eqref{eq:suffice_noise1_minibatch}. We complete the proof of induction argument in Step 1.

\noindent\textbf{Step~2:} 
In this step, we show that every iterate obeys
\begin{equation}
\begin{aligned}
\label{eq:recursion_decrease_minibatch}
&\E_{I_t}\left\|\mb \beta^{t+1}-\mb \beta^\star\right\|_2
\leq \\
&\sqrt{1-\nu} \left\|\mb \beta^t-\mb \beta^\star\right\|_2+C'\mu\sigma \sqrt{k} \cdot\left(\sqrt{\frac{d+\log(n/\delta)}{m}} \vee \sqrt{\frac{kd\log(n/d) + \log(1/\delta)}{n}}\right).
\end{aligned}
\end{equation}
In Step 1, we showed $\mb \beta^t \in \mathcal{N}(\mb \beta^{\star})$.
By following the argument \eqref{eq:eql_noise_minibatch}, we have
{
\begin{equation}
\begin{aligned}
\label{eq:updaterule_minibatch}
&\E_{I_t}\|\mb \beta^{t+1}-\mb \beta^\star\|_2
\leq{\E_{I_t}\left\|\mb \beta^t-\mu\frac{1}{m}\sum_{i\in I_t}\nabla_{\mb \beta}\ell_{i}^{\mathrm{clean}}(\mb \beta^t)-\mb \beta^\star\right\|_2}+{{\E_{I_t}\left\|\frac{1}{m}\sum_{i\in I}\nabla_{\mb \beta}\ell_i^{\mathrm{noise}}(\mb \beta^t)\right\|_2}}\\
&\leq\underbrace{\sqrt{\E_{I_t}\left\|\mb \beta^t-\mu\frac{1}{m}\sum_{i\in I_t}\nabla_{\mb \beta}\ell_{i}^{\mathrm{clean}}(\mb \beta^t)-\mb \beta^\star\right\|_2^2}}_{B_\mathrm{clean}}+\underbrace{\sqrt{\E_{I_t}\left\|\frac{1}{m}\sum_{i\in I}\nabla_{\mb \beta}\ell_i^{\mathrm{noise}}(\mb \beta^t)\right\|_2^2}}_{B_{\mathrm{noise}}},
\end{aligned}
\end{equation} where the last inequality holds by the Jensen's inequality.}
We first show an upper bound on $B_\mathrm{clean}$ in \eqref{eq:updaterule_minibatch}:
\begin{align}
\label{eq:recursion_decrease_clean_minibatch}
B_\mathrm{clean}^2
\leq (1-\nu) \sum_{j=1}^{k}\left\|\mb \beta_j^t-\mb \beta_j^\star\right\|_2^2.
\end{align}
By following the argument in \eqref{eq:noiseless_contraction}, \eqref{eq:recursion_decrease_clean_minibatch} holds if there exist constants $\mu, \lambda \in (0,1)$ such that for all  $\mb \beta^t \in \mathcal{N}(\mb \beta^{\star})$,
\begin{equation}
\label{eq:regulcond_maxlinear_minibatch}
\sum_{j=1}^{k}\E_{I_t}\left\langle\frac{1}{m}\sum_{i\in I_t}\nabla_{\mb \beta_j}\ell_{i}^{\mathrm{clean}}(\mb \beta^t),\mb \beta_j^t - \mb \beta_j^\star\right\rangle\geq\frac{\mu}{2}\sum_{j=1}^{k}\E_{I_t}\left\|\frac{1}{m}\sum_{i\in I_t}\nabla_{\mb \beta_j}\ell_{i}^{\mathrm{clean}}(\mb \beta^t)\right\|_2^2+\frac{\lambda}{2}\sum_{j=1}^{k}\|\mb \beta_j^t - \mb \beta_j^\star\|_2^2.
\end{equation}
Hence, we show \eqref{eq:regulcond_maxlinear_minibatch}.First, since \eqref{eq:res:lem:lwb_gradient} holds, \eqref{eq:lwb_result} holds. Also, the left-hand side in \eqref{eq:regulcond_maxlinear_minibatch} can be computed as \eqref{eq:apply_expectation}. Thus, by \eqref{eq:apply_expectation} and \eqref{eq:lwb_result}, we obtain a lower bound on the left-hand side of \eqref{eq:regulcond_maxlinear_minibatch}:
\begin{align}
\sum_{j=1}^k\E_{I_t}\left\langle\frac{1}{m}\sum_{i\in I_t}\nabla_{\mb \beta_j}\ell_{i}^{\mathrm{clean}}(\mb \beta^t),\mb \beta_j^t - \mb \beta_j^\star\right\rangle
\geq \frac{ \frac{6}{\gamma}\left(\frac{1}{16}\right)^{1+\zeta^{-1}}\pi_{\min}^{1+\zeta^{-1}}}{5} \sum_{j=1}^{k}\|\mb \beta_j^t-\mb \beta_j^\star\|_2^2.\label{eq:lwb_result_minibatch}
\end{align}

Furthermore, to obtain an upper bound on first term in the right-hand side of \eqref{eq:regulcond_maxlinear_minibatch}, applying \eqref{eq:res:lem:upbnd_normgrad_minibatch} with the elementary inequality $\|\mb a+\mb b\|_2^2\leq2\|\mb a\|_2^2+2\|\mb b\|_2^2$ provides
\begin{equation}
\label{eq:result_up_minibatch}
\begin{aligned}
\E_{I_t}\left\| \frac{1}{m}\sum_{i\in I_t}\nabla_{\mb \beta_j}\ell_{i}^{\mathrm{clean}}(\mb \beta^t)\right\|_2^2&\leq C_1\left(1\vee\frac{d+\log(n/\delta)}{m}\right)\Bigg(\left(\sqrt{\pi_{\max}}+\pi_{\min}^{1+\zeta^{-1}}\right)\|\mb \beta_j^t-\mb \beta_j^\star\|_2^2\\
&+\frac{2\pi_{\min}^{1+\zeta^{-1}}}{k}\sum_{j':j'\neq j}\left(\|\mb \beta_j^t-\mb \beta_j^\star\right\|_2^2+\left\|\mb \beta_{j'}^t-\mb \beta_{j'}^\star\|_2^2\right)\Bigg).
\end{aligned}
\end{equation} 
Taking summation on \eqref{eq:result_up_minibatch} over $j \in [k]$ yields  
{
\begin{equation}
\label{eq:RHSbnd1_minibatch}
\begin{aligned}
&\sum_{j=1}^k \E_{I_t}\left\| \frac{1}{m}\sum_{i\in I_t}\nabla_{\mb \beta_j}\ell_{i}^{\mathrm{clean}}(\mb \beta^t)\right\|_2^2\\
&\leq {C_1 \left(1\vee\frac{d+\log(n/\delta)}{m}\right)}\left( \sqrt{\pi_{\max}}+\pi_{\min}^{1+\zeta^{-1}} + {4\pi_{\min}^{1+\zeta^{-1}}} \right)
\sum_{j=1}^k \left\|\mb \beta_j^t-\mb \beta_j^\star\right\|_2^2.
\end{aligned}
\end{equation} }
Putting the bounds \eqref{eq:lwb_result_minibatch} and \eqref{eq:RHSbnd1_minibatch} in \eqref{eq:regulcond_maxlinear_minibatch} with $\mu$ chosen in \eqref{eq:muchoose_minibatch}, we have a sufficient condition for \eqref{eq:regulcond_maxlinear_minibatch}:
{
\begin{equation}
\label{eq:RClwup_minibatch}
\frac{ \frac{6}{\gamma}\left(\frac{1}{16}\right)^{1+\zeta^{-1}} \pi_{\min}^{1+\zeta^{-1}}}{5}
\geq
\frac{\omega\pi_{\min}^{1+\zeta^{-1}} C_1 \left( \sqrt{\pi_{\max}}+5\pi_{\min}^{1+\zeta^{-1}}\right)}{2\left( \sqrt{\pi_{\max}}+\pi_{\min}^{1+\zeta^{-1}}\right)}
+ \frac{\lambda}{2}.
\end{equation}}
\eqref{eq:RClwup_minibatch} is satisfied when we choose $\omega > 0$ small enough and $\lambda$ as in \eqref{eq:lambdachoose}. Hence, we have shown \eqref{eq:recursion_decrease_clean_minibatch} with $\nu=\mu\lambda$ where $\mu$ and $\lambda$ are chosen by \eqref{eq:muchoose_minibatch} and \eqref{eq:lambdachoose}.

Next, we bound $B_{\mathrm{noise}}$ in \eqref{eq:updaterule_minibatch}. By \eqref{eq:res:lem:upbnd_noisegrad_minibatch}, we obtain an upper bound on $B_{\mathrm{noise}}$:
\begin{equation}
\label{eq:noise_bnd_B_minibatch}
\begin{aligned}
B_\mathrm{noise}^2&=\mu^2\sum_{j=1}^{k}\E_{I_t}\left\|\frac{1}{m}\sum_{i\in I_t}\nabla_{\mb \beta_j}\ell_i^{\mathrm{noise}}(\mb \beta^t)\right\|_2^2\\ 
&\lesssim
k\mu^2\sigma^2\left(\frac{d+\log(n/\delta)}{m}\vee\frac{kd\log(n/d) + \log(1/\delta)}{n}\right).
\end{aligned}
\end{equation}
Finally, putting \eqref{eq:recursion_decrease_clean_minibatch} and \eqref{eq:noise_bnd_B_minibatch} in \eqref{eq:updaterule_minibatch} gives  \eqref{eq:recursion_decrease_minibatch}. We complete the proof of Step 2.\\

{
\noindent\textbf{Step~3:} 
We finish the proof of \Cref{thm:main_mini_batch} using the results demonstrated in Step 1 and Step 2. By substituting the expression $\nu=\mu\lambda$ , where we choose $\mu$ and $\lambda$ according to \eqref{eq:muchoose_minibatch} and \eqref{eq:lambdachoose} respectively, into \eqref{eq:recursion_decrease_minibatch}, we obtain
\begin{align*}
&\E_{I_t}\|\mb \beta^{t}-\mb \beta^\star\|_2\\
&\left({1-\mu\lambda}\right)^{t/2}\|\mb \beta^0-\mb \beta^\star\|_2+C_2\cdot \frac{\mu\sigma}{1-\sqrt{1-\mu\lambda}} \cdot \sqrt{k\cdot\left(\frac{d+\log(n/\delta)}{m}\vee\frac{kd\log(n/d)+\log(1/\delta)}{n}\right)}\\
&\overset{\mathrm{(a)}}{\leq}\left({1-\mu\lambda}\right)^{t/2}\|\mb \beta^0-\mb \beta^\star\|_2+C_2 \cdot \frac{2\sigma}{\lambda}\cdot \sqrt{k\cdot\left(\frac{d+\log(n/\delta)}{m}\vee\frac{kd\log(n/d)+\log(1/\delta)}{n}\right)}\\&
\overset{\mathrm{(b)}}{\leq}\left({1-\mu\lambda}\right)^{t/2}\|\mb \beta^0-\mb \beta^\star\|_2+C_3 \cdot \frac{\sigma}{\pi_{\max}}\cdot\sqrt{k\cdot\left(\frac{d+\log(n/\delta)}{m}\vee\frac{kd\log(n/d)+\log(1/\delta)}{n}\right)}\\&
\overset{\mathrm{(c)}}{\leq}\left({1-\mu\lambda}\right)^{t/2}\|\mb \beta^0-\mb \beta^\star\|_2+C_3 \cdot {\sigma k}\cdot \sqrt{k\cdot\left(\frac{d+\log(n/\delta)}{m}\vee\frac{kd\log(n/d)+\log(1/\delta)}{n}\right)}, 
\end{align*} 
where i) $\mathrm{(a)}$ follows from the inequality $\sqrt{1-t}<-t/2+1$ for any $t\in(0,1)$; ii) $\mathrm{(b)}$ holds by the choice of $\tau$ in \eqref{eq:choose_tau_minibatch}; iii) $\mathrm{(c)}$ is a result of $\pi_{\max}^{-1}\leq k$.}

\subsection{Proof of Lemma \ref{lem:upbnd_normgrad_minibatch}}
\label{lem:proof:normbnd_minibatch}

We will show that both \eqref{eq:res:lem:upbnd_normgrad_minibatch} and \eqref{eq:res:lem:upbnd_noisegrad_minibatch} hold with probability at least $1-\delta/3$. Furthermore, for simplicity, we proceed on the proofs using $\mb \beta$ and $\mb v_{j,j'}$ instead of using $\mb \beta^t$ and $\mb v_{j,j'}^t$ in the statements of \Cref{lem:upbnd_normgrad_minibatch}. Thus, we complete the assertions in \eqref{eq:res:lem:upbnd_normgrad_minibatch} and \eqref{eq:res:lem:upbnd_noisegrad_minibatch} by substituting $\mb \beta$ and $\mb v_{j,j'}$ with $\mb \beta^t$ and $\mb v_{j,j'}^t$ respectively.

\noindent\textbf{Proof of \eqref{eq:res:lem:upbnd_normgrad_minibatch}: }
We show that with high probability, \eqref{eq:res:lem:upbnd_normgrad_minibatch} holds if 
\begin{equation}
\label{eq:cond:lem:upbnd_normgrad_minibatch}
n \geq C_1 \left(\log(k/\delta) \vee d\log(n/d)\right)k^4\pi_{\min}^{-4(1+\zeta^{-1})},
\end{equation}. Note that \eqref{eq:samplcomp_noise} is a sufficient condition for \eqref{eq:cond:lem:upbnd_normgrad_minibatch}. We proceed with the proof under the following six events, each of which holds with probability at least $1-\delta/18$. First, by the proof of \eqref{eq:res:lem:upbnd_normgrad} in \Cref{sec:proof:lem:lwb_gradient}, \eqref{eq:cond:lem:upbnd_normgrad_minibatch} is a sufficient condition to invoke \eqref{eq:res:lem:upbnd_normgrad} with probability at least $1-\delta/18$. 
Next, by following the argument for \eqref{eq:hoff}, \eqref{eq:cond:lem:upbnd_normgrad_minibatch} is a sufficient condition to invoke \eqref{eq:hoff} with probability at least $1-\delta/18$. 
Furthermore, \eqref{eq:cond:lem:upbnd_normgrad_minibatch} implies \eqref{eq:samplecom_tailbounds} and is a sufficient condition to invoke Lemma~\ref{lem:bnd_partial} and Lemma~\ref{lem:upperbound} with probability at least $1-\delta/18$ respectively. Hence, by following the arguments for \eqref{eq:unif_conv}, \eqref{eq:tanver_bound}, and \eqref{eq:bnd_vjj'}, \eqref{eq:unif_conv}, \eqref{eq:tanver_bound}, and \eqref{eq:bnd_vjj'} hold with probability at least $1-\delta/18$ respectively. The last event is defined as 
\begin{equation}
\label{eq:lem:expect0}
\max_{i\in[n]}\left\|\mb \xi_i \mb \xi_i^\T\right\|\lesssim d+\log(n/\delta).
\end{equation}
By Lemma~\ref{lem:expect0} and the union bound over $i\in[n]$, \eqref{eq:lem:expect0} holds with probability at least $1-\delta/18$.

Since we have shown that \eqref{eq:res:lem:upbnd_normgrad}, \eqref{eq:hoff}, \eqref{eq:unif_conv}, \eqref{eq:tanver_bound}, \eqref{eq:bnd_vjj'}, and  \eqref{eq:lem:expect0} hold with probability at least $1-\delta/3$, we will move forward with the remainder of the proof by assuming those conditions are satisfied.

Let $\mb \beta^\star \in \mathbb{R}^{d+1}$, $\mb \beta \in \mathcal{N}(\mb \beta^\star)$, and $j \in [k]$ be arbitrarily fixed. By the argument in \cite[Equation~7]{ma2017power}, we decompose
\begin{equation}
\label{eq:decompose_norm_signle_gradient}
\E_{I}\left\| \frac{1}{m}\sum_{i\in I}\nabla_{\mb \beta_j}\ell_{i}^{\mathrm{clean}}(\mb \beta)\right\|_2^2=\underbrace{\frac{1}{m}\E_{i_1}\left\|\nabla_{\mb \beta_j}\ell_{i_1}^{\mathrm{clean}}(\mb \beta)\right\|_2^2}_{\mathrm{(A)}}+\underbrace{\frac{m-1}{m}\|\nabla_{\mb \beta_j}\ell^{\mathrm{clean}}(\mb \beta)\|_2^2}_{\mathrm{(B)}},
\end{equation} where we define $I:=\{i_1,\ldots,i_m\}\subset[n]$ and $\nabla_{\mb \beta_j}\ell^{\mathrm{clean}}(\mb \beta)$ in \eqref{eq:gradient_clean_def}.

Note that \eqref{eq:res:lem:upbnd_normgrad} gives an upper bound on (B):
\begin{equation}
\label{eq:minibatch:bnd_B}
\begin{aligned}
\mathrm{(B)}\lesssim \frac{m-1}{m}\left(\left({\pi_{\max}}+\pi_{\min}^{2(1+\zeta^{-1})}\right)\left\|\mb \beta_j-\mb \beta_j^\star\right\|_2^2 + \frac{\pi_{\min}^{2(1+\zeta^{-1})}}{k^2} \sum_{j':j'\neq j}\left\|\mb v_{j,j'}-\mb v_{j,j'}^\star\right\|_2^2\right).
\end{aligned} 
\end{equation} 

It remains to show the bound on $\mathrm{(A)}$. By following arguments \eqref{eq:grad_clean}, we decompose $\nabla_{\mb \beta_j}\ell_{i}^{\mathrm{clean}}(\mb \beta)$ following
\begin{equation}
\label{eq:grad_clean_minibatch}
\nabla_{\mb \beta_j}\ell_{i}^{\mathrm{clean}}(\mb \beta)= \bbone_{\left\{\mb x_{i}\in\mathcal{C}_j\right\}}\langle\mb \xi_{i},\mb \beta_j-\mb \beta_j^\star\rangle\mb \xi_{i}+\sum_{j':j'\neq j} \bbone_{\left\{\mb x_{i}\in \mathcal{C}_j \cap \mathcal{C}_{j'}^\star \right\}} \langle\mb \xi_{i},\mb \beta_j^\star-\mb \beta_{j'}^\star\rangle\mb \xi_{i},\quad\forall i\in[n].
\end{equation} 
Then it follows from \eqref{eq:grad_clean_minibatch} that for any $i\in[n]$,
\begin{align}
&\left\|\nabla_{\mb \beta_j}\ell_{i}^{\mathrm{clean}}(\mb \beta)\right\|_2^2 \nonumber \\
&\overset{\mathrm{(i)}}{\leq} 2 \left\|\bbone_{\left\{\mb x_{i}\in\mathcal{C}_j\right\}}\langle\mb \xi_{i},\mb \beta_j-\mb \beta_j^\star\rangle\mb \xi_{i}\right\|_2^2
+ 2 \left\|\sum_{j':j'\neq j}  \bbone_{\left\{\mb x_{i}\in \mathcal{C}_j \cap \mathcal{C}_{j'}^\star \right\}} \langle\mb \xi_{i},\mb \beta_j^\star-\mb \beta_{j'}^\star\rangle\mb \xi_{i}\right\|_2^2 \nonumber \\
&\overset{\mathrm{(ii)}}{=} 2 \cdot \norm{ \mb \xi_{i} \mb \xi_{i}^\top}\bbone_{\left\{\mb x_{i}\in\mathcal{C}_j\right\}} \langle\mb \xi_{i},\mb \beta_j-\mb \beta_j^\star\rangle^2
+ 2 \cdot \norm{\mb \xi_{i} \mb \xi_{i}^\top} \cdot \sum_{j':j'\neq j}  \bbone_{\left\{\mb x_{i}\in \mathcal{C}_j \cap \mathcal{C}_{j'}^\star \right\}} \langle\mb \xi_{i},\mb \beta_j^\star-\mb \beta_{j'}^\star\rangle^2 \nonumber \\
&\overset{\mathrm{(iii)}}{\lesssim} (d+\log(n/\delta)) \cdot\left( {\bbone_{\{\mb x_{i}\in\mathcal{C}_j\}} \langle\mb \xi_{i},\mb \beta_j-\mb \beta_j^\star\rangle^2}
+  { \sum_{j':j'\neq j}  \bbone_{\{\mb x_{i}\in \mathcal{C}_j \cap \mathcal{C}_{j'}^\star \}} \langle\mb \xi_{i},\mb \beta_j^\star-\mb \beta_{j'}^\star\rangle^2}\right),
\label{eq:decomposed_minibatch}
\end{align}
where (i) holds due to $\|\mb a+\mb b\|_2^2\leq2\|\mb a\|_2^2+2\|\mb b\|_2^2$; (ii) holds since $\mathcal{C}_j\cap \mathcal{C}_{l}^\star$ and $\mathcal{C}_j\cap \mathcal{C}_{l'}^\star$ are disjoint for any $l\neq l'\in[k]$; and (iii) holds by \eqref{eq:lem:expect0}.

Applying the expectation on \eqref{eq:decomposed_minibatch} yields
\begin{equation}
\label{eq:decomposed_AB_minibatch}
\begin{aligned}
&\E_{i_1}\left\|\nabla_{\mb \beta_j}\ell_{i_1}(\mb \beta)\right\|_2^2\lesssim \\
&\left(d+\log(n/\delta)\right) \cdot\left( \underbrace{\frac{1}{n}\sum_{i=1}^n \bbone_{\{\mb x_i\in\mathcal{C}_j\}} \langle\mb \xi_i,\mb \beta_j-\mb \beta_j^\star\rangle^2}_{\mathrm{(a)}}
+ \frac{1}{n} \underbrace{ \sum_{j':j'\neq j} \sum_{i=1}^n \bbone_{\{\mb x_i\in \mathcal{C}_j \cap \mathcal{C}_{j'}^\star \}} \langle\mb \xi_i,\mb \beta_j^\star-\mb \beta_{j'}^\star\rangle^2}_{\mathrm{(b)}}\right).
\end{aligned}
\end{equation}
An upper bound on (b) is provided by \eqref{eq:bnd_vjj'}. It remains to derive an upper bound on (a). 

The triangle inequality provides
\begin{equation}
\label{eq:ub_a_lem:normgrad_minibatch}
\mathrm{(a)} 
\leq \sum_{j'=1}^k \left\| \sum_{i=1}^{n}\bbone_{\{\mb x_i\in \mathcal{C}_j\cap \mathcal{C}_{j'}^\star\}}\mb \xi_i\mb \xi_i^\T \right\| \cdot \left\|\mb \beta_j-\mb \beta_j^\star\right\|_2^2
\end{equation}
For the summand indexed by $j'=j$, the set inclusion, $\mathcal{C}_j\cap \mathcal{C}_j^\star \subseteq \mathcal{C}_j^\star$ yields 
\begin{align*}
\sum_{i=1}^{n}\bbone_{\{\mb x_i\in \mathcal{C}_j\cap \mathcal{C}_j^\star\}}\mb \xi_i\mb \xi_i^\T
\preceq
\sum_{i=1}^{n}\bbone_{\{\mb x_i\in \mathcal{C}_j^\star\}}\mb \xi_i\mb \xi_i^\T.
\end{align*}
Therefore, by \eqref{eq:hoff} and \eqref{eq:tanver_bound}, we have  
\begin{equation}
\label{eq:reslt_A_minibatch}
\begin{aligned}
\left\| \frac{1}{n}
\sum_{i=1}^n \bbone_{\{\mb x_i \in\mathcal{C}_j^\star\}} \mb \xi_i\mb \xi_i^\T \right\|
&\leq \max_{\mathcal{I}:|\mathcal{I}|\leq2n\P(\mb x\in\mathcal{C}_j^\star)} \left\| \frac{1}{n}
\sum_{i\in\mathcal{I}}\mb 
\xi_i\mb \xi_i^\T\right\| \\
&\lesssim(\eta^2\vee1)\sqrt{\P(\mb x\in \mathcal{C}_j^\star)} \\
& \leq (\eta^2\vee1)\sqrt{\pi_{\max}},
\end{aligned}
\end{equation} 
where the last inequality holds by the definition of $\pi_{\max}$ in \eqref{}.
Similarly, by \eqref{eq:unif_conv} and \eqref{eq:tanver_bound}, we have
\begin{equation}
\label{eq:eigenbound2_minibatch}
\left\| \sum_{i=1}^{n}\bbone_{\{\mb x_i\in \mathcal{C}_j\cap \mathcal{C}_{j'}^\star\}}\mb \xi_i\mb \xi_i^\T\right\|
\lesssim(\eta^2\vee1)\sqrt{c}\left(\frac{\pi_{\min}^{1+\zeta^{-1}}}{k}\right), \quad \forall j' \neq j.
\end{equation}
Then by plugging in \eqref{eq:reslt_A_minibatch} and \eqref{eq:eigenbound2_minibatch} into \eqref{eq:ub_a_lem:normgrad_minibatch}, we obtain 
\begin{equation*}
\mathrm{(a)} 
\lesssim\left(\sqrt{\pi_{\max}}+\pi_{\min}^{1+\zeta^{-1}}\right) \left\|\mb \beta_j-\mb \beta_j^\star\right\|_2^2.
\end{equation*} 
Finally, applying obtained upper bounds on (a) and (b) in \eqref{eq:decomposed_AB_minibatch} gives
\begin{equation}
\label{eq:resultlargeA}
\mathrm{(A)}\lesssim\frac{\left(d+\log(n/\delta)\right)}{m}\left(\left(\sqrt{\pi_{\max}}+\pi_{\min}^{(1+\zeta^{-1})}\right)\left\|\mb \beta_j-\mb \beta_j^\star\right\|_2^2 +\frac{\pi_{\min}^{(1+\zeta^{-1})}}{k} \sum_{j':j'\neq j}\left\|\mb v_{j,j'}-\mb v_{j,j'}^\star\right\|_2^2\right).
\end{equation} 
Putting \eqref{eq:minibatch:bnd_B} and \eqref{eq:resultlargeA} in \eqref{eq:decompose_norm_signle_gradient} completes the proof.

\noindent\textbf{Proof of \eqref{eq:res:lem:upbnd_noisegrad_minibatch}: }
We proceed with the proof under the following three events, each of which holds with probability at least $1-\delta/9$. First, \eqref{eq:samplcomp_noise} invokes \eqref{eq:res:lem:upbnd_noisegrad} with probability at least $1-\delta/9$. Next, by following the same argument in the proof of \eqref{eq:res:lem:upbnd_normgrad_minibatch}, \eqref{eq:lem:expect0} holds with probability at least $1-\delta/9$. The last event is the following:
\begin{equation}
\label{eq:bound_noisesquare}
\frac{1}{n}\sum_{i=1}^{n}z_i^2\leq\sigma^2\left(1+\sqrt{\frac{C\log(1/\delta)}{n}}\right).
\end{equation} Since $\{z_i\}_{i=1}^n$ are i.i.d $\sigma$-sub-Gaussian random variables, the Bernstein's inequality yields that \eqref{eq:bound_noisesquare} holds with probability at least $1-\delta/9$.

We have shown that \eqref{eq:res:lem:upbnd_noisegrad}, \eqref{eq:lem:expect0}, and \eqref{eq:bound_noisesquare} hold with probability at least $1-\delta/3$. For the remainder of the proof, we assume that those conditions are satisfied.

Then, by the argument in \cite[Equation~7]{ma2017power}, we decompose
\begin{equation}
\label{eq:decompose_norm_noise_gradient}
\E_{I}\left\| \frac{1}{m}\sum_{i\in I}\nabla_{\mb \beta_j}\ell_{i}^{\mathrm{noise}}(\mb \beta)\right\|_2^2=\underbrace{\frac{1}{m}\E_{i_1}\left\|\nabla_{\mb \beta_j}\ell_{i_1}^{\mathrm{noise}}(\mb \beta)\right\|_2^2}_{\mathrm{(A)}}+\underbrace{\frac{m-1}{m}\|\nabla_{\mb \beta_j}\ell^{\mathrm{noise}}(\mb \beta)\|_2^2}_{\mathrm{(B)}},
\end{equation} where we define $I:=\{i_1,\ldots,i_m\}\subset[n]$ and $\nabla_{\mb \beta_j}\ell^{\mathrm{noise}}(\mb \beta)$ in \eqref{eq:gradient_clean_def}.

\eqref{eq:res:lem:upbnd_noisegrad} gives an upper bound on (B):
\begin{equation}
\label{eq:minibatch:bnd_B_noise}
\begin{aligned}
{\mathrm{(B)}}\lesssim \frac{\sigma^2 {kd\log(n/d) + \log(k/\delta)}}{n}.
\end{aligned} 
\end{equation} 
The remaining step is to obtain a bound on (A). Since we have
\begin{equation*}
\begin{aligned}
\left\|\nabla_{\mb \beta_j}\ell_{i_1}^{\mathrm{noise}}(\mb \beta)\right\|_2^2\leq\|z_{i_1}\mb \xi_{i_1}\|_2^2\leq\|\mb \xi_{i_1}\mb \xi_{i_1}^\T\|z_{i_1}^2{\lesssim}{d+\log(n/\delta)}z_{i_1}^2,
\end{aligned}
\end{equation*} where the last inequality holds by \eqref{eq:lem:expect0},
applying the expectation and  \eqref{eq:bound_noisesquare} gives an upper bound on (A):
\begin{equation}
\label{eq:minibatch:bnd_A_noise}
\begin{aligned}
{\mathrm{(A)}}\lesssim{\frac{1}{n}\sum_{i=1}^n z_i^2\left(\frac{d+\log(n/\delta)}{m}\right)}&\lesssim\sigma^2{\left(1\vee\left({\frac{\log(1/\delta)}{n}}\right)^{1/2}\right)\left(\frac{d+\log(n/\delta)}{m}\right)}\\
&\leq\sigma^2{\left(\frac{d+\log(n/\delta)}{m}\right)},
\end{aligned}
\end{equation} where the last inequality hold by \eqref{eq:samplcomp_noise}.
Putting the results \eqref{eq:minibatch:bnd_B_noise} and  \eqref{eq:minibatch:bnd_A_noise} into \eqref{eq:decompose_norm_noise_gradient} reduces to \eqref{eq:res:lem:upbnd_noisegrad_minibatch}.

\section{Discussion on the proofs of \cite[Theorem~1]{ghosh2021max} and \cite[Theorem~1]{ghosh2020max}}
\label{sec:correctness}

In the proof of \cite[Theorem~1]{ghosh2021max}, they claimed that $n\gtrsim\delta^{-2}$ implies \cite[Equation~(45)]{ghosh2021max}. 
They showed that \cite[Equation~(45)]{ghosh2021max} follows from \cite[Lemmas~10 and 11]{ghosh2021max}. 
Their \cite[Lemma~10]{ghosh2021max} presents the concentration of the supremum of an empirical measure via the VC dimension and \cite[Lemma~11]{ghosh2021max} computes an upper bound on the VC dimension of the feasible set of the maximization. 
According to their proof argument, the number of observations $n$ should be proportional to the VC dimension $d\log(n/d)$ to obtain the concentration in \cite[Equation~(45)]{ghosh2021max}.
Their sufficient condition $n\gtrsim\delta^{-2}$ for \cite[Equation~(45)]{ghosh2021max} missed the dependence on the VC dimension. 
We suspect that this is a typo. 
While it does not ruin their main result, the sample complexity in \cite[Theorem~1]{ghosh2021max} might need to be corrected accordingly. 
Specifically, between \cite[Equation~(32) and (33)]{ghosh2021max}, the parameter $\delta$ in \cite[Lemma~6]{ghosh2021max} was set to $\delta=C k^{-2} \pi_{\min}^6$ to upper-bound the second summand in the right-hand side of {\cite[Equation~(32)]{ghosh2021max}}. 
Therefore, the corrected sample complexity of \cite[Lemma~6]{ghosh2021max} increases to $\tilde{O}(k^4d \pi_{\min}^{-12})$ so that it dominates the sample complexity for part (b) in \cite[Proposition~1]{ghosh2021max} ($n\gtrsim kd \pi_{\min}^{-3}$). 
Consequently, the sample complexity in \cite[Theorem~1]{ghosh2021max} will increase by a factor $k^3 \pi_{\min}^{-9}$. 

Next, we report another mistake in their analysis under the generalized covariate model \cite[Theorem~1]{ghosh2020max}. 
They mistakenly omitted the dependence of $\sigma$ in the sample complexity. A careful examination of their proof on page 48 in \cite{ghosh2019max} will reveal that they use the same technique as in their other analysis in the Gaussian covariates case \cite{ghosh2021max}. 
Therefore, we expect that their sample complexity should depend on the noise variance $\sigma^2$ to ensure that the next iterate belongs to the local neighborhood of the ground truth (refer to the proof of their Theorem 1 on page 1865 in \cite{ghosh2021max}).